%% file: samplepaper.tex
\PassOptionsToPackage{svgnames}{xcolor}
\documentclass[runningheads]{llncs}
\usepackage[T1]{fontenc}
\usepackage{graphicx}
\input{definitions}

\usepackage{soul}
\usepackage{url}
\usepackage[hidelinks]{hyperref}
\usepackage[utf8]{inputenc}
\usepackage{amsmath}
\usepackage{booktabs}
\usepackage{algorithm}
\usepackage{algorithmic}
\usepackage{longtable} 
\usepackage{pdflscape} 

\newcommand{\future}[1]{}

\allowdisplaybreaks

\usepackage{thmtools}
\usepackage{thm-restate}

\usepackage{tcolorbox}
\tcbuselibrary{skins,breakable}
\usetikzlibrary{shadings,shadows}

    {\endtcolorbox}

    {\endtcolorbox}

\newenvironment{myblock}[1]{%
    \tcolorbox[
    breakable,
    boxrule=0.5pt,
    arc=3pt,
    auto outer arc,
    colback=white,
    colframe=black!75,
    colbacktitle=white,
    coltitle=black,
    fonttitle=\bfseries,
    titlerule=0.5pt,
    boxsep=0pt,
    left=3pt,
    right=3pt,
    top=3pt,
    bottom=3pt,
    toptitle=2pt,
    bottomtitle=2pt,
    title={#1}]}%
    {\endtcolorbox}

\usepackage{ragged2e}
\usepackage{tcolorbox}
\tcbuselibrary{skins}
\definecolor{ovgray}{gray}{0.97}
\newcolumntype{Y}{>{\RaggedRight\arraybackslash}X}
\newcolumntype{K}{>{\bfseries\RaggedRight\arraybackslash}p{0.20\linewidth}}

\newcolumntype{L}{>{\RaggedRight\arraybackslash}p{0.22\linewidth}}
\newcolumntype{R}{>{\RaggedRight\arraybackslash}p{0.74\linewidth}}

\begin{document}
\title{Joint Air Traffic Flow and Capacity Management via Answer Set Programming}
\titlerunning{Joint ATFCM via ASP}
%
\author{Alexander Beiser\inst{1}\orcidID{0009-0009-4252-1043} \and
Markus Hecher\inst{2}\orcidID{0000-0003-0131-6771} \and
Nysret Musliu\inst{1}\orcidID{0000-0002-3992-8637} \and
Stefan Woltran\inst{1}\orcidID{0000-0003-1594-8972}}
\authorrunning{Beiser et. al.}
%
\institute{TU Wien, Vienna, Austria, \email{\{alexander.beiser\}@tuwien.ac.at} \and
University of Potsdam, Germany \& CNRS, France }
%
\maketitle              
\begin{abstract}
Operational \emph{Air Traffic Flow and Capacity Management} (ATFCM) balances flight demand with available sector capacity, to ensure safe and efficient operations.
Mathematical models enhance operational ATFCM performance by framing demand-capacity balancing as an optimization problem, maximizing efficiency while adhering to safety constraints.
However, SOTA research optimizes the aircraft trajectories (called ATFM) or the sector configuration (called DAC) separately.
This leaves a research gap of whether joint optimization of ATFM and DAC can bring benefits.
We partially address this limitation by introducing a \emph{joint ATFCM} model with an encoding in Answer Set Programming (ASP).
The ASP implementation is evaluated against two baselines applied to our joint model: a SOTA Mixed Integer Programming (MIP) model and an iterative CASA-based heuristic.
Computational experiments utilize an instance generator fitted to historical OpenSky Network flight data.
Our results indicate that the ASP model outperforms the MIP model, while ASP remains competitive against heuristics on small instances.
Furthermore, while DAC has the largest improvement on solving performance compared to rerouting and delaying, unrestricted variants of DAC or rerouting lead to search space thrashing.
\end{abstract}
\vspace{-0.3cm}
\section{Introduction}
\label{sec:introduction}
\vspace{-0.3cm}

Operational \emph{Air Traffic Flow and Capacity Management} (ATFCM) is essential to maintaining
safe and efficient air traffic.
ATFCM aims to balance demand and capacity.
Demand can be understood as the number of flights that wish to operate in a sector at a specific time, while capacity is the maximum number of flights a sector can handle in a time window~\cite{cook_european_2007}.
Demand must be kept below capacity, as overloading a sector increases the risk of safety incidents.
This is achieved by delaying or rerouting aircraft, or restructuring airspace.

However, at the time of writing, ATFCM faces unprecedented challenges due to multiple factors, including air traffic controller shortages and increasing traffic demand~\cite{eurocontrol_sesar_2025}.
These challenges lead to an increase in demand,
with a simultaneous decrease in capacity.
%
%
Tackling this imbalance is non-trivial; however, scientific endeavors are making progress on mathematical models suitable for optimization.
Existing ATFCM research can be split into work that focuses on increasing the efficiency of \emph{Air Traffic Flow Management} (ATFM) and optimizing airspace via \emph{Dynamic Airspace Configuration} (DAC).
ATFM focuses on measures to keep demand below capacity by delaying aircraft or rerouting them. 
Their models assume a static or pre-determined airspace and methodologically often use Mixed Integer (linear) Programming (MIP)~\cite{bertsimas_air_1998,dalmau_multi-objective_2024,garcia-heredia_combinatorial_2019}.
In contrast to this, DAC research uses methods based on geometric, graph-based, or trajectory-based
approaches~\cite{lui_airspace_2024},
but keeps flight plans fixed.

Nevertheless, gaps remain. 
First, \emph{joint} optimization of ATFM and DAC can potentially yield performance improvements, due to higher degrees of freedom.
Second, the dominance of MIP raises the question whether declarative logic programming methods such as Answer Set Programming (ASP)~\cite{gelfond_logic_2002} can serve as a viable formalization language.
From an ASP perspective, joint ATFCM is an interesting real-world setting with interacting hard and soft constraints (safety, efficiency, and fairness/stability). 

\smallskip
\noindent
\textbf{Contributions}.
We address these gaps by proposing a novel ATFCM model:
\vspace{-0.2cm}

\begin{itemize}
    \item We present the first combined air traffic flow and capacity management model for simultaneous optimization in ASP.
    \item We provide an instance generator for small instances fitted to open flight data (including the OpenSky Network COVID-19 dataset), enabling controlled scaling studies across geographic regions.
    \item We present the first ATFCM ASP encoding, which we benchmark on data sampled from the instance generator.
    We compare the ASP encoding to simulated state-of-the-art (SOTA) MIP and heuristic algorithms, and finally
    perform an ablation study to investigate the impact of enabling / partially disabling/disabling actions.
    Our findings indicate
    (1) promising results of ASP,
    (2) that restricting the search spaces of DAC and rerouting improves results under a fixed compute budget,
    and (3) that computational bottlenecks in both solving and grounding remain for unbounded variants.
\end{itemize}

After the introduction (Section~\ref{sec:introduction}) and the preliminaries (Section~\ref{sec:prelims}), we present the ATFCM model with the ASP encoding in Section~\ref{sec:atfcm}.
Section~\ref{sec:data} shows the data pipeline, Section~\ref{sec:exps} the experiments, and Section~\ref{sec:concl} our conclusions. 

\vspace{-0.3cm}
\section{Preliminaries and Related Work}\label{sec:prelims}
\vspace{-0.3cm}

\noindent
\textbf{Answer Set Programming}. 
ASP is a logic programming paradigm based on non-monotonic reasoning~\cite{gelfond_logic_2002}.
A program is a set of rules 
$p_1(\vecv {X}_1) \vee \ldots \vee p_\ell(\vecv {X_\ell}) \leftarrow p_{\ell{+}1}(\vecv {X}_{\ell{+}1}), \ldots, p_{m}(\vecv {X}_{m}), \neg p_{m{+}1}(\vecv {X}_{m{+}1}), \ldots, \neg p_n({\vecv {X}_n})$,
where $p_i(\vecv{X}_i)$ is a literal, $l,m,n$ are non-negative integers s.t. $l \leq m \leq n$,
and $\vecv{X}_i = \langle x_1, \ldots, x_z \rangle$ is a term vector comprising constants and variables.
\emph{Grounding} refers to the instantiation of variables by all possible domain values,
resulting in a ground program.
%
Semantics of ground programs is defined via \emph{Gelfond-Lifschitz reduct}.
\emph{Weak constraints} enable optimization~\cite{877512}.
ASP is attractive for ATFCM because it supports rapid modeling of constraints, and lexicographic optimization.

\noindent
\textbf{Setting.}
Air Traffic Flow and Capacity Management (ATFCM) is the layer of Air Traffic Management that aims to keep demand below capacity~\cite{cook_european_2007}.
Capacity is primarily driven by controller workload: airspace is partitioned into \emph{sectors} and each sector admits at most $X$ flights per hour.
For \emph{filed} flights (requested/planned operations), a network manager considers interventions that trade off safety (no overload) and efficiency (delay, fuel, distance, stability).
In practice, common levers are (i) ground delay, (ii) rerouting within (constrained) acceptable trajectories, and (iii) \emph{Dynamic Airspace Configuration} (DAC).
While (i) and (ii) operate on a per-aircraft trajectory basis, (iii) operates on a sector level.

\smallskip
\noindent
\textbf{Related work}.
In European operations, ATFM departure regulations are largely computed by the \emph{Computer Assisted Slot Allocation} (CASA) algorithm, which follows a ``First-Planned, First-Served''
principle~\cite{bucuroiu_european_2025}.
CASA is known to be efficiency-suboptimal; even early mathematical optimization models improved total delay and were later extended with additional criteria (e.g., fairness)~\cite{vranas_optimal_1996}.
Bertsimas and Patterson showed that the ATFM model for static airspaces and static routes, but dynamic delaying, is already \texttt{NP}-hard~\cite{bertsimas_air_1998}.
On the DAC side, many approaches construct sectorizations in continuous airspace using heuristics (e.g., genetic algorithms, Voronoi diagrams)~\cite{yin_multi-objective_2016} or exploit graph/flow and learning-based methods~\cite{chen_dynamic_2014,chandra_integration_2024}.
Model-wise, our approach occupies a middle ground between bottom-up sector construction by combining atomic sector structures and airspace partitioning.
From a logic programming perspective, the aviation literature contains only isolated uses (e.g., drones/vertiports)~\cite{nguyen_optimized_2024,kim_enhancing_2025}. 
Still, ASP is used in many other industrial domains~\cite{abels_train_2021,falkner_industrial_2018,baumeister_towards_2024}. 
We provide the first ASP encoding for joint ATFCM, compare it to SOTA methods, conduct an ablation study, and investigate whether ATFCM suffers from the grounding bottleneck.

\section{Joint ATFCM Model via ASP}\label{sec:atfcm}

In this section, we present the basics of our ATFCM model by following a running example based on Figure~\ref{fig:schematic-actions}.
Our model is based on the basic structure of the \emph{navpoint graph}.
Aircraft fly between \emph{navpoints} or \emph{airports} along \emph{airways} on the navpoint graph.
We assume that flights can not take other maneuvers and must stay on the airways.
With this, we can define our navpoint graph structure.

\begin{definition}[Navpoint Graph]
\label{def:navpoint-graph}
    The navpoint graph $G = (V,E)$ is undirected and labeled.
    Each $v \in V$ is either a navpoint or an airport and has an associated coordinate $(lat,lon,alt) \in \mathbb{R}^3$.
    Edges (airways) $e = (v_i,v_j) \in E$ connect vertices.
    Distances of airways are given by the geodesic distance $d:V^2 \rightarrow \mathbb{R}$.
\end{definition}
\begin{example}
\label{ex:navpoint-graph}
    $G$ is defined according to Figure~\ref{fig:schematic-actions}.
    Let $V' = \{v_0, v_1, v_2\}$, with $v_0 = (50,14,100)$, $v_1 = (48,14,100)$, $v_2 = (46,14,100)$.
    and $V_a = \{a_0, a_1\}$ with $a_0 =  (48.3536, 11.7832,0)$ and $a_1 = (48.1179,16.5566,0)$, so $V = V' \cup V_a$, 
    and further, $E = \{(a_0, v_0),(a_0,v_1), (a_0,v_2), (v_0,a_1), (v_1, a_1), (v_2,a_1)\}$.
    For the sake of this running example, we approximate the distances to be $d(a_0,v_0) = d(a_0,v_2) = d(a_1,v_0) = d(a_1,v_2) = 290 [km]$
    and $d(a_0,v_1) = d(a_1,v_1) = 180 [km]$.
\end{example}

  \begin{figure}[t]
    
        \begin{subfigure}[t]{0.49\linewidth}
              \centering
              \includegraphics[width=6.5cm]{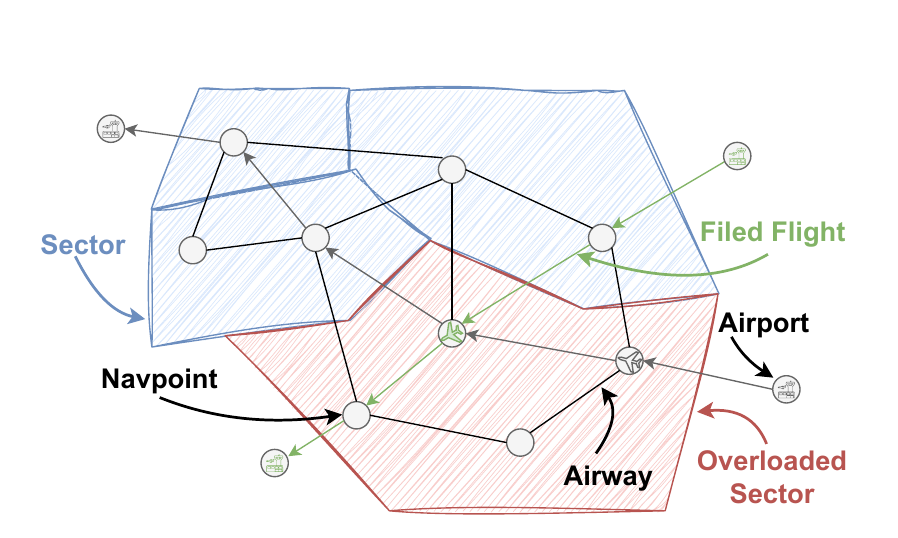}~\\[-0.55em]
              \caption{Schematics of ATFCM for a time $t$. 
              }
              \label{fig:schematic-general-graph}
              \vspace{-0.25cm}
    \end{subfigure}
    \begin{subfigure}[t]{0.49\linewidth}
          \centering
          \includegraphics[width=6.5cm]{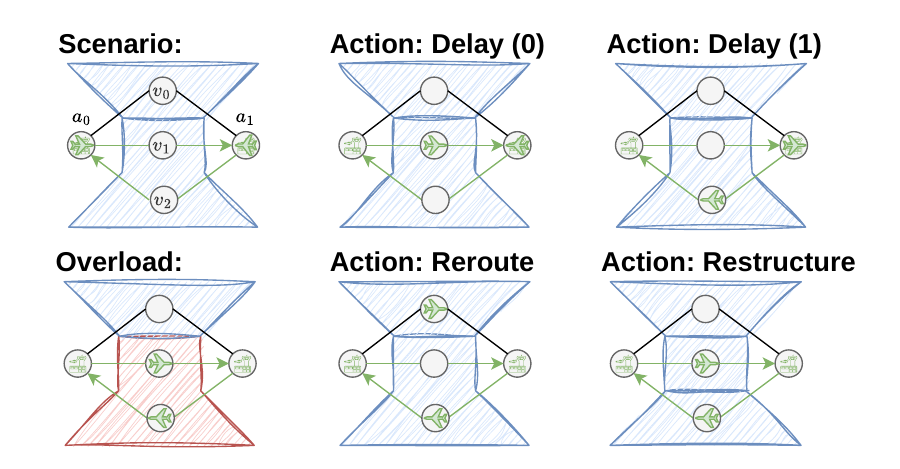}~\\[-0.55em]
          \caption{
            Schematics with available actions.
          }
          \label{fig:schematic-actions}
    \end{subfigure}
      \caption{
          (a):
          Filed flights fly along airways (edges) on the navpoint graph, overlaid with a sector configuration with four sectors, where one sector is overloaded (red).
          (b):
            From left top to bottom right:
            two filed flights, delaying a flight,
            indication of an overload, rerouting, and restructuring.
        }
  \end{figure}

\begin{restatable}[Time Model]{definition}{timemodel}
\label{def:time-model}
    Time is discretized into timesteps $t_i \in T$, with $T_{\textit{gran}}$ timesteps per hour where $T_{\textit{gran}}$ divides 60
    such that the timestep length is defined in minutes by 
    $T_s = \frac{60}{T_{\textit{gran}}}$.
    The set of timesteps $T$, is for $dur \in \mathbb{N}$ given by $T = \{t \mid t \in \mathbb{N}, 1 \leq t \leq dur \cdot T_{\textit{gran}}\} \cup \{0\}$.
\end{restatable}
\begin{example}
   We consider $T_{\textit{gran}} = 1$ with $|T| = 25$, so $dur = 24$ ($1$h granularity).
   If $T_{\textit{gran}} = 4$, with $|T| = 97$, then also $dur = 24$ ($15$ minutes granularity).
\end{example}


\begin{definition}[Sector Configuration]
\label{def:sector-model}
A \emph{sector} at time $t$ is a set of vertices (navpoints/airports).  
To keep the index set fixed across time, we represent each sector $i$ by a distinguished
\emph{sector identifier} $i \in V$. We write $S = V$ for the set of all identifiers.
The sector identified by $i$ at time $t$ is the set $\textit{sec}(i,t) \subseteq V$.
Formally, a \emph{sector configuration} is a function
$\textit{sec}: S \times T \rightarrow 2^{V}$ such that for every $t \in T$:
(i) $\{\textit{sec}(i,t) \mid i \in S,\ \textit{sec}(i,t)\neq\emptyset\}$
forms a partition of $V$; and
(ii) every airport is a singleton sector.
Conversely, $\textit{sec}^{-1}: V \times T \rightarrow S$ denotes the inverse assignment.
The initial sector configuration at timestep $t_s = 0$ is the atomic configuration, with atomic capacities $\textit{cap}(i,0) = c_i$.
Due to operational constraints, merging sectors does not lead to an increase in overall capacity.
We model this by defining $\textit{cap}(i,t)$ for a composite sector to be the max function.

    %
\end{definition}
\begin{example}
\label{ex:sector-model}
We define two sector models on $G$ from Example~\ref{ex:navpoint-graph} and let $T_{\textit{gran}} = 4$ and $|T| = 97$.
Let the atomic capacities be $\textit{cap}(a_0,0) {\,=\,} \textit{cap}(a_1,0) {\,=\,} 2$ and $\textit{cap}(v_0,0) {\,=\,}$ $ \textit{cap}(v_1,0) {\,=\,} \textit{cap}(v_2,0) {\,=\,} 1$.
We define the static model as $\forall t >0:$ $\textit{sec}(a_0,t) = \{a_0\}$, $\textit{sec}(a_1,t) = \{a_1\}$, $\textit{sec}(v_0,t) = \{v_0\}$, and $\textit{sec}(v_1,t) = \{v_1,v_2\}$.
%
\end{example}


\vspace{-0.1cm}
\begin{definition}[Aircraft Model]
\label{def:airplane-model}
    Let $\texttt{A}$ be the set of aircraft where an aircraft $p = (I_p,u_p,F_p) \in \texttt{A}$ is defined by its ID $I_p$,
    (time granularity normalized) velocity $u_p$, and its flights $F_p$.
    A flight $f = (I_f, \textit{tr}) \in F_p$ is comprised of an ID $I_f$ and a trajectory $\textit{tr} \in \mathcal{TR}_a$.
    A trajectory $\textit{tr} \in \mathcal{TR}_a$ is a (simple) path on $G$ augmented with timestamps in the form $(v_i,t_i)$, where $v_i \in V$ and $t_i \in T$;
    for pairs $(v_i,t_i),(v_{i+1},t_{i+1}) \in \textit{tr}$ it holds $t_i < t_{i+1}$ s.t. $t_{i+1} - t_i = \max\{\lceil \frac{d(v_i,v_{i+1})}{u_a}\rceil, 1 \}$.
    The set of flights is $\mathcal{F}$ and the set of flights w.r.t. a velocity $u_a$ is
    $\mathcal{F}_a$.
\end{definition}
\begin{example}
\label{ex:flight-model}
Let $G$ be as in Example~\ref{ex:navpoint-graph}, $T_{\textit{gran}}=4$, and $|T| = 97$.
We assume that both aircraft $\texttt{A} = \{p_0, p_1\}$ are typical single-piston engine aircraft, where $p_0 = (p_0,u_a,f_0)$ and $p_1 = (p_1, u_a, f_1)$.
Therefore, we choose their velocity to be $u_a = 172.79 [kts] = 0.08 [\frac{km}{s}] = 80 [\frac{km}{15 \cdot min}]$, where the last one is the time granularity normalized velocity.
With this velocity, we obtain the flight times $\Delta t = 3$ for $(a_0,v_1)$ and $(v_1,a_1)$, and $\Delta t = 4$ for $(a_0,v_0),(v_0,a_1),(a_0,v_2)$, and $(v_2,a_1)$.
Finally, according to Figure~\ref{fig:schematic-actions} the filed flights are: $\mathcal{F} = \{f_0, f_1\} = \{(0,((a_0,1),(v_1,4),(a_1,7))), (1,((a_1,1),(v_2,5),(a_0,9)))\}$.
\end{example}

\begin{definition}[\textbf{ATFCM Instance}]
An instance is $I = (G, T_{\textit{gran}}, T, \textit{sec}, \textit{cap}, \texttt{A})$, comprising a navpoint graph $G$, time granularity $T_{\textit{gran}}$, timesteps $T$, initial sector configurations $\textit{sec}$, atomic capacities $\textit{cap}$, and aircraft $\texttt{A}$. All identifiers are unique.
\end{definition}

\begin{definition}[\textbf{ATFCM Solution}]
\label{def:atfcm-solution}
A valid solution $\mathbf{S} = (\mathbf{T}, \mathbf{SEC}, \mathbf{A})$ defines solution timesteps $\mathbf{T} \supseteq T$, an updated DAC $\mathbf{SEC}$, and adjusted aircraft $\mathbf{A}$. 
We denote with $\mathbf{F}$ the set of adjusted flights stemming from $\mathbf{A}$. We require:
\begin{itemize}\setlength{\itemsep}{0pt}
    \item \textbf{Flight Adjustments:} Flights $f \in \mathbf{F}$:
    we measure arrival delay as $d_f$; if $f$ is delayed and/or rerouted, then $r_f=1$ (else $0$).
    \item \textbf{Sector Constraints:} $\mathbf{SEC}$ partitions navpoints over time (airports atomic).
    \item \textbf{Demand Routing:} A flight $f$ traversing navpoints $v_a$ to $v_b$ in $\Delta t = t_b - t_a$ occupies $\textit{sec}^{-1}(v_a,t)$ for $t \in [t_a, t_a + \Delta t/2]$ and $\textit{sec}^{-1}(v_b,t)$ subsequently. We define $\textit{over}(f,i,t) \in \{0,1\}$ to indicate if $f$ occupies sector $i$ at $t$.
    \item \textbf{Capacity Constraints:} Sector demand $q(i,t) = \sum_{f \in \mathbf{F}} \textit{over}(f,i,t)$ must satisfy $q(i,t) \leq \textit{cap}(i,t)$ for all $t \in \mathbf{T}$, $i \in V$.
\end{itemize}
\end{definition}

\begin{definition}[\textbf{Optimization Problem}]
\label{def:atfcm-optimization-problem}
Valid solutions are strictly ordered ($\mathbf{S} \prec \mathbf{S}'$) by a lexicographically minimized objective tuple $(\ATFCMoDELAY, \ATFCMoNUMSEC, \ATFCMoSECCHANGES, \ATFCMoREROUTED, \ATFCMoRECONFIG)$, balancing efficiency and stability:
\begin{itemize}\setlength{\itemsep}{0pt}
    \item \textbf{Arrival Delay:} $\ATFCMoDELAY = \sum_{f \in \mathbf{F}} d_f$
    \item \textbf{Active Sectors:} $\ATFCMoNUMSEC = \sum_{t \in \mathbf{T}_{>0}} \sum_{i \in V} \mathbb{I}\big(|\mathbf{SEC}(i,t)| > 0\big)$
    \item \textbf{Vertex-Sector Changes:} $\ATFCMoSECCHANGES = \sum_{t \in \mathbf{T}_{\geq 1}} \sum_{v \in V} \mathbb{I}\big(\mathbf{SEC}^{-1}(v,t-1) \neq \mathbf{SEC}^{-1}(v,t)\big)$
    \item \textbf{Reroutes/Delays:} $\ATFCMoREROUTED = \sum_{f \in \mathbf{F}} r_f$
    \item \textbf{Reconfigurations:} $\ATFCMoRECONFIG = \sum_{t \in \mathbf{T}_{\geq 0}} \sum_{v \in V} \mathbb{I}\big(\textit{sec}^{-1}(v,t) \neq \mathbf{SEC}^{-1}(v,t)\big)$
\end{itemize}
\end{definition}

\vspace{-0.6cm}
\subsection{ASP Encoding: Definition of Instance}
\label{subsec:asp-encoding-instance}
Based on these definitions, we are ready to present the ASP instance and encoding.
First, we are going to present the ASP instance with some preprocessing steps, followed by the encoding.
In our prototype written in Python\footnote{GitHub: \url{https://github.com/alexl4123/ASPaeroFlow-Optimizer/tree/main/02_ASP}}, the input as defined above is translated into an ASP instance, which we structure into three parts: \emph{navpoint graph instance},
\emph{initial DAC instance}, and \emph{filed flights instance}.
The navpoint graph instance defines the navpoint graph with the distinction between airports and navpoints.
Additionally, we perform a preprocessing step which eases the computational load of ASP, by computing normalized distances w.r.t. $T_{\textit{gran}}$ based on aircraft velocities --- so the distances are given in timesteps.

\begin{myblock}{Navpoint Graph Instance}
\small
    \begin{itemize}
        \item \texttt{arpt(<ID>)}:
            Indicates that a vertex is an airport.
        \item \texttt{nvpt\_e(<V\_i>,<V\_j>,<V\_a>,<$\Delta t$>)}:
            For two connected vertices $e = (v_i,v_j) \in E$,
            we pre-compute the normalized travel duration $\Delta t$ for all normalized velocities of aircraft $u_a$.
    \end{itemize}
\end{myblock}

For the initial DAC instance we define atomic capacities and init configuration.

\begin{myblock}{Initial DAC Instance}
\small
    \begin{itemize}
        \item \texttt{atomic\_s(<NVPT-ID>,<CAP>,<T>)}:
            Initial capacities (assumed static).
        \item \texttt{nvpt\_s(<NVPT-ID>,<SEC-ID>,<T=0 for FACTS>)}:
            The initial sector configuration for time $t = 0$.
    \end{itemize}
\end{myblock}

The filed flights are given with the additional information about aircraft velocity and the mapping for multi-leg flights.

\begin{myblock}{Filed Flights Instance}
\small
    \begin{itemize}
        \item \texttt{acft(<ID>,<u\_a>)}:
            Aircraft ID and its velocity $u_a$.
        \item \texttt{acft\_flt(<ACFT-ID>,<FLIGHT-ID>)}:
        The mapping between aircraft and flights; needed among others for multi-leg flights.
        \item \texttt{nvpt\_flt\_pln(<FLT-ID>,<NVPT-ID>,<T>)}:
            The filed flight plan as 3-tuples of flight ids, navpoint ids, and time steps.
    \end{itemize}
\end{myblock}

\begin{example}
The converted running example into the ASP instance:
\begin{lstlisting}[style=examplestyle]
arpt(a0). arpt(a1). nvpt_e(a0, v0, 80, 4). nvpt_e(a0, v1, 80, 3). nvpt_e(a0, v2, 80, 4).
nvpt_e(a1, v0, 80, 4). nvpt_e(a1, v1, 80, 3). nvpt_e(a1, v2, 80, 4).  atomic_s(a0,2,0). atomic_s(a1,2,0).
atomic_s(v0,1,0). atomic_s(v1,1,0). atomic_s(v2,1,0). nvpt_s(a0, a0, 0). nvpt_s(a1, a1, 0).
nvpt_s(v0, v0, 0). nvpt_s(v1, v1, 0). nvpt_s(v2, v1, 0).  acft(0,80). acft(1,80). acft_flt(0,0).
acft_flt(1,1). nvpt_flt_pln(0,a0,1). nvpt_flt_pln(0,v1,4). nvpt_flt_pln(0,a1,7). nvpt_flt_pln(1,a1,1). 
nvpt_flt_pln(1,v2,5). nvpt_flt_pln(1,a0,9). atomic_s(N,C,T) :- atomic_s(N,C,0), t(T).
\end{lstlisting}
    \end{example}

\vspace{-0.6cm}
\subsection{ASP Encoding: Actions}
\label{subsec:asp-encoding-actions}

We continue to present the ASP encoding.
The encoding is written in a way that enables us to perform an ablation study, which actions (\emph{ground delaying}, \emph{rerouting}, or \emph{airspace restructuring}) are actually beneficial.
For each action, we implement $3$ versions, which gives us in total $27$ possible ASP configurations for our ablation study.
An ASP configuration is given by the 3-tuple $(R,D,S)$, 
with $R \in \{r, r_p, \neg r\}$,
$D \in \{d, d_p, \neg d\}$, and
$S \in \{s, s_p, \neg s\}$.
The Python wrapper translates a configuration $(R,D,S)$ to a set of facts, which provides instructions to the encoding.
Here, we focus on the definitions of these facts and how this is translated in the encoding.
We start with $D$, followed by $R$, and lastly $S$.

\begin{myblock}{Ground Delaying ($D \in \{d, d_p, \neg d\}$)}
\small
\begin{itemize}
    \item $d$: Delaying along the entire considered time $T$
    \begin{itemize}
        \item ASP atom: \texttt{reg\_del}
    \end{itemize}
    \item $d_p$: Partial delaying (only max. \texttt{<MAX-TIME>} = 4 timesteps)
    \begin{itemize}
        \item ASP atom: \texttt{reg\_restr\_del}
        \item ASP atom: \texttt{reg\_dmd}\texttt{(<MAX-TIME>)}
    \end{itemize}
    \item $\neg d$: Delaying is forbidden
    \begin{itemize}
        \item ASP atom: \texttt{-reg\_del}
    \end{itemize}
\end{itemize}
\end{myblock}

Ground delaying means that we guess the actual departure time depending on the atom ($d,d_p,\neg d$) present in the instance (Lines~(1--4)).
For $d_p$, we restrict the guess to a pre-defined selection (Line~(3)).

\begin{lstlisting}
p_dep_t(ID,T) :- flt_id(ID), T = #min{T':nvpt_flt_pln(ID,_,T')}.
1{a_dep_t(ID,T):t(T), T >= TP}1 :- p_dep_t(ID,TP), reg_del.
1{a_dep_t(ID,T):t(T),T>=TP,T<=TP+MD}1 :- p_dep_t(ID,TP),reg_dmd(MD),reg_rd.
a_dep_t(ID,TP) :- p_dep_t(ID,TP), -reg_del.
\end{lstlisting}
\begin{myblock}{Rerouting ($R \in \{r, r_p, \neg r\}$)}
\small
    \begin{itemize}
        \item $r$: Full rerouting is allowed (ASP finds paths).
        \begin{itemize}
            \item ASP atom: \texttt{reg\_rer}
        \end{itemize}
        \item $r_p$: Partial rerouting; pick best of 4 diverse precomputed shortest paths (Jaccard similarity at most $60\%$).
        Let \texttt{<FLT>} be the flight ID, \texttt{<NVPT>} be the navpoint at flight time \texttt{<FLT-T>} for a chosen path \texttt{<PTH>}.
        \begin{itemize}
            \item ASP atom: \texttt{reg\_restr\_rer}
            \item ASP atom: \texttt{reg\_rrp(<FLT>,<NVPT>,<FLT-T>,<PTH>)}
        \end{itemize}
        \item $\neg r$: Rerouting is forbidden.
        \begin{itemize}
            \item ASP atom: \texttt{-reg\_rer}
        \end{itemize}
    \end{itemize}    
\end{myblock}

We continue with $R$.
First, we combine the ATFM side, whether a flight is rerouted or delayed, into one atom \texttt{rer} (Lines~(1--2)).
If $r$ is enabled we guess a set of vertices (Line~(4)), if $r_p$ is enabled then we guess one path of the provided ones (Lines~(5)--(6)), and if rerouting is disabled ($\neg r$) or \texttt{rer} is not set, we require that the flight follows the vertices of the filed flight (Lines~(3) and (7)).

\begin{lstlisting}
{rer(ID)} :- nvpt_flt_pln(ID,_,_).
rer(ID) :- p_dep_t(ID,T), a_dep_t(ID,T'), T' > T.
nvpt_flt(ID,X,T) :- nvpt_flt_pln(ID,X,T), not rer(ID).
{nvpt_flt(ID,X,T):nvpt(X)}1 :- rer(ID),t(T),a_dep_t(ID,T'),T>=T',reg_rer.
1{nvpt_flt_cp(ID,P):reg_rrp(ID,_,_,P)}1 :- rer(ID),reg_restr_rer.
nvpt_flt(ID,X,T+T') :- reg_rrp(ID,X,T',P),rer(ID),a_dep_t(ID,T),nvpt_flt_cp(ID,P),reg_restr_rer.
nvpt_flt(ID,X,T+TT) :- nvpt_flt_pln(ID,X,T),a_dep_t(ID,TACT),p_dep_t(ID,TPLAN),TT=TACT-TPLAN,-reg_rer.
\end{lstlisting}

\begin{myblock}{Airspace Restructuring ($S \in \{s, s_p, \neg s\}$)}
\small
    \begin{itemize}
        \item $s$: Full airspace restructuring is allowed (no explicit operational constraints).
        \begin{itemize}
            \item ASP atom: \texttt{reg\_dyn\_sec\_alloc}.
        \end{itemize}
    \item $s_p$: Partial restructuring --- 4 choices (\texttt{<D>}): (i) keep airspace structure, split into (ii) 2 or (iii) 4 connected sectors, or (iv) split into atomic structure.
        We enforce connected sectors in our partial restructuring to comply with typical operational needs of lowering the \emph{sector complexity}.
        Let \texttt{<S>} be the original sector, \texttt{<NVPT>} be the navpoint, and \texttt{<S1>} the new sector.
        \begin{itemize}
            \item ASP atom: \texttt{reg\_rdsa}
            \item ASP atom: \texttt{nvpt\_restr\_sec\_alloc(<S>,<D>,<NVPT>,<S1>)}
        \end{itemize}
        \item $\neg s$: Keep initial structure.
        \begin{itemize}
            \item ASP atom: \texttt{-reg\_dyn\_sec\_alloc}.
        \end{itemize}
    \end{itemize}
\end{myblock}

We continue with $S$.
First, each airport sector is atomic (Line~(1)).
If $s$ is enabled, then an arbitrary sectorization set is guessed (Line~(2)), if $s_p$ is enabled, then one of four choices is picked (Lines~(3)--(5)), and if $\neg s$ is enabled, then we enforce the initial sectorization (Line~(6)).

\begin{lstlisting}
nvpt_s(SEC,SEC,T) :- nvpt(SEC), t(T), arpt(SEC).
1{nvpt_s(N,S,T):nvpt_s(_,S,0)}1:- nvpt(N),t(T),not arpt(N),reg_dyn_sec_alloc.
isec(SEC) :- nvpt_s(_,SEC,0), reg_rdsa.
1{nvpt_dec(S,0..3,T)}1 :- isec(S),t(T),not arpt(S),reg_rdsa.
nvpt_s(N1,S1,T) :- nvpt_dec(S,D,T),nvpt_restr_sec_alloc(S,D,N1,S1),reg_rdsa.
nvpt_s(NAV,SEC,T) :- t(T), nvpt_s(NAV,SEC,0), -reg_dyn_sec_alloc.
\end{lstlisting}

\vspace{-0.2cm}
\subsection{ASP Encoding: Constraints}
\label{subsec:asp-encoding-constraints}
\vspace{-0.2cm}

We require that the sectorization is a partition and adheres to our definition of a sectorization.
Therefore, for each timestep,
it must not be the case that a vertex is assigned to two sectors (Line~(1)),
a navpoint is assigned to an airport (Line~(2)),
an airport is assigned to a navpoint sector (Line~(3)),
or that a non-empty sector labeled by index $i$ does not contain vertex $i$ (Line~(4)).

\begin{lstlisting}
:- nvpt_s(NAV,SEC1,T), nvpt_s(NAV,SEC2,T), SEC1 != SEC2.
:- nvpt_s(NAV,SEC,T), not arpt(NAV), arpt(SEC).
:- nvpt_s(NAV,SEC,T), arpt(NAV), not arpt(SEC).
:- nvpt_s(NAV1,SEC,T), NAV1 !=SEC, not nvpt_s(SEC,SEC,T).
\end{lstlisting}

We further require that at any step the flight must not be operating before its planned departure time (Line~(1)).
Additionally, we require that for $r$ the rerouting is an actual path, therefore the flight starts at the departure airport and lands at the destination airport (Lines~(2--7)).
A flight is finished if its destination is reached (Lines~(8--9)).

\begin{lstlisting}
:- rer(ID), nvpt_flt(ID,_,T), p_dep_t(ID,T'), T<T'.
planned_origin(ID,X) :- rer(ID),p_dep_t(ID,T),nvpt_flt_pln(ID,X,T).
nvpt_flt(ID,X,T) :- rer(ID), a_dep_t(ID,T), planned_origin(ID,X).
p_dest(ID,X) :- rer(ID),T=#max{T':nvpt_flt_pln(ID,_,T')},nvpt_flt_pln(ID,X,T).
a_arr_t(ID,T) :- flt_id(ID), T=#max{T':nvpt_flt(ID,_,T')}.
:- flt_id(ID), not a_arr_t(ID,_).
:- flt_id(ID), p_dest(ID,X), a_arr_t(ID,T), nvpt_flt(ID,X',T), X != X'.
a_dest_f_rch_t(ID,T) :- rer(ID), p_dest(ID,X), T=#min{T':nvpt_flt(ID,X,T')}.
:- rer(ID), nvpt_flt(ID,_,T), a_dest_f_rch_t(ID,T'), T > T'.
\end{lstlisting}

We require that a flight is a valid trajectory, so is a path with valid timesteps (Lines~(1--4)),
that a flight is never twice at the same location (Line~(5)),
that a flight actually occurs (Lines~(6--7)),
and ensure that multi-leg flights occur in the right order (Line~(8)).
\begin{lstlisting}
p_nvpt_seq(ID,T,TT) :- nvpt_flt(ID,_,T), nvpt_flt(ID,_,TT), T < TT.
n_nvpt_seq(ID,T,TT) :- p_nvpt_seq(ID,T,TT),T<TTT,TTT<TT,p_nvpt_seq(ID,T,TTT).
nvpt_seq(ID,T,TT) :- p_nvpt_seq(ID,T,TT), not n_nvpt_seq(ID,T,TT).
:- nvpt_seq(ID,T,TT), nvpt_flt(ID,X,T), nvpt_flt(ID,XX,TT), acft_flt(AID,ID), acft(AID,S), D=TT-T, not nvpt_e(X,XX,S,D).
:- nvpt_flt(ID,X,T), nvpt_flt(ID,X,T'), T != T', not arpt(X).
flightOccurs(ID) :- nvpt_flt(ID,_,_).
:- flt_id(ID), not flightOccurs(ID).
:- acft_flt(AID,FID0), acft_flt(AID,FID1), FID0!=FID1, p_arr_t(FID0,T0), p_arr_t(FID1,T1), T0 < T1, a_arr_t(FID0, T0'),a_dep_t(FID1,T1'),T0'>=T1'.
\end{lstlisting}

\subsection{ASP Encoding: Navpoint Flight Mapping}
\label{subsec:navpoint-flight-mapping}

Due to the joint modeling of DAC and flow optimization, we are left to map navpoint reach times to travel times.
We approximate sector occupancy by assigning the first half of the flight between two navpoints to the first navpoint sector, and the second half to the second navpoint sector:
let $f \in \mathbf{F}$ be a flight reaching navpoint $v_a \in V$ at time $t_a$ and then flying to navpoint $v_b \in V$, which it reaches at time $t_b$.
Assuming that $\Delta t = t_b - t_a$, then $f$ is in the timespan $t \in [t_a, t_a + \lfloor \frac{\Delta t}{2} \rfloor]$ in sector $\textit{sec}^{-1}(v_a,t)$,
and for $t' \in [t_a + \lfloor \frac{\Delta t}{2} \rfloor + 1, t_b]$ in sector $\textit{sec}^{-1}(v_b,t')$.
We encode this in Lines~(1--5),
while we forbid ``holes'' (Line~(6)). 

\begin{lstlisting}
flight(ID,S,T) :- nvpt_flt(ID,X,T), nvpt_s(X,S,T).
flightT(ID,T,T',0) :- nvpt_seq(ID,T,TT),D=TT-T,t(T'),T'>T,T'<=T+D/2.
flightT(ID,TT,T',1) :- nvpt_seq(ID,T,TT),D=TT-T,t(T'),T'<TT,T'>T+D/2.
flight(ID,S,T') :- flightT(ID,T,T',0), nvpt_flt(ID,X,T), nvpt_s(X,S,T').
flight(ID,S,T') :- flightT(ID,TT,T',1), nvpt_flt(ID,XX,TT), nvpt_s(XX,S,T').
:- t(T), a_dep_t(ID,T'), T > T', a_arr_t(ID,T''), T < T'', not flight(ID,_,T).
\end{lstlisting}

\vspace{-0.2cm}
\subsection{ASP Encoding: Optimization}
\label{subsec:optimization}
\vspace{-0.2cm}
We map the optimization problem of Definition~\ref{def:atfcm-optimization-problem} into a multi-objective lexicographic optimization ASP program.
Regarding specifics, we treat safety as a soft constraint in order to track its computational results in the experiment section.
Still, only 0-overload solutions are valid solutions.
Overall, we have then six optimization criteria: the total number of overloads followed by the criteria from Definition~\ref{def:atfcm-optimization-problem}.
\begin{lstlisting}
sec_len(SEC,T,CL) :- t(T),nvpt(SEC),CL=#count{NAV:nvpt_s(NAV,SEC,T)}.
sec(SEC,T,0) :- sec_len(SEC,T,CAP_LEN), CAP_LEN = 0. 
sec(SEC,T,C) :- t(T), nvpt(SEC), C = #max{C',NAV:nvpt_s(NAV,SEC,T),atomic_s(NAV,C',T)}, sec_len(SEC,T,CAP_LEN), CAP_LEN > 0.
overload(X,T,LOAD-C) :- sec(X,T,C), #count{ID:flight(ID,X,T)}=LOAD, LOAD>C.
:~ overload(X,T,OVER). [OVER@10,X,T]     % PRIMARY WEAK CONSTRAINT
\end{lstlisting}

We use efficiency objectives as our secondary objectives, which follows common practice in the literature~\cite{bertsimas_air_1998}.
However, we expand the current practice by incorporating \emph{airspace efficiency} metrics, which we partially take from the DAC literature.
Generally speaking, DAC wants to have as few sectors as possible with as few changes as possible, while additionally minimizing sector complexity.
As proxies for this, people may use minimization of excess capacity~\cite{ATRDS-ROBUST-DAC}, or explicitly minimizing for controller communication overload~\cite{lui_airspace_2024}.
We minimize the number of sector changes, while considering airspace complexity with our partial sectorization approach ($s_p$).

\begin{lstlisting}
p_arr_t(ID,T) :- flt_id(ID), T = #max{T':nvpt_flt_pln(ID,_,T')}.
arrival_delay(ID,Y-T) :- p_arr_t(ID,T), a_arr_t(ID,Y).
sector_number(T,NUM) :- t(T), NUM = #count{SEC:sec_len(SEC,T,C), C > 0}.
nvpt_ch(NAV,T,1) :- nvpt_s(NAV,SEC0,T),nvpt_s(NAV,SEC1,T+1),SEC0!=SEC1.
nvpt_ch(NAV,T,0) :- nvpt_s(NAV,SEC0,T),nvpt_s(NAV,SEC0,T+1).
sector_diff(T,DIFF) :- t(T), DIFF = #sum{V,NAV: nvpt_ch(NAV,T,V)}.
:~ arrival_delay(ID,DIFF). [DIFF@9,ID]   % SECONDARY WEAK CONSTRAINT
:~ sector_number(T,NUM). [NUM@8,T]       % TERTIARY WEAK CONSTRAINT
:~ sector_diff(T,DIFF). [DIFF@7,T]       % OPTIMIZE FOR LEAST SECTOR CHANGES
\end{lstlisting}

Lastly, we optimize for fairness/stability to the original plan, where we follow the general plan of optimizing for flow first and then airspace.

\begin{lstlisting}
reconfig(NAV,T) :- nvpt_s(NAV,SEC,0), nvpt_s(NAV,SEC2,T), SEC != SEC2, T > 0.
:~ rer(ID). [1@6,ID]             % OPTIMIZE FOR LEAST REROUTES/DELAYS
:~ reconfig(ID,T). [1@5,ID,T]        % OPTIMIZE FOR LEAST RECONFIGS
\end{lstlisting}

\vspace{-0.2cm}
\subsection{Finalizing the Example}
\vspace{-0.2cm}

\begin{example}
We show the results of the running example applied to the encoding above under the configuration $(r,d,s)$:
\begin{lstlisting}
max_t(96). reg_del. reg_dyn_sec_alloc. reg_rer. 
\end{lstlisting}
%
Solving this instance leads to the following results:
$0$ overload, $0$ arrival delay, $294$ number of sectors, $3$ sector-diffs, $1$ rerouting, and $94$ reconfigurations.
The solution of the sectorization is: $\forall t > 0$:
$\textit{sec}(a_0,t) = \{a_0\}$ and
$\textit{sec}(a_1,t) = \{a_1\}$.
For $t \leq 3$ and $t \geq 6$ it is: $\textit{sec}(v_1,t) = \{v_0,v_1,v_2\}$,
whereas for $t \in \{4,5\}$ it is: $\textit{sec}(v_0,t) = \{v_0\}$ and $\textit{sec}(v_1,t) = \{v_1,v_2\}$.
Further, the optimized flight trajectories are: $\mathcal{F}_a = \{(0,((a_0,1),(v_1,4),(a_1,7))), (1,((a_1,1),(v_0,5),(a_0,9)))\}$.
The trajectory-sector mappings are:
$f_0: (a_0,1),(a_0,2),(v_1,3),(v_1,4),$ $(v_1,5),(a_1,6),(a_1,7)$;
$f_1: (a_1,1),$ $(a_1,2),(a_1,3),(v_0,4),$ $(v_0,5),(v_1,6),(v_1,7),(a_0,8),(a_0,9)$.
The (total) number of sectors is also optimal: each airport must be a singleton sector, therefore their contribution is $96 \cdot 2$.
The en-route sector contribution is $96 \cdot 1$ for $v_1$ and $2 \cdot 1$ for $v_0$ (for $t \in \{4,5\}$).
As the lowest possible number of en-route total sectors is $96$, any lower number than $98 = 96 + 2$ would result in merging $v_1$ and $v_0$, which would lead to an overload.
This results in a (total) number of sectors of $290$ (for $t > 0$), which in addition with the initial configuration with $4$ sectors for $t=0$ yields in $294$.
The solution changes the assignment of navpoint $v_0$ three times: once between $t = 0$ and $t = 1$, between $t = 3$ and $t = 4$ and between $t = 5$ and $t = 6$, leading to sector diff of $3$.
Further, flight $1$ is rerouted from flying over navpoint $v_2$ to flying via $v_0$ and finally, $22$ reconfigurations results from reconfigurations of $v_0$ in $94$ timesteps ($t \not \in \{0,4,5\}$), yielding $97-3 = 94$.
\end{example}

\subsection{Correctness}
We briefly argue for correctness for $\textit{ASP} (r,d,s)$.
%
Let $\mathcal{I}$ be an arbitrary instance;
let $\Pi$ be the encoding, then $\Pi'$ is the encoding without soft constraints, but with $c_1 \coloneqq \{\leftarrow \texttt{overload}(X,T,OVER)\}$.
Let $\mathcal{AS}$ be the answersets of $\Pi'$.
For any flight $f \in \mathbf{F}$, let $t^a_f$ be its actual arrival time.
Let $c(\mathcal{I})$ be instance converted to ASP without $max\_t(t)$ (Section~\ref{subsec:asp-encoding-instance});
for answersets $AS \in \mathcal{AS}$ let $c^{-1}(AS)$ be the converted corresponding solution $\mathbf{S}$ where $\mathbf{T}$ is chosen to be $\max_{f \in \mathbf{F}} t^a_f$. 
Let $\mathcal{S}_{= t}$ be the set of valid solutions for $\mathcal{I}$ s.t. $\forall \mathbf{S} \in \mathcal{S}_{= t}:$ $\max\{\mathbf{T}\} = t$,
where we require $\max_{f \in \mathbf{F}}\{t^a_f\} = t$;
further $\mathcal{S}_{\leq t} = \bigcup_{t' \in \{0,\ldots,t\}} \mathcal{S}_{= t'}$.
%
%
\begin{restatable}[$\star$]{lemma}{LemSolutionSpaceEquivalence}
\label{thm:solution-space-equivalence}
\textbf{Solution Space Equivalence}.
For any $t \geq 24 \cdot T_{\textit{gran}}$,
the sets $\mathcal{S}_{AS} \coloneqq \{c^{-1}(AS) \mid AS \in \mathcal{AS}(\Pi' \cup c(\mathcal{I)} \cup \{max\_t(t)\})\}$ and
$\mathcal{S}_{\leq t}$ match.  
%
%
\end{restatable}

\begin{restatable}[$\star$]{theorem}{ThmSolutionExistence}
\textbf{Solution Existence.}
Let $\mathcal{I}$ be an instance and $t \geq 24 \cdot T_{\textit{gran}}$.
$\Pi \cup c(\mathcal{I}) \cup \{max\_t(t)\}$ has a 0-overload answerset iff there exists a solution $\mathbf{S} \in \mathcal{S}_{\leq t}$ for the optimization problem (Definition~\ref{def:atfcm-optimization-problem}).
\end{restatable}

\begin{restatable}[$\star$]{theorem}{ThmOptimality}
\textbf{Optimality.}
Let $\mathcal{I}$ be an arbitrary instance, let $t \geq 24 \cdot T_{\textit{gran}}$, and let one of its best solutions $\mathbf{S}_{opt}$ be found within $\max\{\mathbf{T}\} = t:$ $\mathbf{S}_{opt} \in \mathcal{S}_{\leq t}$.
Then $\Pi \cup c(\mathcal{I}) \cup \{max\_t(t)\}$ is optimal.
\end{restatable}

\vspace{-0.2cm}
\section{Pipeline Integration}
\label{sec:data}
\vspace{-0.2cm}

\begin{figure}[t]
    \centering
    \includegraphics[width=12cm]{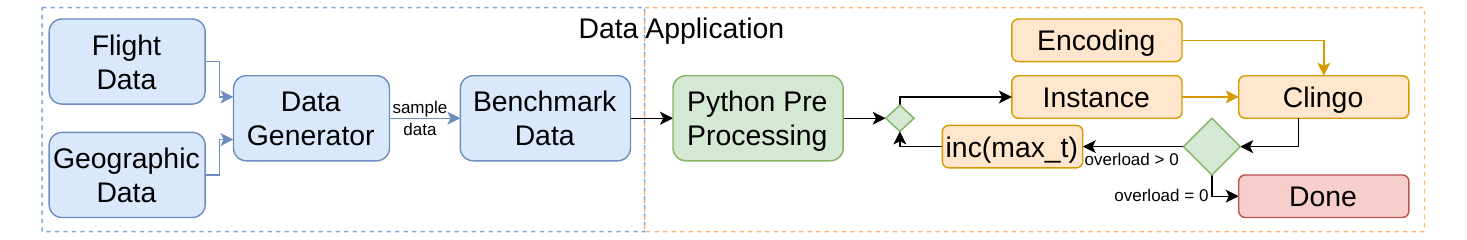}
    \caption{Pipeline displaying the entire process, from data generation (left), over Python preprocessing to solving with ASP (right).}
    \label{fig:pipeline}
    \vspace{-0.5cm}
\end{figure}

In Figure~\ref{fig:pipeline} we show the integration of the encoding into the overall pipeline.
Based on actual data, we developed an open source data generator, which enables us to sample data from different geographic regions and with an arbitrary number of flights.
This benchmark data is then pre-processed in Python, which again starts the ASP solving loop:
we start with a fixed maximum time (\texttt{max\_t(t)}), which is iteratively expanded until a solution with $0$ overload is found\footnote{See supplementary material for details.}.

\vspace{-0.2cm}
\subsection{An Open Source Data Generator}
\vspace{-0.2cm}

Data in ATM research is often closed source, due to privacy concerns or the fear of competitive disadvantage.
Although these concerns are understandable from an economic perspective, they are disadvantageous for science --- which leads to two common issues:
(a) published research cannot be reproduced, thereby manifesting itself in a so-called reproducibility crisis;
(b) researchers relying on open data are sometimes forced to use datasets that are ill-suited to the task.
In order to circumvent this problem, movements have been initiated, such as the formation of the Open Science Alliance for ATM research~\cite{bolic_roadmap_2024}.
However, open data for ATFCM research is still scarce.
This paper takes a pro-open-data stance, by creating a data generator built from publicly available open data.
Our data generator is available on GitHub\footnote{GitHub Link: \url{https://github.com/alexl4123/ASPaeroFlow-DataGenerator}} 
and we published the dataset used in the experiments on Zenodo\footnote{Zenodo DOI: \url{https://doi.org/10.5281/zenodo.20782934}}.
The \emph{Data Generator} executes three sequential internal phases to construct problem instances:
\begin{itemize}
\item \textbf{Graph Generation:} Constructs a configurable grid-graph of navpoints augmented with real-world airport coordinates using the public \emph{OurAirports} dataset.
\item \textbf{Flight Pair Generation:} Employs a stochastic point process~\cite{gerstner2014neuronal} to sample flight operations.
The underlying \emph{Poisson model} is parameterized by historical departure and arrival intensities derived from the \emph{OpenSky Network} COVID-19 surveillance dataset~\cite{strohmeier_crowdsourced_2021}.
\item \textbf{Trajectory \& Sector Generation:} Computes baseline trajectories via \emph{shortest-path routing} for origin-destination pairs.
Initial airspace configurations are partitioned using a \emph{Breadth-First Search (BFS)} algorithm to guarantee sectors are connected or geometrically convex.
\end{itemize}

\vspace{-0.2cm}
\section{Experiments}
\label{sec:exps}
\vspace{-0.2cm}

We conduct experiments on $5$ graph scenarios, as shown in Table~\ref{tab:scenario-description}.
Thereby, we enable a scaling study from  16 to 117 vertices.
Further, we range flights from 10 to 100, with steps of $10$.
For each such combination of geographic and flight scenario, we generated 4 instances with different seeds.
The size of these instances was chosen to be able to conduct experiments with exact solvers.
Industry-sized instances largely depend upon the selected area of investigation and range from several hundreds of flights and sectors up to $30,000$ flights, $\approx 10^3$ sectors, and $\approx 10^5$ navpoints for European airspace.
However, SOTA research using exact MIP models considers smaller instances which go up to approximately $\approx 3,000$ flights with bounded choices, while considering static sector configurations~\cite{garcia-heredia_combinatorial_2019}.
By enabling more degrees of freedom and dynamic alternative rerouting/delaying choices,
we cannot expect to solve instances with thousands of flights.

\begin{table}[t]
  \centering
  \tiny
  \setlength{\tabcolsep}{5.5pt} 
  \renewcommand{\arraystretch}{0.95} 
    \begin{tabular}{lrrrrrrrr}
      \toprule
      Instance & Region & En-route & Airports & Total & Capacity & $T_{\textit{gran}}$ & Sec.-Num\\
      \midrule
      EA-3x3 & East-Asia & 9 & 7 & 16 & 1 & 1 & 10 \\
      CE-5x5 & Europe & 25 & 7 & 32 & 1 & 1 & 12 \\
      IND-4x10 & India & 40 & 8 & 48 & 1 & 1 & 16 \\
      USA-7x7 & USA & 49 & 8 & 57  & 1 & 1 & 16 \\
      EUR-10x10 & Europe & 100 & 17  & 117 & 1 & 1 & 27 \\
      \bottomrule
    \end{tabular}
  \caption{
  Scenario details for the navpoint graphs.
  Besides the instance and region, we show \emph{En-route} vertices, \emph{Airport} vertices, \emph{Total} vertices,
  capacity per timestep, timestep granularity ($T_{\textit{gran}}$), and the number of the initial sectors at $t=0$.
  }
  \label{tab:scenario-description}
  \vspace{-0.7cm}
\end{table}

\smallskip
    \noindent
    \textbf{Variants}.
    We consider $27$ variants of the ASP encoding, where each variant is defined by the previously defined 3-tuple $(R,D,S)$.
    Our model implements a heuristic proxy for the CASA algorithm, iteratively assigning ground delays to aircraft to resolve sector capacity constraints.
    We denote the baseline solution as INITIAL, which are the values of the instance without any actions taken.
    Further, we simulate a SOTA dynamic MIP algorithm, that acts on delaying and rerouting aircraft, with bounded delay and reroute choices.

\smallskip
    \noindent
    \textbf{Technical Setup}.
    Experiments were conducted on a cluster
    with 2x Intel Xeon Silver 4314 CPUs, 512GB RAM, and Ubuntu 22.04 (Kernel 5.15.0-131-generic).
    The compute budget is limited by a TIMEOUT of $1800$s and a MEMOUT of $35$GB.
    We use Python~3.13.7 with Clingo (5.8.0) API, which preprocesses and converts the input to an ASP instance.
    For MIP we use Gurobi (build v13.0.0rc1).


\begin{table}[t]
\centering
\scriptsize
\resizebox{\textwidth}{!}{%
\begin{tabular}{lrrrrrrrrr}
\toprule
Variant & Overload [\#] & Arrival Delay [h] & Sector-Number [\#] & Sector-Diff [\#] & Reroute [\#] & Reconfig [\#] & Execution-Time [s] & Grounding-Time [s] & RAM-Usage [MB] \\
\midrule
CASA & {\color{blue}$0.00 \pm 0.00$} & {\color{red}$2353.12 \pm 185.87$} & $1909.16 \pm 83.14$ & $0.00 \pm 0.00$ & $45.31 \pm 1.99$ & $0.00 \pm 0.00$ & $7.09 \pm 0.16$ & -- & $214.59 \pm 0.79$ \\
ASP ($r_p, d, s_p$) & {\color{red}$23.27 \pm 2.98$} & {\color{blue}$211.44 \pm 15.19$} & $1048.75 \pm 50.01$ & $434.03 \pm 26.29$ & $46.59 \pm 2.13$ & $858.01 \pm 49.01$ & $1800.20 \pm 0.00$ & $2.10 \pm 0.10$ & $487.27 \pm 13.72$ \\
ASP ($r_p, d_p, s_p$) & $29.55 \pm 3.78$ & $82.01 \pm 5.21$ & $1041.72 \pm 49.92$ & $441.54 \pm 25.77$ & $44.65 \pm 2.09$ & $838.51 \pm 48.10$ & $1800.19 \pm 0.00$ & $0.44 \pm 0.02$ & $291.44 \pm 3.82$ \\
ASP ($\neg r, d, s_p$) & $30.11 \pm 3.33$ & $244.55 \pm 14.87$ & $1046.53 \pm 49.98$ & $451.13 \pm 25.80$ & $46.95 \pm 2.13$ & $856.70 \pm 48.86$ & $1800.21 \pm 0.00$ & $1.54 \pm 0.07$ & $424.18 \pm 9.76$ \\
ASP ($r_p, d, \neg s$) & $32.25 \pm 3.51$ & $313.11 \pm 15.99$ & $448.44 \pm 12.56$ & $0.00 \pm 0.00$ & $48.99 \pm 2.15$ & $0.00 \pm 0.00$ & $1510.90 \pm 45.87$ & $1.63 \pm 0.08$ & $362.03 \pm 10.12$ \\
ASP ($\neg r, d, \neg s$) & $34.97 \pm 3.53$ & $327.61 \pm 16.14$ & $448.29 \pm 12.62$ & $0.00 \pm 0.00$ & $49.57 \pm 2.14$ & $0.00 \pm 0.00$ & $1537.20 \pm 43.97$ & $1.14 \pm 0.06$ & $313.56 \pm 7.59$ \\
ASP ($\neg r, d_p, s_p$) & $40.46 \pm 4.12$ & $94.37 \pm 5.01$ & $1043.82 \pm 51.06$ & $445.25 \pm 24.93$ & $44.53 \pm 2.02$ & $840.92 \pm 48.76$ & $1800.20 \pm 0.00$ & $0.38 \pm 0.02$ & $280.62 \pm 3.33$ \\
ASP ($r_p, \neg d, s_p$) & $45.58 \pm 3.96$ & $0.00 \pm 0.00$ & $986.67 \pm 46.46$ & $479.08 \pm 28.64$ & $17.84 \pm 0.96$ & $784.60 \pm 44.74$ & $1800.18 \pm 0.00$ & $0.27 \pm 0.02$ & $261.24 \pm 3.01$ \\
ASP ($r_p, d, s$) & $47.44 \pm 3.98$ & $273.53 \pm 15.31$ & $391.00 \pm 10.57$ & $693.67 \pm 41.36$ & $51.67 \pm 2.06$ & $983.91 \pm 56.89$ & $1800.24 \pm 0.00$ & $6.10 \pm 0.40$ & $965.09 \pm 45.63$ \\
MIP & $48.67 \pm 4.43$ & $128.52 \pm 6.02$ & $3619.49 \pm 112.30$ & $0.00 \pm 0.00$ & $30.35 \pm 1.08$ & $0.00 \pm 0.00$ & $1229.52 \pm 55.54$ & -- & $269.51 \pm 2.22$ \\
ASP ($\neg r, d, s$) & $52.25 \pm 4.07$ & $297.90 \pm 16.15$ & $387.28 \pm 10.83$ & $683.91 \pm 39.96$ & $52.03 \pm 2.17$ & $1000.03 \pm 56.25$ & $1800.25 \pm 0.00$ & $5.14 \pm 0.34$ & $860.20 \pm 38.15$ \\
ASP ($r_p, d_p, s$) & $56.05 \pm 4.69$ & $110.06 \pm 5.57$ & $398.79 \pm 11.02$ & $698.25 \pm 40.38$ & $52.17 \pm 2.13$ & $985.38 \pm 56.56$ & $1800.22 \pm 0.00$ & $3.44 \pm 0.31$ & $631.39 \pm 30.19$ \\
ASP ($\neg r, \neg d, s_p$) & $63.02 \pm 4.32$ & $0.00 \pm 0.00$ & $944.65 \pm 44.78$ & $448.78 \pm 30.93$ & $0.00 \pm 0.00$ & $708.00 \pm 44.16$ & $1800.17 \pm 0.00$ & $0.24 \pm 0.01$ & $248.93 \pm 3.42$ \\
ASP ($\neg r, d_p, s$) & $65.50 \pm 4.85$ & $121.61 \pm 5.72$ & $392.37 \pm 10.90$ & $672.52 \pm 39.77$ & $51.08 \pm 2.14$ & $986.70 \pm 56.45$ & $1800.22 \pm 0.00$ & $3.34 \pm 0.30$ & $626.46 \pm 30.18$ \\
ASP ($r_p, \neg d, s$) & $66.90 \pm 4.62$ & $0.00 \pm 0.00$ & $381.69 \pm 10.88$ & $665.29 \pm 38.59$ & $26.53 \pm 1.38$ & $1010.56 \pm 57.23$ & $1794.55 \pm 5.67$ & $3.15 \pm 0.30$ & $589.37 \pm 29.12$ \\
ASP ($r_p, d_p, \neg s$) & $68.72 \pm 4.74$ & $94.05 \pm 5.25$ & $448.44 \pm 12.56$ & $0.00 \pm 0.00$ & $40.27 \pm 2.18$ & $0.00 \pm 0.00$ & $1398.07 \pm 52.50$ & $0.16 \pm 0.01$ & $202.86 \pm 2.03$ \\
ASP ($r, \neg d, s_p$) & $68.98 \pm 5.33$ & $21.90 \pm 3.56$ & $774.60 \pm 30.21$ & $223.72 \pm 21.75$ & $12.89 \pm 1.30$ & $619.86 \pm 32.96$ & $1795.13 \pm 5.33$ & $64.05 \pm 6.61$ & $5511.59 \pm 527.30$ \\
ASP ($\neg r, d_p, \neg s$) & $74.80 \pm 4.94$ & $84.88 \pm 4.28$ & $448.44 \pm 12.56$ & $0.00 \pm 0.00$ & $35.99 \pm 1.90$ & $0.00 \pm 0.00$ & $1437.50 \pm 50.67$ & $0.13 \pm 0.00$ & $196.70 \pm 1.82$ \\
ASP ($\neg r, \neg d, s$) & $80.83 \pm 4.81$ & $0.00 \pm 0.00$ & $386.35 \pm 12.10$ & $528.17 \pm 29.14$ & $0.00 \pm 0.00$ & $856.79 \pm 43.54$ & $1800.21 \pm 0.00$ & $3.10 \pm 0.29$ & $566.61 \pm 26.05$ \\
ASP ($r, \neg d, \neg s$) & $84.84 \pm 5.48$ & $38.55 \pm 4.37$ & $448.44 \pm 12.56$ & $0.00 \pm 0.00$ & $13.57 \pm 1.38$ & $0.00 \pm 0.00$ & $1489.11 \pm 46.79$ & $62.19 \pm 6.55$ & $4960.60 \pm 487.43$ \\
ASP ($r_p, \neg d, \neg s$) & $84.90 \pm 4.94$ & $0.00 \pm 0.00$ & $448.44 \pm 12.56$ & $0.00 \pm 0.00$ & $8.23 \pm 0.57$ & $0.00 \pm 0.00$ & $726.02 \pm 60.64$ & $0.05 \pm 0.00$ & $161.50 \pm 1.06$ \\
ASP ($r, d, s_p$) & $87.90 \pm 7.27$ & $173.85 \pm 12.08$ & $777.90 \pm 36.72$ & $396.43 \pm 28.12$ & $44.52 \pm 2.34$ & $590.65 \pm 38.20$ & $1795.14 \pm 5.35$ & $30.92 \pm 2.85$ & $5779.65 \pm 540.39$ \\
Initial & $92.39 \pm 5.07$ & $0.00 \pm 0.00$ & $448.44 \pm 12.56$ & $0.00 \pm 0.00$ & $0.00 \pm 0.00$ & $0.00 \pm 0.00$ & $0.00 \pm 0.00$ & -- & $0.00 \pm 0.00$ \\
ASP ($\neg r, \neg d, \neg s$) & $92.39 \pm 5.07$ & $0.00 \pm 0.00$ & $448.44 \pm 12.56$ & $0.00 \pm 0.00$ & $0.00 \pm 0.00$ & $0.00 \pm 0.00$ & $6.01 \pm 0.67$ & $0.03 \pm 0.00$ & $144.03 \pm 0.24$ \\
ASP ($r, d_p, s_p$) & $96.32 \pm 7.89$ & $155.79 \pm 10.14$ & $781.05 \pm 33.76$ & $397.17 \pm 26.98$ & $47.29 \pm 2.35$ & $598.31 \pm 35.86$ & $1795.12 \pm 5.36$ & $33.62 \pm 2.91$ & $5827.82 \pm 551.50$ \\
ASP ($r, \neg d, s$) & $103.84 \pm 6.43$ & $37.56 \pm 4.38$ & $346.33 \pm 7.65$ & $359.92 \pm 26.70$ & $16.07 \pm 1.54$ & $864.26 \pm 56.30$ & $1762.82 \pm 14.07$ & $64.39 \pm 6.31$ & $6802.62 \pm 627.52$ \\
ASP ($r, d, s$) & $112.29 \pm 8.93$ & $180.80 \pm 11.80$ & $334.50 \pm 8.72$ & $512.28 \pm 35.30$ & $42.72 \pm 2.13$ & $665.04 \pm 47.28$ & $1762.68 \pm 14.13$ & $27.42 \pm 2.30$ & $6919.20 \pm 633.03$ \\
ASP ($r, d_p, \neg s$) & $124.30 \pm 8.30$ & $158.29 \pm 8.21$ & $438.62 \pm 12.08$ & $0.00 \pm 0.00$ & $50.45 \pm 2.21$ & $0.00 \pm 0.00$ & $1515.45 \pm 45.23$ & $54.26 \pm 5.58$ & $5041.34 \pm 493.87$ \\
ASP ($r, d, \neg s$) & $126.64 \pm 8.58$ & $209.50 \pm 12.51$ & $431.47 \pm 11.68$ & $0.00 \pm 0.00$ & $49.16 \pm 2.19$ & $0.00 \pm 0.00$ & $1528.53 \pm 44.31$ & $48.19 \pm 4.73$ & $5056.61 \pm 493.95$ \\
ASP ($r, d_p, s$) & $132.92 \pm 8.79$ & $151.88 \pm 8.92$ & $354.54 \pm 8.27$ & $522.28 \pm 39.89$ & $43.37 \pm 2.24$ & $794.71 \pm 53.09$ & $1762.79 \pm 14.08$ & $55.65 \pm 6.18$ & $6910.15 \pm 632.98$ \\
\bottomrule
\end{tabular}
}
\caption{Results aggregated by variant (across all scenarios and instances). $\pm$ denotes the SEM. Variants are sorted lexicographically.}
\vspace{-0.8cm}
\label{tab:summary-by-variant-main}
\end{table}

\begin{table}[t]
\centering
\scriptsize
\resizebox{\textwidth}{!}{%
\begin{tabular}{lrrrrrrrrr}
\toprule
Option & Overload [\#] & Arrival Delay [h] & Sector-Number [\#] & Sector-Diff [\#] & Reroute [\#] & Reconfig [\#] & Execution-Time [s] & Grounding-Time [s] & RAM-Usage [MB] \\
\midrule
$r$ & $103.89 \pm 2.53$ & $121.59 \pm 3.37$ & $520.51 \pm 8.27$ & $253.88 \pm 9.34$ & $34.82 \pm 0.76$ & $445.11 \pm 14.51$ & $1689.64 \pm 9.67$ & $50.09 \pm 1.82$ & $5867.73 \pm 186.23$ \\
$r_p$ & $50.51 \pm 1.46$ & $120.53 \pm 4.12$ & $621.68 \pm 12.00$ & $378.94 \pm 11.45$ & $37.44 \pm 0.71$ & $606.55 \pm 17.52$ & $1603.40 \pm 13.05$ & $1.93 \pm 0.08$ & $439.13 \pm 9.16$ \\
$\neg r$ & $59.38 \pm 1.53$ & $129.99 \pm 4.28$ & $616.33 \pm 11.93$ & $359.06 \pm 10.81$ & $31.12 \pm 0.78$ & $583.56 \pm 16.80$ & $1531.33 \pm 15.05$ & $1.67 \pm 0.07$ & $406.81 \pm 8.18$ \\
\addlinespace[0.35em]
\midrule
\addlinespace[0.15em]
d & $58.36 \pm 1.95$ & $251.73 \pm 5.11$ & $595.19 \pm 11.62$ & $348.76 \pm 11.15$ & $48.25 \pm 0.72$ & $549.46 \pm 16.92$ & $1703.93 \pm 9.25$ & $12.86 \pm 0.74$ & $2351.98 \pm 123.41$ \\
$d_p$ & $75.33 \pm 2.12$ & $115.81 \pm 2.29$ & $594.91 \pm 11.35$ & $351.10 \pm 11.13$ & $45.52 \pm 0.72$ & $558.44 \pm 16.88$ & $1678.86 \pm 10.41$ & $15.95 \pm 1.07$ & $2223.20 \pm 124.97$ \\
$\neg d$ & $76.71 \pm 1.72$ & $10.79 \pm 0.87$ & $574.84 \pm 10.25$ & $300.16 \pm 9.86$ & $10.54 \pm 0.39$ & $536.65 \pm 15.80$ & $1441.58 \pm 16.73$ & $21.76 \pm 1.42$ & $2138.50 \pm 122.31$ \\
\addlinespace[0.35em]
\midrule
\addlinespace[0.15em]
s & $78.01 \pm 2.01$ & $129.04 \pm 4.00$ & $376.34 \pm 3.49$ & $597.17 \pm 12.71$ & $37.14 \pm 0.76$ & $913.98 \pm 18.29$ & $1787.11 \pm 2.81$ & $18.17 \pm 1.12$ & $2763.45 \pm 139.60$ \\
$s_p$ & $52.56 \pm 1.71$ & $107.35 \pm 3.68$ & $944.18 \pm 15.21$ & $413.70 \pm 9.04$ & $33.51 \pm 0.74$ & $749.38 \pm 14.96$ & $1798.50 \pm 1.03$ & $14.20 \pm 1.00$ & $2123.64 \pm 119.68$ \\
$\neg s$ & $80.09 \pm 2.02$ & $135.68 \pm 4.20$ & $445.54 \pm 4.14$ & $0.00 \pm 0.00$ & $32.77 \pm 0.75$ & $0.00 \pm 0.00$ & $1238.76 \pm 19.35$ & $18.40 \pm 1.24$ & $1826.58 \pm 108.42$ \\
\bottomrule
\end{tabular}
}
\caption{Ablation study aggregated over all scenarios, instances, and ASP variants. $\pm$ shows the standard error of the mean.}
\vspace{-0.5cm}
\label{tab:ablation-aggregated-main}
\end{table}

\smallskip
\noindent
\textbf{Results and Discussion}.
We evaluate the suitability of \emph{ASP} for ATFCM (Tables~\ref{tab:summary-by-variant-main} and~\ref{tab:ablation-aggregated-main}; Figures~\ref{fig:heatmap-indx-4x10} and~\ref{fig:ttf-ttb-scatter}). 
\begin{itemize}
    \item \textbf{Lexicographic Performance:}
    CASA achieves strict feasibility ($0$ overload) but incurs massive arrival delay ($2353.12$ h).
    In contrast, \emph{ASP} ($r_p,d,s_p$) provides a superior operational trade-off, achieving an overload of $23.27$ with $211.44$ h delay.
    Arguably, \emph{ASP} ($r_p, d_p, s_p$) offers an even stronger operational balance; while its overload is slightly higher ($29.55$), its arrival delay ($82.01$ h) is strictly lower than the \emph{MIP} baseline ($48.67$ overload, $128.52$ h delay).
    All highlighted variants significantly outperform the unmodified Initial state ($92.39$ overload).
    \item \textbf{Computational Bottlenecks:}
    For heuristically bounded variants, solving time dominates the computational effort, leading to TIMEOUTs on larger instances.
    However, enabling unrestricted rerouting ($r$) triggers a grounding bottleneck.
    This manifests in drastically larger grounding times and a severe memory explosion (averaging $5867.73$ MB and $50.09 s$ grounding time) compared to partial ($r_p$, $439.13$ MB and $1.93 s$ grounding time) or no rerouting ($\neg r$, $406.81$ MB and $1.67s$ grounding time).
\end{itemize}

These results demonstrate that shifting spatial routing and sectorization from the solver to a heuristic pre-processor is critical to bypass memory and time bottlenecks.
We attribute the performance discrepancy between our exact \emph{ASP} models and reported \emph{MIP} literature to our looser model definitions and the treatment of rerouting alternatives as dynamic rather than static data.

\smallskip
\noindent
\textbf{TTF/TTB Analysis}.
Figure~\ref{fig:ttf-ttb-scatter} visualizes the search dynamics by plotting Time to First Solution (TTF) against Time to Best Solution (TTB).
We quantify the optimization phase using the median time gap $\Delta t = \text{TTB} - \text{TTF}$ and the median percentage reduction in initial overload $o_m$.
\begin{itemize}
    \item \textbf{Search Space Thrashing:} The unbounded variant \emph{ASP} ($r,d,s$) struggles to find initial feasibility and frequently terminates without improving the initial bound (clustering on the diagonal).
    It yields the lowest optimization impact ($\Delta t = 878.74$s, $o_m = 34.10\%$).
    \item \textbf{Fast Feasibility:}
    Tightly constrained variants rapidly establish initial upper bounds, efficiently utilizing the remaining budget to traverse the search space ($\Delta t > 1500$s, $o_m > 76\%$).
    Visual clustering differentiates their behavior: \emph{ASP} ($r_p, d_p, s_p$) strictly clusters in the top-left (extremely fast TTF, sustained late TTB), whereas \emph{ASP} ($\neg r, d, \neg s$) clusters primarily in the center-top and left-mid,
    indicating earlier search space exhaustion.
    \item \textbf{Adequate Convergence:}
    The globally best-performing variant, \emph{ASP} ($r_p, d, s_p$), is bounded in the scatter plot by the extremes of the other methods.
    It requires slightly more time to establish initial feasibility than fully constrained models but sustains a robust optimization phase ($\Delta t = 1644.49$s, $o_m = 74.88\%$). 
\end{itemize}

\noindent
\textbf{Open Challenge:} Frequent timeouts of ASP shows that exact solvers cannot yet prove optimality for large instances within operational time limits.

\begin{figure}[t]
\centering
    \hspace{-4.5em}
    \begin{subfigure}[t]{0.45\linewidth}
        \centering
        \includegraphics[width=5.50cm]{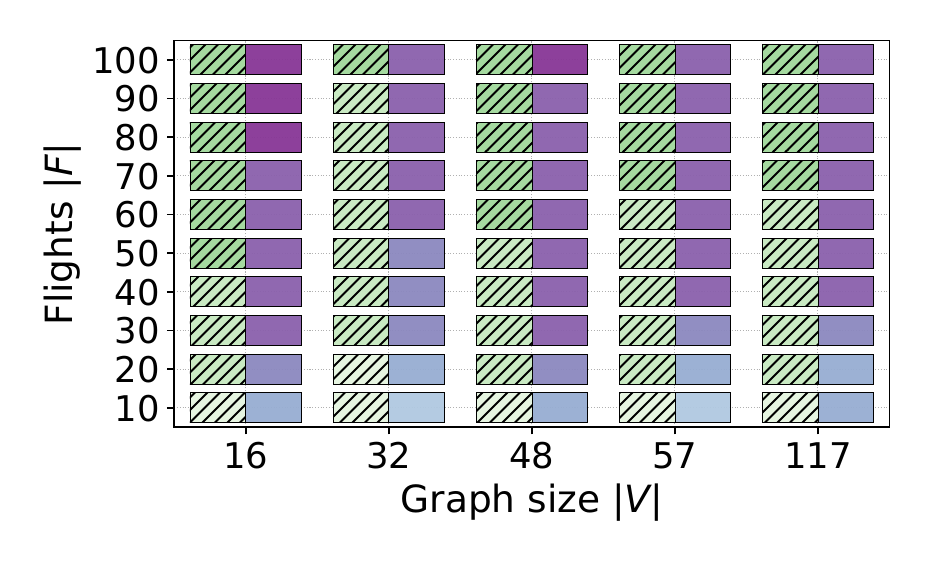}
    \end{subfigure}
    \hspace{-1.3em}
    \begin{subfigure}[t]{0.45\linewidth}
        \centering
        \includegraphics[width=5.50cm]{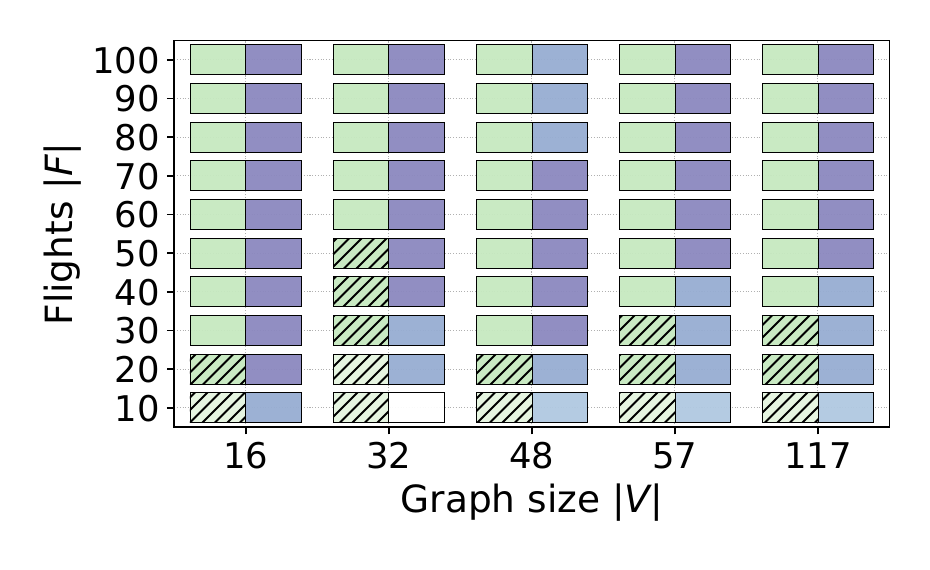}
    \end{subfigure}
    \hspace{-1.3em}
    \begin{subfigure}[t]{0.09\linewidth}
        \centering
        \includegraphics[width=2.40cm]{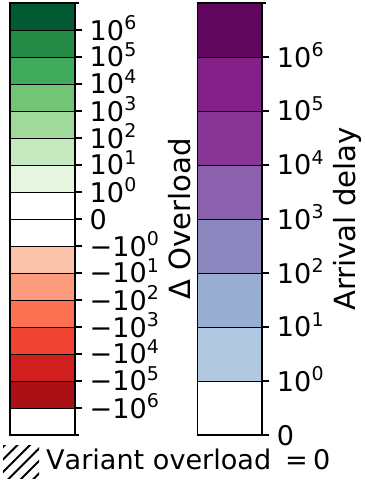}
    \end{subfigure}

    \hspace{-4.5em}
    \begin{subfigure}[t]{0.45\linewidth}
        \centering
        \includegraphics[width=5.50cm]{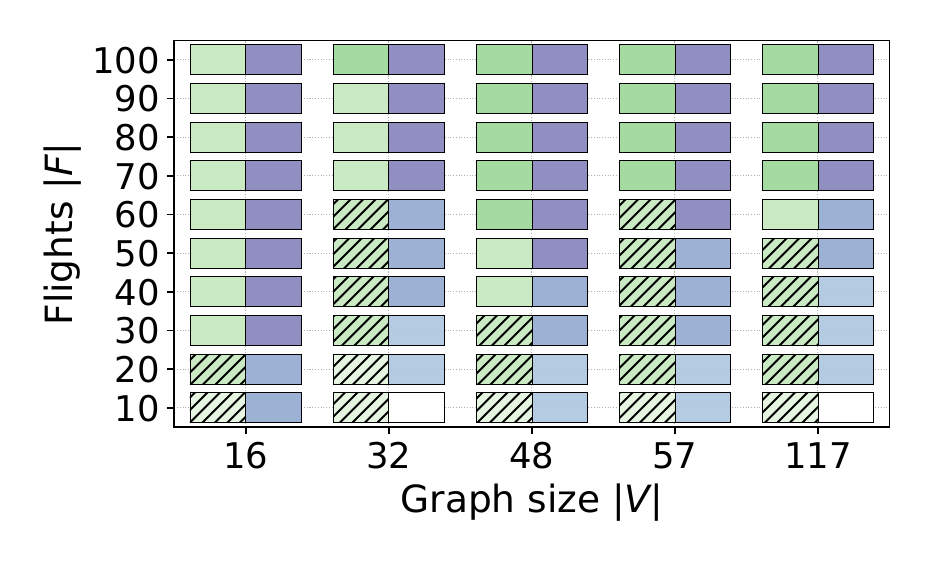}
    \end{subfigure}
    \hspace{-1.3em}
    \begin{subfigure}[t]{0.45\linewidth}
        \centering
        \includegraphics[width=5.50cm]{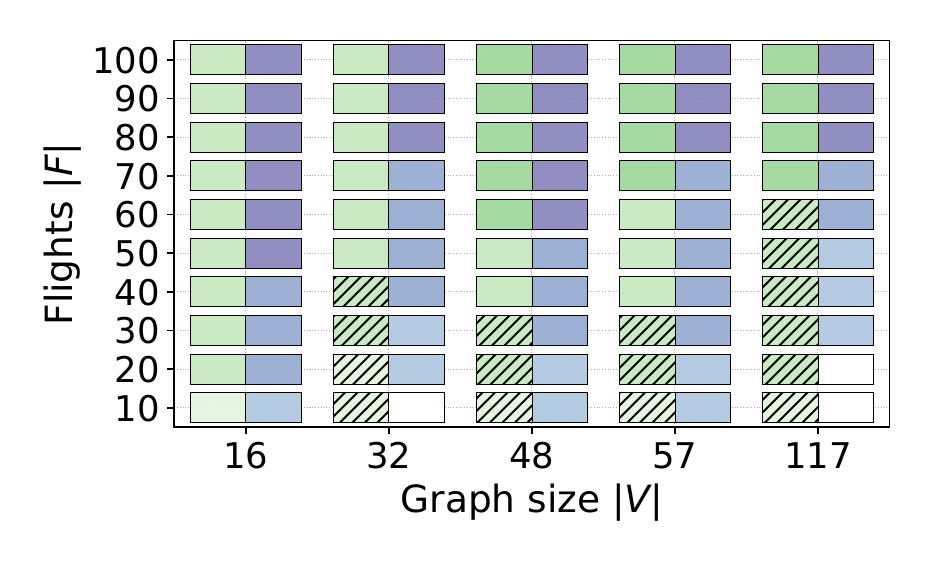}
    \end{subfigure}
    \hspace{-1.3em}
    \begin{subfigure}[t]{0.09\linewidth}
        \centering
        \includegraphics[width=2.40cm]{imgs/03_scaling_plots/20260505__variant_id-03_DELAY__capacity-NA__legend.pdf}
    \end{subfigure}

    \vspace{-.75em}
    \caption{
    Scaling results of CASA (top left), MIP (top right), ASP  ($r_p,d,s_p$) (bottom left), and ASP ($r_p, d_p, s_p$) (bottom right).
    X-axis shows size of the considered geographical instance ($|V|$) and Y-axis shows number of flights.
    Each data point has three instances, where the hatching is shown iff all seeds achieve 0 overload.
    }
    \label{fig:heatmap-indx-4x10}
    \vspace{-0.5cm}
\end{figure}

\section{Conclusion}\label{sec:concl}
We present a unified and extensible formalization of Air Traffic Flow and Capacity Management (ATFCM) that integrates flow management decisions and dynamic airspace configuration (DAC). 
Unlike existing formulations that treat these components in isolation or assume pre-defined sector changes, our model captures delaying, rerouting, and DAC within a single navpoint-graph-based framework.
We develop an ASP encoding to investigate the feasibility of ASP in the ATFCM domain.
Our ASP encoding is able to freely switch between full/partial/disabled actions, which enable the conduction of an ablation study, where we further compare our encoding to a simulation of the practical CASA algorithm and a simulation of a SOTA MIP model.
Our results indicate that ASP with partial sectorization is able to achieve strong results, especially for smaller instance sizes.
When the instance sizes grow larger, the computational budget does not suffice for ASP to finish the computation, which leads to the result that the CASA heuristics is the only tested variant able to solve all benchmarked instances to 0-overload --- however, with the downside of having low efficiency (high delay), in comparison to ASP's solutions.
Our ablation study of the ASP variants indicates, that the size of the search space should be ``just big enough'', as restrictions of the search space to a fixed number of alternatives lead to the best results.
Several directions are promising for future work.
First, the lack of theoretical analysis is a hindrance for abstract modelling, where we plan to analyze which modelling choices drive computational hardness.
Second, we will investigate the combination of ASP into a hybrid heuristic solver for larger instances.
Third, due to the inherent safety criticality of the domain, we will enable operational explainability by generating human-interpretable justifications for delays, reroutes, and reconfigurations.

\begin{figure}[t]
    \centering
    \includegraphics[width=12cm]{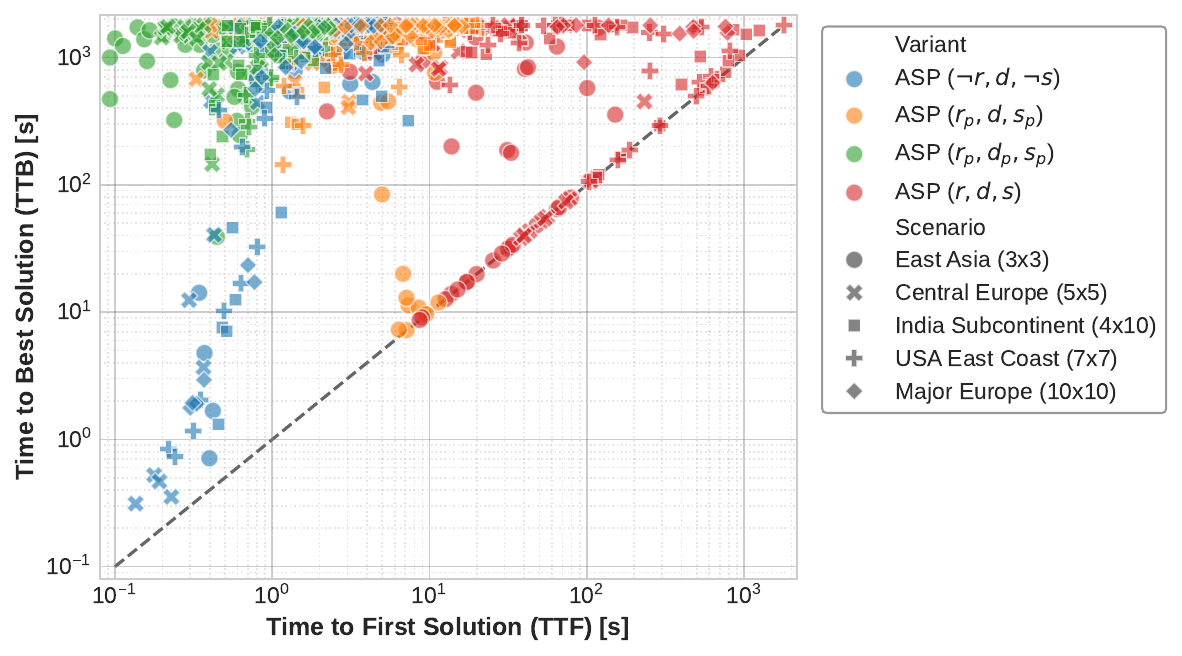}
    \vspace{-0.3cm}
    \caption{Time to First Solution [TTF] vs. Time to Best Solution (TTB) scatter plot comparing four selected ASP variants.}
    \label{fig:ttf-ttb-scatter}
    \vspace{-0.5cm}
\end{figure}

%
%
%
%
{\small
\bibliographystyle{splncs04}
\bibliography{kr-sample-new}
}

\clearpage
\appendix
\include{appendix}

\end{document}

%% file: definitions.tex
\usepackage{amsmath}    
\usepackage{amssymb}    
\usepackage{mathtools}  
\usepackage{microtype}  
\usepackage[inline]{enumitem} 
\usepackage{multirow}   
\usepackage{booktabs}   
\usepackage{tabularx}

\usepackage[dvipsnames,table]{xcolor} 
\usepackage{nag}        
\usepackage{morewrites} 

\usepackage{booktabs}   
\usepackage{graphicx}   
\usepackage{longtable} 

\usepackage{tikz}
\usepackage{calc}
\usetikzlibrary{calc}
\usepackage{svg}

\usepackage{xspace}
\usepackage{color}

\usepackage[utf8]{inputenc}
\usepackage{booktabs}
\usepackage{amssymb}
\usepackage[switch]{lineno}
\usepackage{amsmath}
\usepackage{makecell}
\usepackage{caption}
\usepackage{subcaption}
\usepackage{listings}
\usepackage{tikz}
\usepackage{multirow}

\usepackage{lscape}
\usepackage{longtable}

\usepackage{array}
\usepackage{mathtools}
\usepackage{cancel}

\usepackage{tikz}
\usetikzlibrary{matrix}

\tikzset{ 
	table/.style={
		matrix of nodes,
		row sep=-\pgflinewidth,
		column sep=-\pgflinewidth,
		nodes={rectangle,text width=3.5cm,align=center,
			anchor=center},
		text depth=1.25ex,
		text height=2ex,
		nodes in empty cells
	},
	row 1/.style={nodes={minimum height=3.5em}},
		column 1/.style={nodes={text width=4.5cm}},
}

\usepackage{placeins} 

\newcommand{\ATFCMoDELAY}{\ensuremath{K_d}\xspace}
\newcommand{\ATFCMoNUMSEC}{\ensuremath{K_{\#s}}\xspace}
\newcommand{\ATFCMoSECCHANGES}{\ensuremath{K_{ds}}\xspace}
\newcommand{\ATFCMoREROUTED}{\ensuremath{K_{r}}\xspace}
\newcommand{\ATFCMoRECONFIG}{\ensuremath{K_{c}}\xspace}

\newcommand{\longversion}[1]{}

\newcolumntype{H}{>{\setbox0=\hbox\bgroup}c<{\egroup}@{}}

\allowdisplaybreaks

\newcommand{\vecv}[0]{\mathbf}

\newcounter{myenumctr}

\usetikzlibrary{matrix}

\tikzset{ 
	table/.style={
		row sep=-\pgflinewidth,
		column sep=-\pgflinewidth,
		nodes={rectangle,
			align=center,
				anchor=center},
	},
	row 1/.style={nodes={
		}},
}

\DeclareMathAlphabet\mathbfcal{OMS}{cmsy}{b}{n}

\pgfdeclarelayer{bg}    
\pgfsetlayers{bg,main}

\newtheorem{definition}{Definition}
\newtheorem{observation}{Observation}
\newtheorem{corollary}{Corollary}

\definecolor{backcolor}{rgb}{0.95,0.95,0.95}

\makeatletter
\providecommand{\leftsquigarrow}{%
  \mathrel{\mathpalette\reflect@squig\relax}%
}
\newcommand{\reflect@squig}[2]{%
  \reflectbox{$\m@th#1\rightsquigarrow$}%
}
\makeatother

\lstdefinestyle{mystyle}{
    backgroundcolor=\color{backcolor},
    breaklines=true,
    numbers=left,
    numbersep=5pt,
    basicstyle=\ttfamily\scriptsize,
    captionpos=b,
    literate={:-}{{$\;\leftarrow$}}1
             {!=}{{$\neq$}}1
             {:~}{{$\;\leftsquigarrow$}}1,
    mathescape=true
}

\lstdefinestyle{examplestyle}{
    backgroundcolor=\color{backcolor},
    breaklines=true,
    numbers=none,
    basicstyle=\ttfamily\tiny,
    captionpos=b,
    literate={:-}{{$\;\leftarrow$}}1
             {!=}{{$\neq$}}1,
    mathescape=true,
    aboveskip=2pt,
    belowskip=2pt
}

%% file: appendix.tex
\section{Data Generator: Details}

Let $A = \textit{airport(G)}$ be the airports, then each airport $a \in A$ has for a timestep $t_b \in T$ a sending strength $\lambda_{a,t_b}$.
Let $d$ be a day of a number of specified days $d \in D$, then the actual number of departures is $\#\textit{departures}(a,t,d)$, and the actual number of departures at $a$ with destination $a_d \in A$ is $\#\textit{dest}(a,t,d,a_d)$.
The probabilistic model is created by aggregating the departures into a matrix $\Lambda \in \mathbb{N}^{|A| \times |T|}$, where $\lambda_{a,t} = \Lambda_{[a,t]} = \frac{\sum_{d \in D} \#\textit{departures}(a,t,d)}{|D|}$.
For data generation, we sample the number of departures for an airport $a_o \in A$, and for a timestep $t\in T$, from $\#a_o \sim \text{Pois}(\lambda_{a_o,t})$.
It remains to generate for each departure at airport $a_o$, a destination airport $a_d$:
Let $\Psi \in \mathbb{N}^{|A| \times |A| \times |T|}$,
where $\psi_{a_o,a_d,t} = \Psi_{[a_o,a_d,t]} = \frac{\sum_{d \in D} \#\textit{dest}(a_o,t,d,a_d)}{|D|}$.
Further, let $a_o \in A$ and $t \in T$, then the probability for choosing destination $a_d \in A$ is $P(a_d \mid a_o,t) = \frac{\psi_{a_o,a_d,t} + \alpha}{\sum_{a \in A} \psi_{a_o,a_d,t} + \alpha}$,
where $\alpha \in \mathbb{R}$ is a small constant.
Destinations are drawn as a categorical random variable:
$a_d \sim \textit{Categorical}\left( P(a_d \mid a_o,t)_{a_d \in A} \right)$.

\begin{table*}[t]
\centering
\tiny
\resizebox{\textwidth}{!}{%
\begin{tabular}{lrrrrrrrrrrrr}
\toprule
Scenario & \#I & \multicolumn{2}{c}{Initial} & \multicolumn{2}{c}{CASA} & \multicolumn{2}{c}{MIP} & \multicolumn{2}{c}{$ASP_{\neg r, \neg d, \neg s}$} & \multicolumn{2}{c}{$ASP_{\neg r, \neg d, s_p}$} & Draws \\
 &  & \#W & \#S & \#W & \#S & \#W & \#S & \#W & \#S & \#W & \#S & \#D \\
\midrule
CE-5x5 & 40 & 0 & 1 & \textbf{10} & \textbf{40} & 0 & 23 & 0 & 1 & 0 & 3 & 0 \\
EA-3x3 & 40 & 0 & 0 & \textbf{28} & \textbf{40} & 0 & 11 & 0 & 0 & 0 & 0 & 1 \\
IND-4x10 & 40 & 0 & 0 & \textbf{23} & \textbf{40} & 0 & 10 & 0 & 0 & 0 & 0 & 0 \\
EUR-10x10 & 40 & 0 & 0 & 8 & \textbf{40} & 0 & 13 & 0 & 0 & 0 & 1 & 0 \\
USA-7x7 & 40 & 0 & 0 & \textbf{15} & \textbf{40} & 0 & 12 & 0 & 0 & 0 & 0 & 0 \\
\midrule
Overall & 200 & 0 & 1 & \textbf{84} & \textbf{200} & 0 & 69 & 0 & 1 & 0 & 4 & 1 \\
\addlinespace[0.6em]
Scenario & \#I & \multicolumn{2}{c}{$ASP_{\neg r, \neg d, s}$} & \multicolumn{2}{c}{$ASP_{r_p, \neg d, \neg s}$} & \multicolumn{2}{c}{$ASP_{r_p, \neg d, s_p}$} & \multicolumn{2}{c}{$ASP_{r_p, \neg d, s}$} & \multicolumn{2}{c}{$ASP_{r, \neg d, \neg s}$} & Draws \\
 &  & \#W & \#S & \#W & \#S & \#W & \#S & \#W & \#S & \#W & \#S & \#D \\
\midrule
CE-5x5 & 40 & 0 & 3 & 0 & 2 & 0 & 4 & 2 & 4 & 0 & 8 & 0 \\
EA-3x3 & 40 & 0 & 0 & 0 & 0 & 0 & 0 & 0 & 0 & 0 & 4 & 1 \\
IND-4x10 & 40 & 0 & 0 & 0 & 0 & 0 & 1 & 0 & 1 & 0 & 4 & 0 \\
EUR-10x10 & 40 & 0 & 1 & 0 & 0 & 3 & 6 & 2 & 3 & 0 & 9 & 0 \\
USA-7x7 & 40 & 0 & 0 & 0 & 0 & 0 & 1 & 0 & 1 & 0 & 5 & 0 \\
\midrule
Overall & 200 & 0 & 4 & 0 & 2 & 3 & 12 & 4 & 9 & 0 & 30 & 1 \\
\addlinespace[0.6em]
Scenario & \#I & \multicolumn{2}{c}{$ASP_{r, \neg d, s_p}$} & \multicolumn{2}{c}{$ASP_{r, \neg d, s}$} & \multicolumn{2}{c}{$ASP_{\neg r, d_p, \neg s}$} & \multicolumn{2}{c}{$ASP_{\neg r, d_p, s_p}$} & \multicolumn{2}{c}{$ASP_{\neg r, d_p, s}$} & Draws \\
 &  & \#W & \#S & \#W & \#S & \#W & \#S & \#W & \#S & \#W & \#S & \#D \\
\midrule
CE-5x5 & 40 & 5 & 12 & 1 & 11 & 0 & 10 & 0 & 14 & 1 & 14 & 0 \\
EA-3x3 & 40 & 0 & 4 & 0 & 4 & 0 & 3 & 0 & 3 & 1 & 3 & 1 \\
IND-4x10 & 40 & 0 & 8 & 2 & 7 & 0 & 3 & 0 & 9 & 0 & 8 & 0 \\
EUR-10x10 & 40 & 0 & 11 & 0 & 5 & 0 & 4 & 0 & 21 & 0 & 6 & 0 \\
USA-7x7 & 40 & 2 & 13 & 2 & 8 & 0 & 5 & 0 & 14 & 1 & 11 & 0 \\
\midrule
Overall & 200 & 7 & 48 & 5 & 35 & 0 & 25 & 0 & 61 & 3 & 42 & 1 \\
\addlinespace[0.6em]
Scenario & \#I & \multicolumn{2}{c}{$ASP_{r_p, d_p, \neg s}$} & \multicolumn{2}{c}{$ASP_{r_p, d_p, s_p}$} & \multicolumn{2}{c}{$ASP_{r_p, d_p, s}$} & \multicolumn{2}{c}{$ASP_{r, d_p, \neg s}$} & \multicolumn{2}{c}{$ASP_{r, d_p, s_p}$} & Draws \\
 &  & \#W & \#S & \#W & \#S & \#W & \#S & \#W & \#S & \#W & \#S & \#D \\
\midrule
CE-5x5 & 40 & 0 & 12 & 1 & 22 & 1 & 17 & 0 & 11 & 9 & 25 & 0 \\
EA-3x3 & 40 & 0 & 3 & 0 & 3 & 0 & 2 & 0 & 4 & 1 & 6 & 1 \\
IND-4x10 & 40 & 0 & 4 & 2 & 13 & 6 & 11 & 0 & 5 & 1 & 9 & 0 \\
EUR-10x10 & 40 & 0 & 8 & \textbf{24} & 31 & 0 & 8 & 0 & 9 & 0 & 8 & 0 \\
USA-7x7 & 40 & 0 & 6 & 7 & 19 & 3 & 17 & 0 & 6 & 0 & 10 & 0 \\
\midrule
Overall & 200 & 0 & 33 & 34 & 88 & 10 & 55 & 0 & 35 & 11 & 58 & 1 \\
\addlinespace[0.6em]
Scenario & \#I & \multicolumn{2}{c}{$ASP_{r, d_p, s}$} & \multicolumn{2}{c}{$ASP_{\neg r, d, \neg s}$} & \multicolumn{2}{c}{$ASP_{\neg r, d, s_p}$} & \multicolumn{2}{c}{$ASP_{\neg r, d, s}$} & \multicolumn{2}{c}{$ASP_{r_p, d, \neg s}$} & Draws \\
 &  & \#W & \#S & \#W & \#S & \#W & \#S & \#W & \#S & \#W & \#S & \#D \\
\midrule
CE-5x5 & 40 & 0 & 12 & 0 & 16 & 0 & 25 & 0 & 17 & 0 & 20 & 0 \\
EA-3x3 & 40 & 1 & 4 & 1 & 11 & 2 & 12 & 0 & 8 & 0 & 9 & 1 \\
IND-4x10 & 40 & 0 & 7 & 0 & 10 & 0 & 11 & 0 & 9 & 0 & 10 & 0 \\
EUR-10x10 & 40 & 0 & 3 & 0 & 19 & 0 & 18 & 0 & 8 & 0 & 24 & 0 \\
USA-7x7 & 40 & 1 & 7 & 0 & 16 & 0 & 19 & 0 & 12 & 0 & 20 & 0 \\
\midrule
Overall & 200 & 2 & 33 & 1 & 72 & 2 & 85 & 0 & 54 & 0 & 83 & 1 \\
\addlinespace[0.6em]
Scenario & \#I & \multicolumn{2}{c}{$ASP_{r_p, d, s_p}$} & \multicolumn{2}{c}{$ASP_{r_p, d, s}$} & \multicolumn{2}{c}{$ASP_{r, d, \neg s}$} & \multicolumn{2}{c}{$ASP_{r, d, s_p}$} & \multicolumn{2}{c}{$ASP_{r, d, s}$} & Draws \\
 &  & \#W & \#S & \#W & \#S & \#W & \#S & \#W & \#S & \#W & \#S & \#D \\
\midrule
CE-5x5 & 40 & 5 & 30 & 0 & 21 & 0 & 12 & 2 & 24 & 3 & 12 & 0 \\
EA-3x3 & 40 & 4 & 11 & 0 & 12 & 0 & 4 & 0 & 6 & 1 & 4 & 1 \\
IND-4x10 & 40 & 4 & 17 & 1 & 12 & 0 & 7 & 1 & 9 & 0 & 7 & 0 \\
EUR-10x10 & 40 & 2 & 26 & 1 & 9 & 0 & 9 & 0 & 8 & 0 & 2 & 0 \\
USA-7x7 & 40 & 6 & 25 & 1 & 16 & 0 & 7 & 1 & 10 & 1 & 7 & 0 \\
\midrule
Overall & 200 & 21 & 109 & 3 & 70 & 0 & 39 & 4 & 57 & 5 & 32 & 1 \\
\bottomrule
\end{tabular}
}
\caption{Tournament wins by scenario (\#W --- lexicographic: overload, delay, sectors, reroutes, reconfigs) and number of solved instances (\#S).}
\label{tab:tournament-wins}
\end{table*}



\section{Additional Notes on ATM}\label{sec:appdxmodel}

\textbf{Air Traffic Management} (ATM).
ATM is a broad field which aims at ensuring safe and efficient air traffic at all times.
In Europe, the primary entitiy responsible for ATM is EUROCONTROL,
while in the USA it is the Federal Aviation Administration (FAA).
We introduce ATFCM along a timeline, which is anchored around a \textit{day of operation} (DOP).
At the DOP, ATCs need to separate aircraft -- which is called the tactical phase.
In \textit{en-route} airspace (cruise phase of a flight) any two aircraft must be separated by at least $5 \; nm$ horizontally, or $1000 \; ft$ vertically~\cite{cook_european_2007,ribeiro_review_2020}.
ATCs operate in a \textit{sector}, which is a defined volume of airspace.
Due to the human nature, ATCs can only effectively handle a certain number of flights in a given time period,
which is the primary reason why sectors have \emph{capacity} constraints.
The exact capacity value is dependent on various factors, ranging from the competence of the ATC staff,
over managerial and technical matters such as ATC shift planning,
used automation aids or sector geography, 
to temporary effects, such as thunderstorms or worker union strikes~\cite{cook_european_2007}.
Evidently, when \emph{demand} - the number of aircraft that want to fly through a sector - exceeds capacity, this puts additional workload on ATCs and therefore increases the likelihood of human error,
with potential catastrophic consequences.
Therefore, it is required that demand remains below capacity at all times.

On the other hand, the responsibility of the strategic level is to ensure that sectors are not overloaded - which is also called \emph{Air Traffic Flow (and Capacity) Management} (ATFCM).
Generally, airlines \emph{file} flights months in advance.
Filing a flight effectively means to register a flight that will in the future start at a departure airport and take a certain route to fly to its destination.
This gives air navigation service providers (ANSPs) and EUROCONTROL time to react to potential imbalances.
For example, ANSPs may conduct managerial decisions, like sending teams of ATCs to another facility, where more demand is expected, or on the very far end of the spectrum training more ATC staff.
Closer to the DOP, the so-called network manager of EUROCONTROL is responsible for mitigating remaining imbalances by taking \emph{regulations} to either reduce demand or increase capacity.
Measures for reducing demand include \emph{delaying} or \emph{rerouting} aircraft.
Effective capacity may be increased by DAC~\cite{lui_airspace_2024,criscuolo_enhanced_2024},
which either restructures sectors to enhance flow,
or opens sectors exactly at positions where they are needed the most.

\section{Detailed Model Definition}

\begin{definition}[Navpoint Graph]
Detailed definition of Definition~\ref{def:navpoint-graph}.
    The navpoint graph $G = (V,E)$ is undirected and labeled.  We require that $V = \mathbb{N}^{\leq |V|}$, i.e., each $v$ is an integer and 
    and $V$ has no holes. 
%
    Each $v \in V$ is either a navpoint or an airport, which we test with $v \in \textit{en-route}(G)$ and $v \in \textit{airport}(G)$.
    Distances between two navpoints are defined as the function $d:V^2 \rightarrow \mathbb{Q}$, which we assume to be computed in $\mathcal{O}(1)$.
\end{definition}

\begin{definition}[Time model, extends Definition~\ref{def:time-model}]
    Time is discretized into timesteps $t_i \in T$, with $T_{\textit{gran}}$ timesteps per hour s.t. $T_{\textit{gran}}$ divides 60
    such that the timestep length is defined in minutes by 
    $T_s = \frac{60}{T_{\textit{gran}}}$.
    $T$ is the set of timesteps, which for a $dur \in \mathbb{N}$ is given by $T = \{t \mid t \in \mathbb{N}, 1 \leq t \leq dur \cdot T_{\textit{gran}}\} \cup \{0\}$.
\end{definition}

\begin{definition}[Sector Model, extends Definition~\ref{def:sector-model}]
A \emph{sector} at time $t$ is a set of vertices (navpoints/airports).  
To keep the index set fixed across time, we represent each sector by a distinguished
\emph{sector identifier} $i \in V$. We write $S = V$ for the set of all identifiers.
The sector identified by $i$ at time $t$ is the set $\textit{sec}(i,t) \subseteq V$.
Formally, a \emph{sector configuration} is a function
$\textit{sec}: S \times T \rightarrow 2^{V}$ such that for every $t \in T$:
(i) $\{\textit{sec}(i,t) \mid i \in S,\ \textit{sec}(i,t)\neq\emptyset\}$
forms a partition of $V$; and
(ii) every airport is a singleton sector.
More precisely, let $t \in T$, then for any $i,j \in S$ s.t. $i \not = j$ it must hold $\textit{sec}(i,t) \cap \textit{sec}(j,t) = \emptyset$
    and
    $\bigcup_{i \in S} \textit{sec}(i,t) = V$.
Conversely, $\textit{sec}^{-1}: V \times T \rightarrow S$ denotes the inverse assignment.
The initial sector configuration at timestep $t_s = 0$ is the atomic configuration, with atomic capacities $\textit{cap}(i,0) = c_i$.
\end{definition}

\begin{definition}[Composite Sector Model]
    The capacity of \emph{composite} sectors for $t > 0$ is defined by $\textit{cap}':S \times \mathbf{T} \rightarrow \mathbb{N}$,
    where we let for $i \in S$: $\textit{cap}'(i,t) = \max_{v \in \mathbf{SEC}(i,t)}\{ \textit{cap}(v,0) \}$.
    The capacity $\textit{cap}'$ is measured in $1$-hour timeframes and distributed in the following way.
    For each $1$-hour timeframe, there are $T_{\textit{gran}}$ timesteps, each with capacitiy $\textit{cap}(i,t)$.
    Intuitively, the following function distributes the $1$-hour timeframe capacity among the timesteps, while ensuring dynamic capacity changes.
    Let $\textit{cap}_b(i,t) = \lfloor \frac{\textit{cap}'(i,t)}{T_{\textit{gran}}} \rfloor$ be the base capacity and
    let $\textit{cap}_r(i,t) = \textit{cap}'(i,t) \% T_{\textit{gran}}$ be the remaining capacity,
    which is added to the base capacity in steps $\textit{step}(i,t) = \lfloor \frac{T_{\textit{gran}}}{\textit{cap}_r(i,t)} \rfloor$ - 
    or put otherwise, exactly for the timestep
    $\textit{REM}(i,t) = \{t' \mid 0 \leq i < \textit{cap}_r(i,t), t' = \textit{step}(i,t) \cdot i + T_{\textit{gran}} \cdot \lfloor \frac{t}{T_{\textit{gran}}} \rfloor \}$.
    
    Now we are able to define the capacity per timestep $\textit{cap}(i,t)$ as 
    $\textit{cap}(i,t) = \textit{cap}_b(i,t) + 1|_{t \in \textit{REM}(i,t)}$,
    where $1|_{t \in \textit{REM}(i,t)} = \begin{cases}
        1 \text{ if } t \in \textit{REM}(i,t)\\
        0 \text{ if } t \not \in \textit{REM}(i,t)
    \end{cases}$

    This behavior extends to cases where $|\mathbf{T}| \geq |T|$. 
    To accommodate capacity constraints, we map flight times to fly-over times of sectors.
    The function $\textit{over}: \mathbf{F} \times S \times \mathbf{T} \rightarrow \{0,1\}$
    indicates whether flight $f$ is in a sector $s$ at time $t$.
    We define $\textit{over}$ as follows:
    let $f$ be a flight s.t. $(v_i,t_i),(v_{i+1},t_{i+1})$ is part of the trajectory s.t. $\Delta t = t_{i+1} - t_i$.
    Then, flight $f$ is in time $t \in [t_i,t_i + \lfloor \frac{\Delta t}{2} \rfloor ]$ over sector $\mathbf{SEC}^{-1}(v_i,t) = s_i$, so $\textit{over}(f,s_i,t) = 1$,
    and $f$ is in time $t' \in [t_i + \lfloor \frac{\Delta t}{2} \rfloor + 1, t_{i+1}]$ over sector $\mathbf{SEC}^{-1}(v_{i+1},t') = s_{i+1}$, so $\textit{over}(f,s_{i+1},t') = 1$.

    Finally, for $t \in \mathbf{T}$ and $s \in S$ we interpret $\textit{sec}(s,t)$ on $t > |T|$ as $\textit{sec}(s,t) = \textit{sec}(s,\max\{T\})$;
    and equally proceed for $v \in V$ and $t \in \mathbf{T}$ with $\textit{sec}^{-1}(v,t) = \textit{sec}^{-1}(v,\max\{T\})$.

\end{definition}

\begin{definition}[Aircraft Model, extends Definition~\ref{def:airplane-model}]
    We distinguish between aircraft $a \in A$ and flights $f \in F$.
    Let  $a = (id,v_a,F_a)$ be an airplane,
    $id \in \mathbb{N}$ is the unique airplane's ID,
    $v_a$ is an airplane's constant (assumption) velocity,
    and $F_a \subseteq F$ is the corresponding set of flights.
    A route is a simple path in the navpoint graph, where we denote a path as $p$ and the set of paths as $P$.
    The set of flights $F$ adds an id $id_f$ and timesteps to a route, i.e.,
    $F = \{(id_f,((v_0, t_0),(v_1,t_1), \ldots, (v_n,t_n))) \mid id_f \in \mathbb{N}, p = (v_0, v_1, \ldots, v_n) \in P, \tau = (t_0, t_1, \ldots, t_n) \in T^n, \forall t_i, t_j \in \tau: i < j \rightarrow t_i < t_j\}$.
    For every aircraft $a \in A$ the id $id_f$ of a flight $(id_f,\ldots) \in F_a$ is unique. 
    
    We map travel times of an aircraft to sector fly-over times in the following way:
    Let $a \in A$ be an arbitrary aircraft with velocity $v_a$ - converted to standard time and meters\footnote{For a given aircraft speed in knots $v_{a}^{kt}$, we convert it to standard time $[\frac{m}{\textit{time-granularity}}]$: $v_a = \frac{3600}{T_{\textit{grad}}} \cdot v_f^{kt} \cdot 0.514444 $},
    which flies the flight $f \in F_a$ and starts at an arbitrary time $t$.
    Let $(v_i,t_i),(v_{i+},t_{i+1}) \in f$ s.t. $t_i < t_{i+1}$ and let $d = d(v_i,v_{i+1})$.
    Then the duration in standard time it takes to fly from $v_i$ to $v_{i+1}$ is $\Delta t = \max\{\lceil \frac{d}{v_a} \rceil, 1\}$.
    Let $f^{sec}_{(v_i,t_i),(v_{i+1},t_{i+1})} = \{(s,t+t_i) \mid 0 \leq t \leq \Delta t, t \leq \lfloor \frac{\Delta t}{2} \rfloor, s = \textit{sec}^{-1}(v_i,t+t_i)\}$
    $\cup \{(s',t+t_i) \mid 0 \leq t \leq \Delta t, t > \lfloor \frac{\Delta t}{2} \rfloor, s' = \textit{sec}^{-1}(v_{i+1},t+t_i)\}$
    and for the whole flight $f$ let
    $f^{sec} = \bigcup_{(v_i,t_i),(v_{i+1},t_{i+1}) \in f} f^{sec}_{(v_i,t_i),(v_{i+1},t_{i+1})}$.
    Let $\textit{over}:F \times S \times T \rightarrow \{0,1\}$
    be the binary function defining when a flight $f \in F$ is in a sector $i \in S$.
    I.e., for a flight $f \in F$, time $t \in T$, and sector $i \in S$: $\textit{over}(f,i,t) = 1$ if $(i,t) \in f^{sec}$,
    else $\textit{over}(f,i,t) = 0$.
    These definitions extend to the solution definition $\mathbf{S}$.
\end{definition}

\section{Additional Details of ASP Model}\label{sec:appdxasp}

\subsection{ASP Encoding: Auxiliaries}

We show the omitted auxiliary rules from the main part.
We compute all aircraft speeds (Line~(1)), ensure that the graph is undirected (Line~(2)), and that airports are connected to themselves (Line~(3)).
Further, we create the set of navpoints and flights (Lines~(4)--(6)) and finally,
compute all possible timesteps in Lines~(7)--(10).

\begin{lstlisting}
aircraft_speed(S) :- acft(_,S).
nvpt_e(Y,X,S,D) :- nvpt_e(X,Y,S,D).
nvpt_e(X,X,S,1) :- arpt(X), aircraft_speed(S).
nvpt(X) :- nvpt_e(X,_,_,_).
nvpt(Y) :- nvpt_e(_,Y,_,_).
flt_id(ID) :- nvpt_flt_pln(ID,_,_).
t(T) :- nvpt_flt_pln(_,_,T).             t(0).
t(T+1) :- t(T), max_t(TMAX), T < TMAX.
t(T+1) :- t(T), induced_max_t(TMAX), T < TMAX.
induced_max_t(T) :- T = #max{T': nvpt_flt_pln(_,_,T')}.
\end{lstlisting}

We present the full ASP encoding (without rewriting due to space restrictions):
\begin{lstlisting}
% FACTS:
% navpoint_flight_plan(<ID>,<NAVPOINT-LOCATION>,<TIME>)
% navpoint_sector(<NAVPOINT-ID>,<SECTOR-ID>,<TIME=0 for FACTS>)
% navpoint_edge(<X0>,<X1>,<PLANE-SPEED>,<UNIT-DIST>)
% airport(<X>)
% aircraft(<ID>,<PLANE-SPEED>)
% aircraft_flight(<PLANE-ID>,<FLIGHT-ID>)
% atomic_sector(<NAVPOINT-ID>,<CAPACITY>,<TIME>)

navpoint(NAV) :- navpoint_sector(NAV,_,_).

% Dynamic Sector Allocation:
navpoint_sector(SEC,SEC,T) :- navpoint(SEC), time(T), airport(SEC).
1{navpoint_sector(NAV,SEC,T):navpoint_sector(_,SEC,0)}1:- navpoint(NAV), time(T), not airport(NAV), regulation_dynamic_sector_allocation.
initial_sector(SEC) :- navpoint_sector(_,SEC,0), regulation_restricted_dynamic_sector_allocation.
1{navpoint_sector_decision(SEC,0..3,T)}1 :- initial_sector(SEC), time(T), not airport(SEC), regulation_restricted_dynamic_sector_allocation.
navpoint_sector(NAV1,SEC1,T) :- navpoint_sector_decision(SEC,DEC,T), navpoint_sector_restricted_sector_allocation(SEC,DEC,NAV1,SEC1), regulation_restricted_dynamic_sector_allocation.
navpoint_sector(NAV,SEC,T) :- time(T), navpoint_sector(NAV,SEC,0), -regulation_dynamic_sector_allocation.

:- navpoint_sector(NAV,SEC1,T), navpoint_sector(NAV,SEC2,T), SEC1 != SEC2.
:- navpoint_sector(NAV,SEC,T), not airport(NAV), airport(SEC).
:- navpoint_sector(NAV,SEC,T), airport(NAV), not airport(SEC).
:- navpoint_sector(NAV1,SEC,T), NAV1 !=SEC, not navpoint_sector(SEC,SEC,T).

% MINIMIZE SECTOR-NUMBER:
sector_len(SEC,T,CAP_LEN) :- time(T), navpoint(SEC), CAP_LEN = #count{NAV:navpoint_sector(NAV,SEC,T)}.
sector(SEC,T,0) :- sector_len(SEC,T,CAP_LEN), CAP_LEN = 0. 
sector(SEC,T,C) :- time(T), navpoint(SEC), C = #max{C',NAV:navpoint_sector(NAV,SEC,T),atomic_sector(NAV,C',T)}, sector_len(SEC,T,CAP_LEN), CAP_LEN > 0.
sector_number(T,NUM) :- time(T), NUM = #count{SEC:sector_len(SEC,T,C), C > 0}.

% MINIMIZE NAVPOINT-SECTOR ASSIGNMENT CHANGES:
navpoint_changed(NAV,T,1) :- navpoint_sector(NAV,SEC0,T), navpoint_sector(NAV,SEC1,T+1), SEC0 != SEC1.
navpoint_changed(NAV,T,0) :- navpoint_sector(NAV,SEC0,T), navpoint_sector(NAV,SEC0,T+1).
sector_diff(T,DIFF) :- time(T), DIFF = #sum{V,NAV: navpoint_changed(NAV,T,V)}.

% MINIMIZE RECONFIG:
reconfig(NAV,T) :- navpoint_sector(NAV,SEC,0), navpoint_sector(NAV,SEC2,T), SEC != SEC2, T > 0.

% GET ALL AIRPLANE SPEEDS
aircraft_speed(S) :- aircraft(_,S).

% UNDIRECTED GRAPH:
navpoint_edge(Y,X,S,D) :- navpoint_edge(X,Y,S,D).
navpoint_edge(X,X,S,1) :- airport(X), aircraft_speed(S).

% GET NAVPOINTS
navpoint(X) :- navpoint_edge(X,_,_,_).
navpoint(Y) :- navpoint_edge(_,Y,_,_).

% PRELIMINARIES:
% GET ALL TIME STEPS:
time(T) :- navpoint_flight_plan(_,_,T).
% GET ALL FLIHT IDS:
flightID(ID) :- navpoint_flight_plan(ID,_,_).

% MORE TIME HEURISTIC:
time(0).
time(T+1) :- time(T), max_time(TMAX), T < TMAX.
time(T+1) :- time(T), induced_max_time(TMAX), T < TMAX.
induced_max_time(T) :- T = #max{T': navpoint_flight_plan(_,_,T')}.

% NORMAL FLIGHT:
navpoint_flight(ID,X,T) :- navpoint_flight_plan(ID,X,T), not reroute(ID).

% LIMIT SECTOR CAPACITY (hard constraint modelling):
%overload(X,T,LOAD-C) :- sector(X,T,C), #count{ID:flight(ID,X,T),flight(ID,X2,T-1), X!=X2} = LOAD, LOAD > C.
% CURRENT METHOD -> LIMIT NUMBER FLIGHTS PER TIMESTEP IN SECTOR:
overload(X,T,LOAD-C) :- sector(X,T,C), #count{ID:flight(ID,X,T)} = LOAD, LOAD > C.

% REROUTE - YES/NO:
{reroute(ID)} :- navpoint_flight_plan(ID,_,_).
reroute(ID) :- planned_departure_time(ID,T), actual_departure_time(ID,T'), T' > T.
{navpoint_flight(ID,X,T):navpoint(X)}1 :- reroute(ID), time(T), actual_departure_time(ID,T'), T >= T', regulation_rerouting.
navpoint_flight(ID,X,T + TT) :- navpoint_flight_plan(ID,X,T), actual_departure_time(ID,TACT),planned_departure_time(ID,TPLAN), TT = TACT - TPLAN, -regulation_rerouting.
1{navpoint_flight_chosen_path(ID,P):regulation_restricted_rerouting_paths(ID,_,_,P)}1 :- reroute(ID), regulation_restricted_rerouting.
navpoint_flight(ID,X,T+T_DELTA) :- regulation_restricted_rerouting_paths(ID,X,T_DELTA,P), reroute(ID), actual_departure_time(ID,T), navpoint_flight_chosen_path(ID,P), regulation_restricted_rerouting.

% ----------------- ONLY FOR REROUTED PLANES --------------------
% NO FLIGHT BEFORE PLANNED START TIME:
planned_departure_time(ID,T) :- flightID(ID), T = #min{T':navpoint_flight_plan(ID,_,T')}.
:- reroute(ID), navpoint_flight(ID,_,T), planned_departure_time(ID,T'), T < T'.

% FLIGHT STARTS AT ORIGIN:
planned_origin(ID,X) :- reroute(ID),planned_departure_time(ID,T),navpoint_flight_plan(ID,X,T).
navpoint_flight(ID,X,T) :- reroute(ID), actual_departure_time(ID,T), planned_origin(ID,X).

% FLIGHT ENDS AT DESTINATION:
planned_destination(ID,X) :- reroute(ID), T=#max{T':navpoint_flight_plan(ID,_,T')}, navpoint_flight_plan(ID,X,T).
actual_arrival_time(ID,T) :- flightID(ID), T = #max{T':navpoint_flight(ID,_,T')}.
:- flightID(ID), not actual_arrival_time(ID,_).
:- flightID(ID), planned_destination(ID,X), actual_arrival_time(ID,T), navpoint_flight(ID,X',T), X != X'.

% NO FLIGHT AFTER DESTINATION REACHED:
actual_destination_first_reached_time(ID,T) :- reroute(ID), planned_destination(ID,X), T = #min{T':navpoint_flight(ID,X,T')}.
:- reroute(ID), navpoint_flight(ID,_,T), actual_destination_first_reached_time(ID,T'), T > T'.

% MUST BE PATH:
%navpoint_seq(ID,T,X,X'') :- navpoint_flight(ID,X,T), T'' = #min{T',X': navpoint_flight(ID,X',T'), T' > T}, navpoint_flight(ID,X'',T'').
%:- navpoint_flight(ID,X,T), navpoint_flight(ID,X',T'), T' > T, navpoint_seq(ID,T,X,X'), aircraft_flight(AID,ID), aircraft(AID,S), D = T'-T, not navpoint_edge(X,X',S,D).

% MUST BE PATH:
pot_navpoint_seq(ID,T,TT) :- navpoint_flight(ID,_,T), navpoint_flight(ID,_,TT), T < TT.
not_navpoint_seq(ID,T,TT) :- pot_navpoint_seq(ID,T,TT), T < TTT, TTT < TT, pot_navpoint_seq(ID,T,TTT).
navpoint_seq(ID,T,TT) :- pot_navpoint_seq(ID,T,TT), not not_navpoint_seq(ID,T,TT).
:- navpoint_seq(ID,T,TT), navpoint_flight(ID,X,T), navpoint_flight(ID,XX,TT), aircraft_flight(AID,ID), aircraft(AID,S), D = TT - T, not navpoint_edge(X,XX,S,D).

% NEVER TWICE AT SAME LOCATION:
:- navpoint_flight(ID,X,T), navpoint_flight(ID,X,T'), T != T', not airport(X).

% FLIGHT MUST OCCUR:
flightOccurs(ID) :- navpoint_flight(ID,_,_).
:- flightID(ID), not flightOccurs(ID).

% HELPER PREDS:
1{actual_departure_time(ID,T):time(T), T >= TP}1 :- planned_departure_time(ID,TP), regulation_delaying.
1{actual_departure_time(ID,T):time(T), T >= TP, T <= TP + MD}1 :- planned_departure_time(ID,TP), restricted_delaying_max_delay(MD), regulation_restricted_delaying.
actual_departure_time(ID,TP) :- planned_departure_time(ID,TP), -regulation_delaying.

planned_arrival_time(ID,T) :- flightID(ID), T = #max{T':navpoint_flight_plan(ID,_,T')}.
arrival_delay(ID,Y-T) :- planned_arrival_time(ID,T), actual_arrival_time(ID,Y).

% NO GAPS:
:- time(T), actual_departure_time(ID,T'), T > T', actual_arrival_time(ID,T''), T < T'', not flight(ID,_,T).

% TRANSLATE navpoint_flight to flight
flight(ID,S,T) :- navpoint_flight(ID,X,T), navpoint_sector(X,S,T).
flightT(ID,T,T',0) :- navpoint_seq(ID,T,TT), D = TT - T, time(T'), T' > T, T' <= T + D/2.
flightT(ID,TT,T',1) :- navpoint_seq(ID,T,TT), D = TT - T, time(T'), T' < TT, T' > T + D/2.
flight(ID,S,T') :- flightT(ID,T,T',0), navpoint_flight(ID,X,T), navpoint_sector(X,S,T').
flight(ID,S,T') :- flightT(ID,TT,T',1), navpoint_flight(ID,XX,TT), navpoint_sector(XX,S,T').

%flight(ID,S,T') :- navpoint_flight(ID,X,T), navpoint_flight(ID,_,TT), navpoint_seq(ID,T,TT), D = TT - T, time(T'), T' > T, T' <= T + D/2, navpoint_sector(X,S,T').
%flight(ID,S',T') :- navpoint_flight(ID,_,T), navpoint_flight(ID,X',TT), navpoint_seq(ID,T,TT), D = TT - T,time(T'), T' < TT, T' > T + D/2, navpoint_sector(X',S',T').

% Multi-leg flights:
:- aircraft_flight(AID,FID0), aircraft_flight(AID,FID1), FID0 != FID1, planned_arrival_time(FID0,T0), planned_arrival_time(FID1,T1), T0 < T1, actual_arrival_time(FID0, T0'), actual_departure_time(FID1,T1'), T0' >= T1'.

% ------------------------------------------------------------------------------
% PRIMARY WEAK CONSTRAINT:
:~ overload(X,T,OVER). [OVER@10,X,T]
% SECONDARY WEAK CONSTRAINT:
:~ arrival_delay(ID,DIFF). [DIFF@9,ID]
% TERTIARY WEAK CONSTRAINT:
:~ sector_number(T,NUM). [NUM@8,T]
% OPTIMIZE FOR LEAST SECTOR CHANGES:
:~ sector_diff(T,DIFF). [DIFF@7,T]
% OPTIMIZE FOR LEAST REROUTES/DELAYS:
:~ reroute(ID). [1@6,ID]
% OPTIMIZE FOR LEAST RECONFIGS:
:~ reconfig(ID,T). [1@5,ID,T]
\end{lstlisting}

\section{Dynamic MIP Model}

We simulate the SOTA MIP model with the following dynamic MIP model.
Dynamic in the sense that the ground delay and reroute alternatives are generated by the Python wrapper.
Otherwise, we restrict the MIP model to a SOTA model in the sense that it only optimizes routes and delay, but not DAC.

\begin{figure*}
\begin{flalign}
    \notag &\textbf{Simulated SOTA MIP Model}&& \\
    &\label{eq:mip-real-01} \min \sum_{f \in F} \sum_{p \in P_f} d_p \cdot \omega_{\textit{dest}_f, t_{\textit{dest}}}^{f,p} \\
    &\nonumber \textbf{Subject to:}\\
    &\label{eq:mip-real-03} \omega_{j,t}^{f,p} \leq \omega_{i,t+1}^{f,p} && \textit{For all } f \in F, p \in P_f, t \in T^{f,p}, t+1 \in T^{f,p}, j = S_t^{f,p}, i = S_{t+1}^{f,p}\\
    &\label{eq:mip-real-04} \sum_{p \in P_t^f, j = S_t^{f,p}} \omega_{j,t}^{f,p} \leq 1 && \textit{For all } f \in F, t \in T_f\\
    &\label{eq:mip-real-05} \sum_{p \in P_{j,t}^f} \omega_{j,t}^{f,p} \leq 1 && \textit{For all } f \in F, t \in T_f, j = S_t^f\\
    %
    %
    &\label{eq:mip-real-07} 1 \leq \sum_{p \in P_f} \omega_{s,t_{s,\textit{min}}}^{f,p} && \textit{For all } f \in F, s = \textit{source}_f\\
    &\label{eq:mip-real-07-a} \sum_{p \in P_f} \omega_{s,t_{s,\textit{min}}}^{f,p} \leq 1 && \textit{For all } f \in F, s = \textit{source}_f\\
    &\label{eq:mip-real-08} \omega_{j,t}^{f,p} \in \{0,1\} && \textit{For all } f \in F, p \in P_f, t \in T^{f,p}, j = S^{f,p}_t\\
    \nonumber &\textbf{Consecutive flights}\\
    &\label{eq:mip-real-09} (1 - \omega_{i,t_{dest}}^{f_0,p'}) \geq \omega_{j,t_{start}}^{f_1,p''} && \textit{For all } f_0, f_1 \textit{ s.t. } f_0,f_1 \in A \textit{ and}\\
        \nonumber &&& \textit{planned landing time of } f_0 \textit{ is before planned departure time of } f_1\\
        \nonumber &&& \textit{but path one is } t_{dest} \geq t_{start}\\
    \nonumber &\textbf{Capacity Slots}\\
    &\label{eq:mip-real-02} \sum_{f \in F} \sum_{p \in P_f } \omega_{j,t}^{f,p} < C_j(t) &&\textit{For all } j \in S, t \in T
    \end{flalign}
\vspace{-0.7cm}
\end{figure*}

Let $F$ be the set of flights, where $f \in F$ has a departure ($\textit{source}_f$) and an arrival ($\textit{destination}_f$) airport,
with associated departure time $t_{\textit{start}}$ and expected arrival time $\hat{t}_{\textit{dest}}$.
Let $\mathcal{G}$ be the navpoint graph.
The set of paths $P$ is the set of simple paths between any two vertices in the navpoint graWe create for each jobph $\mathcal{G}$.
Then let $P_f$ be the set of paths restricted to f: $P_f = \{p \mid p = \langle \textit{source}_f, \ldots, \textit{destination}_f \rangle \in P \}$.

Further, let $T$ be the set of times in the whole application,
whereas, $T^{f,p}$ is the set of times that an airplane needs to travel from $\textit{source}_f$ to $\textit{destination}_f$.
Given a flight $f \in F$, a path $p \in P_f$, and a specific time during the flight $t \in T^{f,p}$,
we can compute the estimated current sector-position of the flight $j = S^{f,p}_t$.
Lastly, let $S$ be the set of sectors.
We interpret the sectors as the \emph{Voronoi} diagram of the navpoints, where each sector may have a set of corresponding navpoints.
The location of the flight along a path is computed according to the Voronoi interpretation.
A sector $j \in S$ has a capacity per time unit $C_j(t)$.
Example:  
When a flight $f$ starts at $t=0$ at navpoint $a$ (sector $A$) and travels to navpoint $b$ (sector $B$) - and the duration of the flight takes $\Delta t = 3$ units of time,
then flight $f$ is in sector A during time units $\{0,1\}$,
and in sector B during time units $\{2,3\}$.

Further, let $\textit{isFirst}(j,t,f,p)$ evaluate to true iff $\forall t' \in T^{f,p}: t'<t$ it holds that $i=S^{f,p}_{t'}$, where $i \not = j$.
Conversely, $\textit{isLast}(j,t,f,p)$ evaluates to true iff $\forall t' \in T^{f,p}: t' > t$ it holds that $i = S^{f,p}_{t'}$, where $i \not j$.

We treat capacity violations as a hard constraint and optimize for efficiency (to be more in line with SOTA work).
In the future, we plan to adapt this SOAT simulation model to a new model in our formalism, by incorporating DAC optimization and lexicographic optimization.

\emph{Description of Equations}.
We define variables for flightss, delays, and paths in Equation~(\ref{eq:mip-real-08}).
We require that exactly one departing flight is chosen (Equations~(\ref{eq:mip-real-07}) and~(\ref{eq:mip-real-07-a})).
If a path is chosen it must be kept throughout the flight (Equation~(\ref{eq:mip-real-03})) and there must not be more than one path at any other node (Equation~(\ref{eq:mip-real-04})).
Similarly, at any node along the path there must only be one path active (Equation~(\ref{eq:mip-real-05})).
At any given timestep, a sector must not be overloaden (Equation~(\ref{eq:mip-real-02})) and for an airplane with multiple flights per day, it must not be that the departure time $t_{\textit{start}}$ of the next flight is before the landing time of the previous flight $t_{\textit{dest}}$ (Equation~(\ref{eq:mip-real-09})).
One may switch Equation~(\ref{eq:mip-real-02}) to Equation~(\ref{eq:mip-real-02-diff}) to switch demand measurement from number of airplanes in a sector, to airplanes entering a sector.

\begin{flalign}
    &\nonumber \textbf{Diff Capacity}\\
    &\label{eq:mip-real-02-diff} \sum_{f \in F} \sum_{p \in P_f \land \textit{isFirst}(j,t,f,p)} \omega_{j,t}^{f,p} < C_j(t) &&\textit{For all } j \in S, t \in T
\end{flalign}

\clearpage
\onecolumn

\input{appendix_tables}

\onecolumn
\clearpage

\section{Additional Details on Experiments}\label{sec:appdxexp}


We show additional details of the experiments.
In Tables~\ref{tab:tournament-wins} and~\ref{tab:tournament-wins-asp} we show the tournament wins and number of solved instances, overall and only for the ASP variants, respectively.
In Tables~\ref{tab:summary-by-scenario}--\ref{tab:summary-by-size} we show the results when structured along the geographic and flight scenarios.
In Figures~\ref{fig:heat-maps-east-asia}--\ref{fig:heatmaps-MAJOR-EUROPE} we show averaged metrics for all variants, geographic scenarios, and flight sizes, for the ASP variants as heat maps.
Finally, in Table~\ref{tab:june-2019-regions} we show the detailed results for all variants, for all instances.

\section{Omitted Proofs}

\LemSolutionSpaceEquivalence*
\begin{proof}
$\mathcal{S}_{AS} \subseteq \mathcal{S}_{\leq t}$:
let $s = (\mathbf{T},\mathbf{SEC}, \mathbf{A}) \in \mathcal{S}_{AS}$ be an arbitrary converted answerset.
It is also a solution for $s \in \mathcal{S}_{\leq t}$:
due to $s = c^{-1}(AS)$, $AS$ must be an answerset of $\Pi' \cup c(\mathcal{I)} \cup \{max\_t(t)\}$.
$\mathbf{SEC}$ is a partition of navpoints over time and $\mathbf{A}$ is are valid flights due to Section~\ref{subsec:asp-encoding-constraints}.
Further, $\max\{\mathbf{T}\} = t'$, where $t' \leq t$.
Lastly, Section~\ref{subsec:navpoint-flight-mapping} ensures that the trajectories are correctly mapped to sector occupancies and with $c_1$ no overloads occur in the answerset $AS$.

$\mathcal{S}_{\leq t} \subseteq \mathcal{S}_{AS}$:
let $s = (\mathbf{T},\mathbf{SEC}, \mathbf{A}) \in \mathcal{S}_{\leq t}$ be an arbitrary valid solution.
Then also $s \in \mathcal{S}_{AS}$:
first, due to $max\_t(t)$ it must hold for any answerset of $\Pi' \cup c(\mathcal{I)} \cup \{max\_t(t)\}$ that $\max_{f \in \mathbf{F}} t^a_f \leq t$.
Further, due to having $\textit{ASP} (r,d,s)$ and Section~\ref{subsec:asp-encoding-actions},
the considerd flights in the solver are all departure times and trajectories (under the time limit $t$), combined with all possible sector assignments;
both are restricted to the feasibility constraints of Definition~\ref{def:atfcm-solution}.
Therefore also $s \in \mathcal{S}_{AS}$.
\end{proof}

\ThmSolutionExistence*
\begin{proof}{Proof (Idea):}
By Lemma~\ref{thm:solution-space-equivalence} we know that the answersets of the encoding corresponds to all valid solutions within the same timeframe;
we are left to show that this translates to the optimization problem.
In contrast to $\Pi'$, $\Pi$ does not have a hard constraint on the overload but additional soft constraints.
However, requiring a 0-overload answerset $\Pi'$ corresponds to $c_1$ being active.
The remaining soft constraints of $\Pi$ do not shrink the solution space, they only define an order upon them, which leads the statement to hold.
\end{proof}

\ThmOptimality*
\begin{proof}
Towards a contradiction assume the statement does not hold, so for the found optimal answerset
$\{AS\} = \mathcal{AS}(\Pi \cup c(\mathcal{I}) \cup \{max\_t(t)\})$
it holds $\mathbf{S}_{opt} \prec c^{-1}(AS)$.
First, observe that by Lemma~\ref{thm:solution-space-equivalence} it holds $\mathbf{S}_{opt} \in \mathcal{S}_{AS}$ for $\Pi'$.
Therefore, the soft constraints in $\Pi$ have to rule out $\mathbf{S}_{opt}$.
However, they define a lexicographic optimization (Section~\ref{subsec:optimization}), equivalent to the optimization Definition~\ref{def:atfcm-optimization-problem}.
The highest priority has reducing overloads with a priority of $10$, followed by 
\ATFCMoDELAY with a priority of $9$, 
\ATFCMoNUMSEC with a priority of $8$,
\ATFCMoSECCHANGES with a priority of $7$,
\ATFCMoREROUTED with a priority of $6$, and
\ATFCMoRECONFIG with a priority of $5$.
But therefore the answerset corresponding to $\mathbf{S}_{opt}$ is lexicographically better than $AS$ and is therefore preferred, which contradicts our assumption.
%
%
%
\end{proof}

\begin{figure}[t]
    \centering
    \includegraphics[width=8cm]{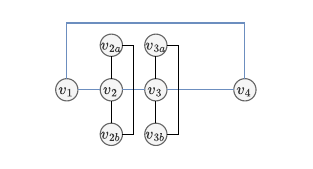}
    \caption{Navpoint counter example graph $G = (V,E)$ for Observation~\ref{obs:counter-example-optimality}.}
    \label{fig:optimality-counter-example}
\end{figure}

Note that the pipeline implementation does not guarantee to find an optimal solution, it still guarantees to find a solution.
To ensure termination, we break the loop if $t > |T| + \frac{1}{2} \cdot |T| \cdot |F| \cdot (|F|+1)$, as further search is not necessary to find a (optimal) solution.
See also Theorem~\ref{thm:upper-bounding-optimal-time}, Observations~\ref{obs:counter-example-optimality} and~\ref{obs:sequential-max-delay}, and Corollaries~\ref{col:upper-optimal-bound} and~\ref{col:non-optimality-pipeline}.
These results hold for $d_f = \max\{t^a_f - t^e_f,0\}$, where $t^a_f$ is the actual and $t^e_f$ expected arrival time.
This holds for all generated instance implicitly, one can explicitly enforce it via:
\begin{lstlisting}
arrival_delay(ID,Y-T) :- p_arr_t(ID,T), a_arr_t(ID,Y), Y > T.
\end{lstlisting}

\begin{observation}
\label{obs:counter-example-optimality}
    Let $\mathcal{I}$ be an instance and $\mathbf{S}_1$ and $\mathbf{S}_2$ two valid solutions.
    A smaller time horizon does not imply a lower delay:
    $\max\{\mathbf{T}_1\} < \max\{\mathbf{T}_2\} \not \rightarrow \ATFCMoDELAY^1 < \ATFCMoDELAY^2$.
\end{observation}
\begin{proof}{Proof by counterexample.}
We assume it holds $\max\{\mathbf{T}_1\} < \max\{\mathbf{T}_2\} \rightarrow \ATFCMoDELAY^1 < \ATFCMoDELAY^2$.
Let $\mathcal{I}$ be with $G = (V,E)$ as in Figure~\ref{fig:optimality-counter-example},
where $d(v_1,v_2) = d(v_2,v_3) = d(v_{2a}, v_2) = d(v_2,v_{2b}) = d(v_{3a},v_3) = d(v_3,v_{3b}) = 1$,
$d(v_3,v_4) = 2$,
$d(v_{2a},v_{2b}) = d(v_{3a},v_{3b}) = 3$, and
$d(v_1, v_4) = 5$.
Let the sector structure $\textit{sec}$ be atomic with $\forall v \in V: \textit{cap}(v,0) = 1$.
The filed flights are $f_0 = (0,((v_1,1),(v_2,2),(v_3,3),(v_4,5)))$,
$f_1 = (1,((v_{2a},1),(v_{2},2),(v_{2b},3)))$, and
$f_2 = (2,((v_{3a},2), (v_3,3), (v_{3b},4)))$;
with aircraft $a_0 = (a_0,1,\{f_0\})$,
$a_1 = (a_1,1,\{f_1\})$, and
$a_2 = (a_2,1,\{f_2\})$.
Further, let $T_{\textit{gran}} = 1$ and $T = \{0,1,2,3,4,5\}$.

Consider now the following valid solutions:
$\mathbf{S}_1 = (\mathbf{T}_1, \mathbf{SEC}_1, \mathbf{A}_1)$ and 
$\mathbf{S}_2 = (\mathbf{T}_2, \mathbf{SEC}_2, \mathbf{A}_2)$.
Let $\mathbf{T}_1 = \{0,1,2,3,4,5\}$, with $\mathbf{SEC}_1 = \textit{sec}$,
and $\mathbf{f}_0 = (0,((v_1,1),$ $(v_2,2),(v_3,3),(v_4,5)))$,
$\mathbf{f}_1 = (1,((v_{2a},1),(v_{2b},4)))$, and
$\mathbf{f}_2 = (2,((v_{3a},2), (v_{3b},5)))$.
Further, let $\mathbf{T}_2 = \{0,1,2,3,4,5,6\}$, with $\mathbf{SEC}_2 = \textit{sec}$,
and $\mathbf{f}_0 = (0,((v_1,1),(v_4,6)))$,
$\mathbf{f}_1 = (1,((v_{2a},1),(v_{2},2),(v_{2b},3)))$, and
$\mathbf{f}_2 = (2,((v_{3a},2), (v_3,3), (v_{3b},4)))$.
It holds $\max\{\mathbf{T}_1\} < \max\{\mathbf{T}_2\}$,
however $\ATFCMoDELAY^1 = 2$, but $\ATFCMoDELAY^2 = 1$.
Which concludes the counterexample.
\end{proof}

\begin{corollary}
    \label{col:non-optimality-pipeline}
    The prototype (pipeline) implementation is not optimal.
\end{corollary}
\begin{proof}
    Directly follows from Observation~\ref{obs:counter-example-optimality}.
\end{proof}

\begin{observation}
\label{obs:sequential-max-delay}
    The delay of sequential flight scheduling (max one flight airborne) is upper bounded by $\ATFCMoDELAY^{seq} \leq \frac{1}{2} \cdot  |F| \cdot (|F| + 1) \cdot |T|$.
\end{observation}

\begin{theorem}
\label{thm:upper-bounding-optimal-time}
%
Let $\mathcal{I}$ be an instance, where $\mathbf{S}_{feas}$ is an arbitrary valid solution thereof with delay $\ATFCMoDELAY^{feas}$.
There exists an optimal solution $\mathbf{S}_{opt}$ such that:
$$\max\{\mathbf{T}_{opt}\} \leq \max_{f \in \mathbf{F}} \{t^e_f\} + \ATFCMoDELAY^{feas}$$
\end{theorem}
\begin{proof}
Since $\mathbf{S}_{opt}$ optimizes arrival delay as its highest priority, its total delay must satisfy $\ATFCMoDELAY^{opt} \leq \ATFCMoDELAY^{feas}$.
Since individual delays are non-negative ($d_f \geq 0$),
no single flight can absorb a delay greater than the total delay: $\forall f \in \mathbf{F}: d_f \leq \ATFCMoDELAY^{opt} \leq \ATFCMoDELAY^{feas}$.
The arrival time of any flight $f$ in the optimal solution is $t^a_f = t^e_f + d_f$.
Note $|T|-1=\max\{T\}$.
Therefore, $\max\{\mathbf{T}_{opt}\} = \max\{\max_{f} \{t^e_f + d_f\} \cup \{|T|-1\}\} \leq \max_{f \in \mathbf{F}}\{t^e_f\} + \ATFCMoDELAY^{feas}$.
\end{proof}

\begin{corollary}
\label{col:upper-optimal-bound}
    Let $\mathcal{I}$ be an instance, where $\mathbf{S}_{seq}$ is a valid solution where flights are sequentially scheduled.
    There exists an optimal solution $\mathbf{S}_{opt}$ such that:
    $$\max\{\mathbf{T}_{opt}\} \leq |T| +  \frac{1}{2} \cdot  |F| \cdot (|F| + 1) \cdot |T|$$
\end{corollary}
\begin{proof}
By Theorem~\ref{thm:upper-bounding-optimal-time} we know
$\max\{\mathbf{T}_{opt}\} \leq \max_{f \in \mathbf{F}} \{t^e_f\} + \ATFCMoDELAY^{seq}$
and by Observation~\ref{obs:sequential-max-delay}
$\ATFCMoDELAY^{seq} = \frac{1}{2} \cdot  |F| \cdot (|F| + 1) \cdot |T|$.
Therefore 
        $\max\{\mathbf{T}_{opt}\} \leq \max_{f \in \mathbf{F}} \{t^e_f\} +  \frac{1}{2} \cdot  |F| \cdot (|F| + 1) \cdot |T|$.
Further, we know $\max_{f \in \mathbf{F}} \{t^e_f\} \leq |T|$, and therefore 
$\max\{\mathbf{T}_{opt}\} \leq |T| +  \frac{1}{2} \cdot  |F| \cdot (|F| + 1) \cdot |T|$.
\end{proof}

\section{Limitations}
``All models are wrong but some are useful'', is a phrase which captures the core of our limitations well.
Whether our model will be useful will be judged by time, where we think that it can be a cornerstone of pushing research in ATFCM to more comparable and grounded work.
Therefore, our mathematical formal model applies abstractions to the safety-critical socio-technical ATFCM task.
For example, we need to discretize time for our time model and space for our navpoint graph, which may induce compounding errors.
Further, our current version does not model several aspects, which however, are required for operational use --- where we argue that they can be integrated in future work, such as:
(a) variable speeds of aircraft, (b) uncertainty and stochasticity, for example due to weather, or (c) modelling passengers/airlines and fairness, \dots
\section{Statement on the use of AI}
As we are non-native English speakers, we improved the readability of the paper by using generative AI for spelling, grammar, and wording suggestions.
AI was used as an assistant during the development phase of the prototype including setting up and running the experiments on the cluster.
Further, AI was partially used to generate Python scripts for plotting the figures (except schematics, which we made ourselves entirely) and generating the \LaTeX tables.
We used variants of ChatGPT-5 and Gemini 3.1 Pro.

\section{Acknowledgments}
This research was supported by Frequentis and the Austrian Science Fund (FW), grant 10.557766/COE12.
Hecher has been supported by the French National Research Agency (ANR), grant ANR-25-CE23-7647-01, and the Austrian Science Fund (FWF), grant 10.55776/J4656.
\clearpage
\onecolumn
\input{appendix_heatmaps}

\onecolumn
\clearpage

\clearpage
\onecolumn
\input{large_table}
\onecolumn
\clearpage

%% file: appendix_tables.tex
\begin{table*}[t]
\centering
\tiny
\resizebox{\textwidth}{!}{%
\begin{tabular}{lrrrrrrrrrrrrrrrrrrrr}
\toprule
Scenario & \#I & \multicolumn{2}{c}{$ASP_{\neg r, \neg d, \neg s}$} & \multicolumn{2}{c}{$ASP_{\neg r, \neg d, sp}$} & \multicolumn{2}{c}{$ASP_{\neg r, \neg d, s}$} & \multicolumn{2}{c}{$ASP_{rp, \neg d, \neg s}$} & \multicolumn{2}{c}{$ASP_{rp, \neg d, sp}$} & \multicolumn{2}{c}{$ASP_{rp, \neg d, s}$} & \multicolumn{2}{c}{$ASP_{r, \neg d, \neg s}$} & \multicolumn{2}{c}{$ASP_{r, \neg d, sp}$} & \multicolumn{2}{c}{$ASP_{r, \neg d, s}$} & Draws \\
 &  & \#W & \#S & \#W & \#S & \#W & \#S & \#W & \#S & \#W & \#S & \#W & \#S & \#W & \#S & \#W & \#S & \#W & \#S & \#D \\
\midrule
CE-5x5 & 40 & 0 & 1 & 0 & 3 & 0 & 3 & 0 & 2 & 0 & 4 & 2 & 4 & 0 & 8 & 5 & 12 & 1 & 11 & 0 \\
EA-3x3 & 40 & 0 & 0 & 0 & 0 & 0 & 0 & 0 & 0 & 0 & 0 & 0 & 0 & 0 & 4 & 0 & 4 & 0 & 4 & 1 \\
IND-4x10 & 40 & 0 & 0 & 0 & 0 & 0 & 0 & 0 & 0 & 0 & 1 & 0 & 1 & 0 & 4 & 0 & 8 & 2 & 7 & 0 \\
EUR-10x10 & 40 & 0 & 0 & 0 & 1 & 0 & 1 & 0 & 0 & 3 & 6 & 2 & 3 & 0 & 9 & 0 & 11 & 0 & 5 & 0 \\
USA-7x7 & 40 & 0 & 0 & 0 & 0 & 0 & 0 & 0 & 0 & 0 & 1 & 0 & 1 & 0 & 5 & 2 & 13 & 2 & 8 & 0 \\
\midrule
Overall & 200 & 0 & 1 & 0 & 4 & 0 & 4 & 0 & 2 & 3 & 12 & 4 & 9 & 0 & 30 & 7 & 48 & 5 & 35 & 1 \\
\addlinespace[0.6em]
Scenario & \#I & \multicolumn{2}{c}{$ASP_{\neg r, dp, \neg s}$} & \multicolumn{2}{c}{$ASP_{\neg r, dp, sp}$} & \multicolumn{2}{c}{$ASP_{\neg r, dp, s}$} & \multicolumn{2}{c}{$ASP_{rp, dp, \neg s}$} & \multicolumn{2}{c}{$ASP_{rp, dp, sp}$} & \multicolumn{2}{c}{$ASP_{rp, dp, s}$} & \multicolumn{2}{c}{$ASP_{r, dp, \neg s}$} & \multicolumn{2}{c}{$ASP_{r, dp, sp}$} & \multicolumn{2}{c}{$ASP_{r, dp, s}$} & Draws \\
 &  & \#W & \#S & \#W & \#S & \#W & \#S & \#W & \#S & \#W & \#S & \#W & \#S & \#W & \#S & \#W & \#S & \#W & \#S & \#D \\
\midrule
CE-5x5 & 40 & 0 & 10 & 0 & 14 & 1 & 14 & 0 & 12 & 1 & 22 & 1 & 17 & 0 & 11 & 9 & 25 & 0 & 12 & 0 \\
EA-3x3 & 40 & 0 & 3 & 0 & 3 & 1 & 3 & 0 & 3 & 0 & 3 & 0 & 2 & 0 & 4 & 1 & 6 & 1 & 4 & 1 \\
IND-4x10 & 40 & 0 & 3 & 0 & 9 & 0 & 8 & 0 & 4 & 11 & 13 & 6 & 11 & 0 & 5 & 1 & 9 & 0 & 7 & 0 \\
EUR-10x10 & 40 & 0 & 4 & 0 & 21 & 0 & 6 & 0 & 8 & \textbf{30} & \textbf{31} & 0 & 8 & 0 & 9 & 0 & 8 & 0 & 3 & 0 \\
USA-7x7 & 40 & 0 & 5 & 0 & 14 & 1 & 11 & 0 & 6 & \textbf{13} & 19 & 3 & 17 & 0 & 6 & 0 & 10 & 1 & 7 & 0 \\
\midrule
Overall & 200 & 0 & 25 & 0 & 61 & 3 & 42 & 0 & 33 & 55 & 88 & 10 & 55 & 0 & 35 & 11 & 58 & 2 & 33 & 1 \\
\addlinespace[0.6em]
Scenario & \#I & \multicolumn{2}{c}{$ASP_{\neg r, d, \neg s}$} & \multicolumn{2}{c}{$ASP_{\neg r, d, sp}$} & \multicolumn{2}{c}{$ASP_{\neg r, d, s}$} & \multicolumn{2}{c}{$ASP_{rp, d, \neg s}$} & \multicolumn{2}{c}{$ASP_{rp, d, sp}$} & \multicolumn{2}{c}{$ASP_{rp, d, s}$} & \multicolumn{2}{c}{$ASP_{r, d, \neg s}$} & \multicolumn{2}{c}{$ASP_{r, d, sp}$} & \multicolumn{2}{c}{$ASP_{r, d, s}$} & Draws \\
 &  & \#W & \#S & \#W & \#S & \#W & \#S & \#W & \#S & \#W & \#S & \#W & \#S & \#W & \#S & \#W & \#S & \#W & \#S & \#D \\
\midrule
CE-5x5 & 40 & 0 & 16 & 1 & 25 & 0 & 17 & 0 & 20 & \textbf{14} & \textbf{30} & 0 & 21 & 0 & 12 & 2 & 24 & 3 & 12 & 0 \\
EA-3x3 & 40 & 8 & 11 & 11 & \textbf{12} & 0 & 8 & 2 & 9 & \textbf{14} & 11 & 0 & \textbf{12} & 0 & 4 & 0 & 6 & 1 & 4 & 1 \\
IND-4x10 & 40 & 0 & 10 & 0 & 11 & 0 & 9 & 0 & 10 & \textbf{18} & \textbf{17} & 1 & 12 & 0 & 7 & 1 & 9 & 0 & 7 & 0 \\
EUR-10x10 & 40 & 0 & 19 & 0 & 18 & 0 & 8 & 1 & 24 & 3 & 26 & 1 & 9 & 0 & 9 & 0 & 8 & 0 & 2 & 0 \\
USA-7x7 & 40 & 0 & 16 & 0 & 19 & 0 & 12 & 5 & 20 & 10 & \textbf{25} & 1 & 16 & 0 & 7 & 1 & 10 & 1 & 7 & 0 \\
\midrule
Overall & 200 & 8 & 72 & 12 & 85 & 0 & 54 & 8 & 83 & \textbf{59} & \textbf{109} & 3 & 70 & 0 & 39 & 4 & 57 & 5 & 32 & 1 \\
\bottomrule
\end{tabular}
}
\caption{Tournament wins by ASP scenario (lexicographic: overload, delay, sectors, reroutes, reconfigs).}
\label{tab:tournament-wins-asp}
\end{table*}

\begin{table}[t]
\centering
\scriptsize
\resizebox{\textwidth}{!}{%
\begin{tabular}{lrrrrrrrr}
\toprule
Scenario & Overload [\#] & Arrival Delay [h] & Sector-Number [\#] & Sector-Diff [\#] & Reroute [\#] & Reconfig [\#] & Execution-Time [s] & RAM-Usage [MB] \\
\midrule
East Asia (3x3) & $116.8 \pm 2.9$ & $155.6 \pm 5.5$ & $294.7 \pm 1.4$ & $42.1 \pm 1.2$ & $34.2 \pm 1.0$ & $73.0 \pm 1.9$ & $1623.950 \pm 16.201$ & $453 \pm 11$ \\
Central Europe (5x5) & $30.6 \pm 1.3$ & $93.5 \pm 3.8$ & $391.4 \pm 4.8$ & $152.1 \pm 3.8$ & $35.1 \pm 0.9$ & $274.8 \pm 6.4$ & $1589.719 \pm 17.455$ & $628 \pm 20$ \\
India Subcontinent (4x10) & $88.8 \pm 2.7$ & $137.4 \pm 5.3$ & $552.1 \pm 7.2$ & $305.9 \pm 8.0$ & $36.9 \pm 1.0$ & $483.1 \pm 11.1$ & $1617.627 \pm 16.295$ & $1690 \pm 72$ \\
USA East Coast (7x7) & $58.7 \pm 2.3$ & $124.7 \pm 5.5$ & $601.8 \pm 8.5$ & $416.7 \pm 10.8$ & $35.0 \pm 1.0$ & $648.3 \pm 15.1$ & $1614.960 \pm 16.388$ & $2068 \pm 95$ \\
Major Europe (10x10) & $55.0 \pm 2.2$ & $108.1 \pm 5.2$ & $1154.8 \pm 23.1$ & $794.0 \pm 22.7$ & $30.7 \pm 1.0$ & $1338.7 \pm 34.1$ & $1594.361 \pm 16.714$ & $6350 \pm 302$ \\
\bottomrule
\end{tabular}
}
\caption{Results aggregated by scenario (across all instances and options)}
\label{tab:summary-by-scenario}
\end{table}

\begin{table}[t]
\centering
\scriptsize
\resizebox{\textwidth}{!}{%
\begin{tabular}{lrrrrrrrr}
\toprule
Instance Size & Overload [\#] & Arrival Delay [h] & Sector-Number [\#] & Sector-Diff [\#] & Reroute [\#] & Reconfig [\#] & Execution-Time [s] & RAM-Usage [MB] \\
\midrule
10 & $0.8 \pm 0.1$ & $3.4 \pm 0.3$ & $498.2 \pm 14.1$ & $354.2 \pm 19.0$ & $3.5 \pm 0.1$ & $507.7 \pm 26.2$ & $1205.494 \pm 36.230$ & $550 \pm 29$ \\
20 & $4.5 \pm 0.4$ & $20.3 \pm 1.3$ & $559.4 \pm 18.8$ & $378.1 \pm 21.3$ & $9.9 \pm 0.3$ & $562.4 \pm 29.7$ & $1384.554 \pm 31.683$ & $883 \pm 62$ \\
30 & $16.1 \pm 0.9$ & $51.8 \pm 2.9$ & $577.4 \pm 19.3$ & $357.0 \pm 19.9$ & $17.5 \pm 0.4$ & $561.1 \pm 29.5$ & $1606.870 \pm 23.702$ & $1181 \pm 91$ \\
40 & $36.4 \pm 1.6$ & $89.0 \pm 4.0$ & $589.9 \pm 19.5$ & $346.7 \pm 19.4$ & $25.7 \pm 0.6$ & $551.4 \pm 29.2$ & $1666.500 \pm 20.139$ & $1533 \pm 124$ \\
50 & $54.2 \pm 2.1$ & $113.0 \pm 5.2$ & $602.5 \pm 21.7$ & $315.2 \pm 18.9$ & $31.6 \pm 0.8$ & $562.0 \pm 31.9$ & $1684.741 \pm 18.667$ & $2285 \pm 222$ \\
60 & $76.1 \pm 2.7$ & $145.1 \pm 6.3$ & $613.5 \pm 21.9$ & $319.7 \pm 19.2$ & $38.2 \pm 1.0$ & $563.7 \pm 31.8$ & $1693.949 \pm 17.846$ & $2655 \pm 257$ \\
70 & $98.5 \pm 3.0$ & $164.5 \pm 7.0$ & $598.3 \pm 20.0$ & $311.4 \pm 18.5$ & $45.2 \pm 1.2$ & $544.8 \pm 29.5$ & $1711.340 \pm 16.577$ & $2539 \pm 207$ \\
80 & $119.1 \pm 3.5$ & $206.8 \pm 9.1$ & $626.8 \pm 22.4$ & $315.4 \pm 19.6$ & $52.3 \pm 1.4$ & $552.0 \pm 32.0$ & $1707.297 \pm 16.581$ & $3470 \pm 308$ \\
90 & $143.1 \pm 3.8$ & $218.5 \pm 9.1$ & $605.6 \pm 21.3$ & $312.0 \pm 19.4$ & $60.2 \pm 1.6$ & $536.2 \pm 30.5$ & $1712.087 \pm 16.230$ & $3321 \pm 294$ \\
100 & $170.2 \pm 4.3$ & $251.8 \pm 10.9$ & $617.4 \pm 22.1$ & $314.2 \pm 20.4$ & $66.0 \pm 1.8$ & $539.0 \pm 31.7$ & $1708.402 \pm 16.312$ & $3961 \pm 345$ \\
\bottomrule
\end{tabular}
}
\caption{Results aggregated by instance size (across all scenarios and options)}
\label{tab:summary-by-size}
\end{table}

%% file: appendix_heatmaps.tex
\begin{figure}
    \centering
    \begin{subfigure}{0.5\textwidth}
        \includegraphics[width=8cm]{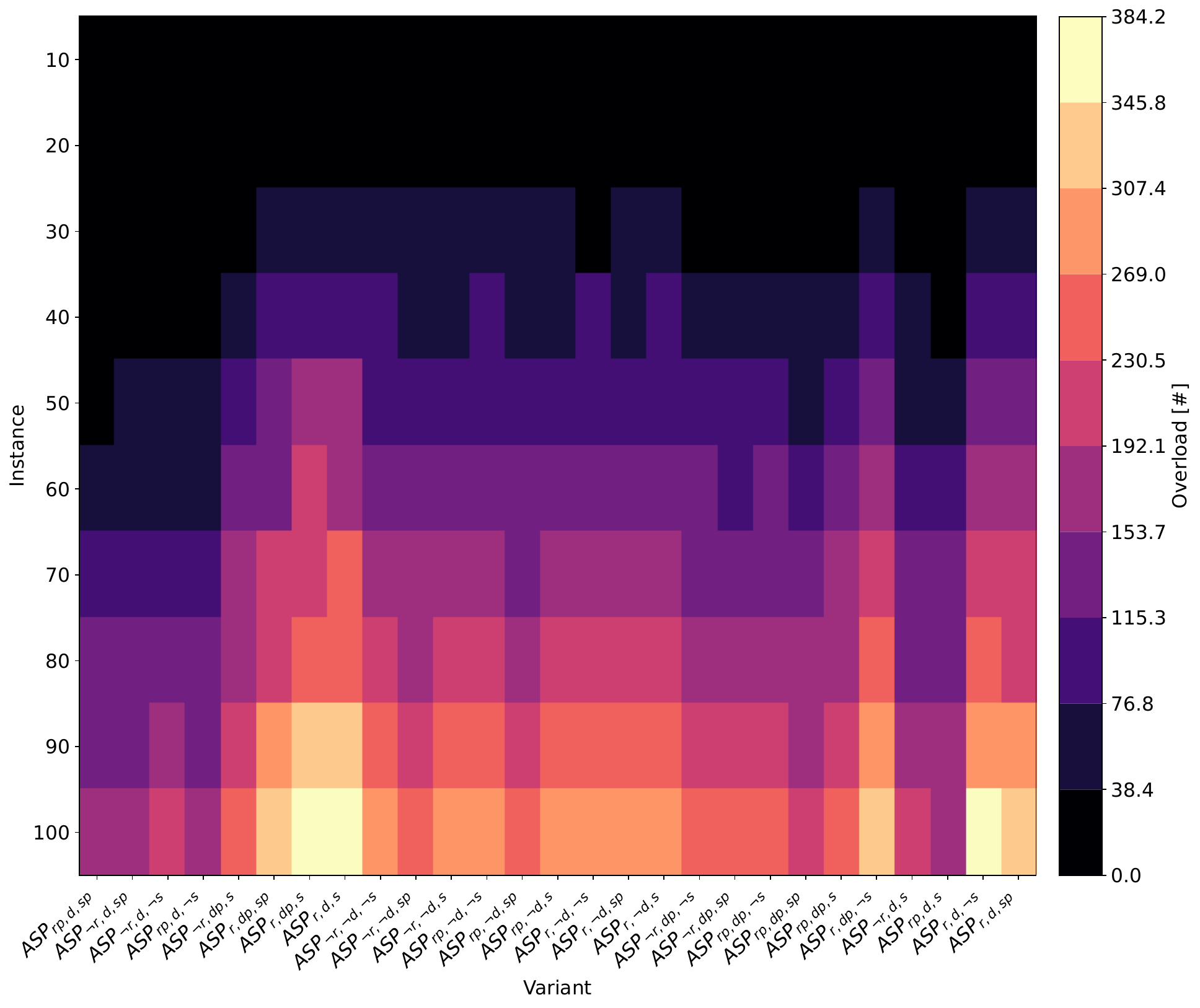}
        \caption{Overload.}        
    \end{subfigure}
    \begin{subfigure}{0.49\textwidth}
        \includegraphics[width=8cm]{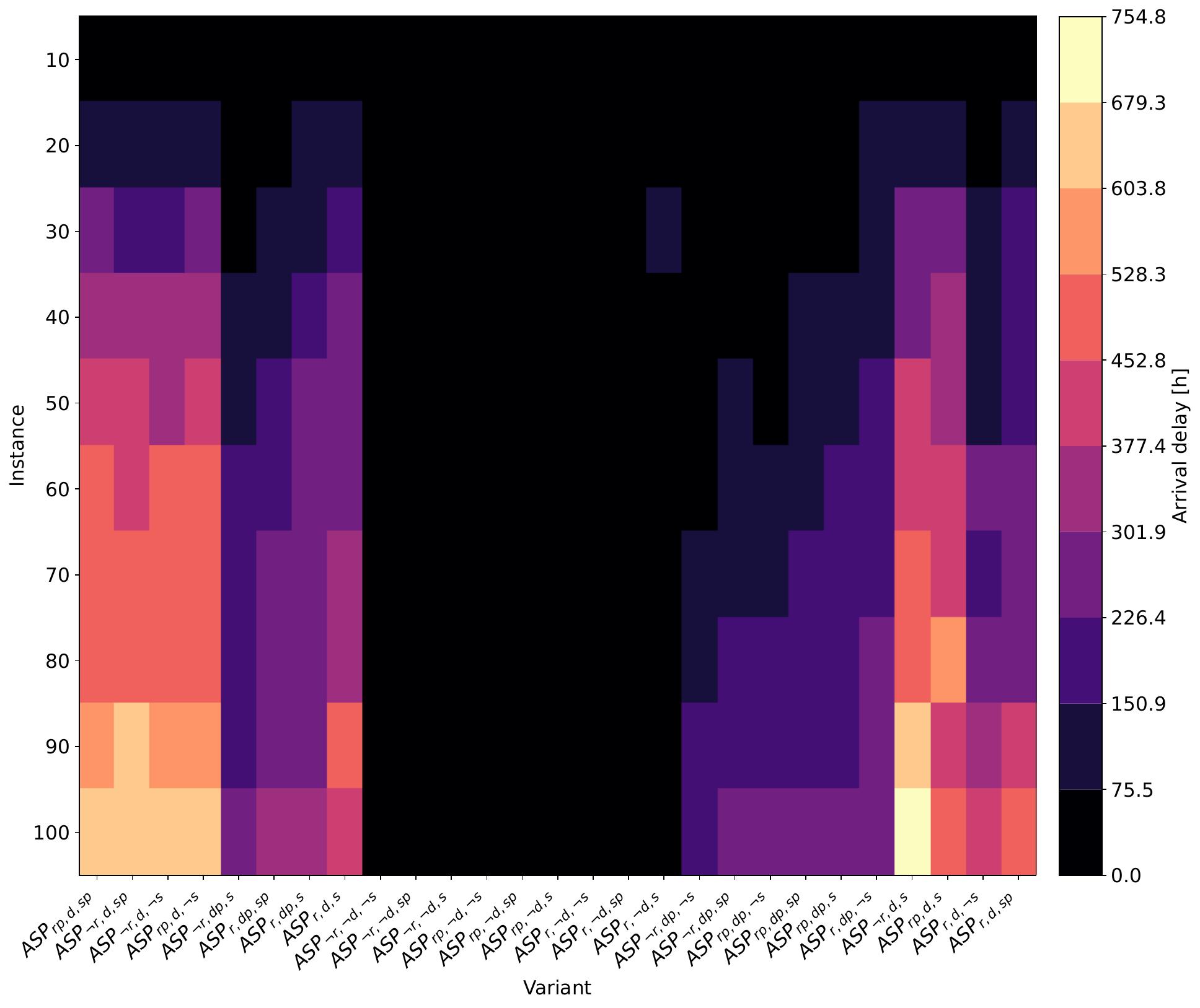}
        \caption{Arrival Delay.}        
    \end{subfigure}
    
    \begin{subfigure}{0.5\textwidth}
        \includegraphics[width=8cm]{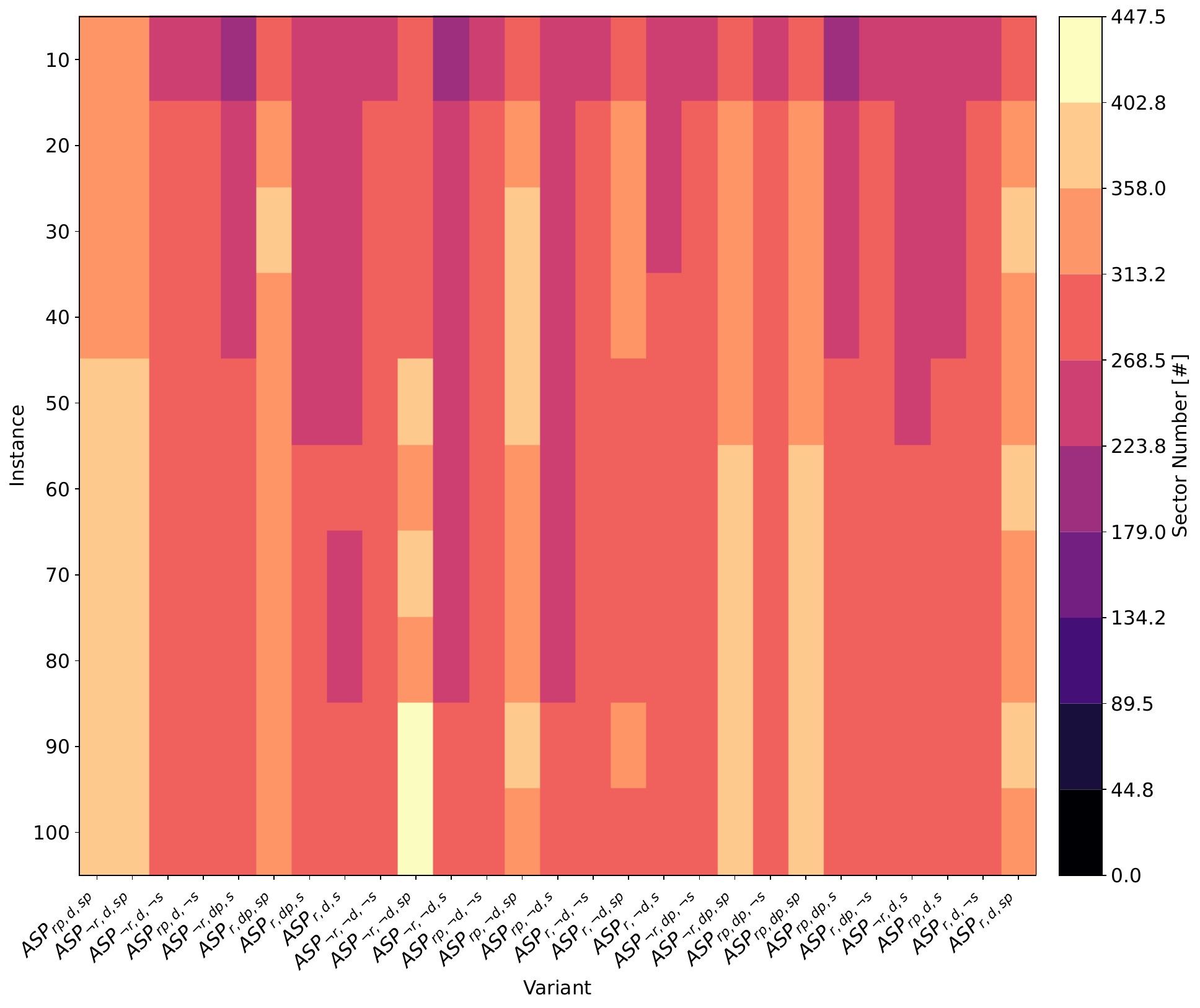}
        \caption{Sector Number.}        
    \end{subfigure}
    \begin{subfigure}{0.49\textwidth}
        \includegraphics[width=8cm]{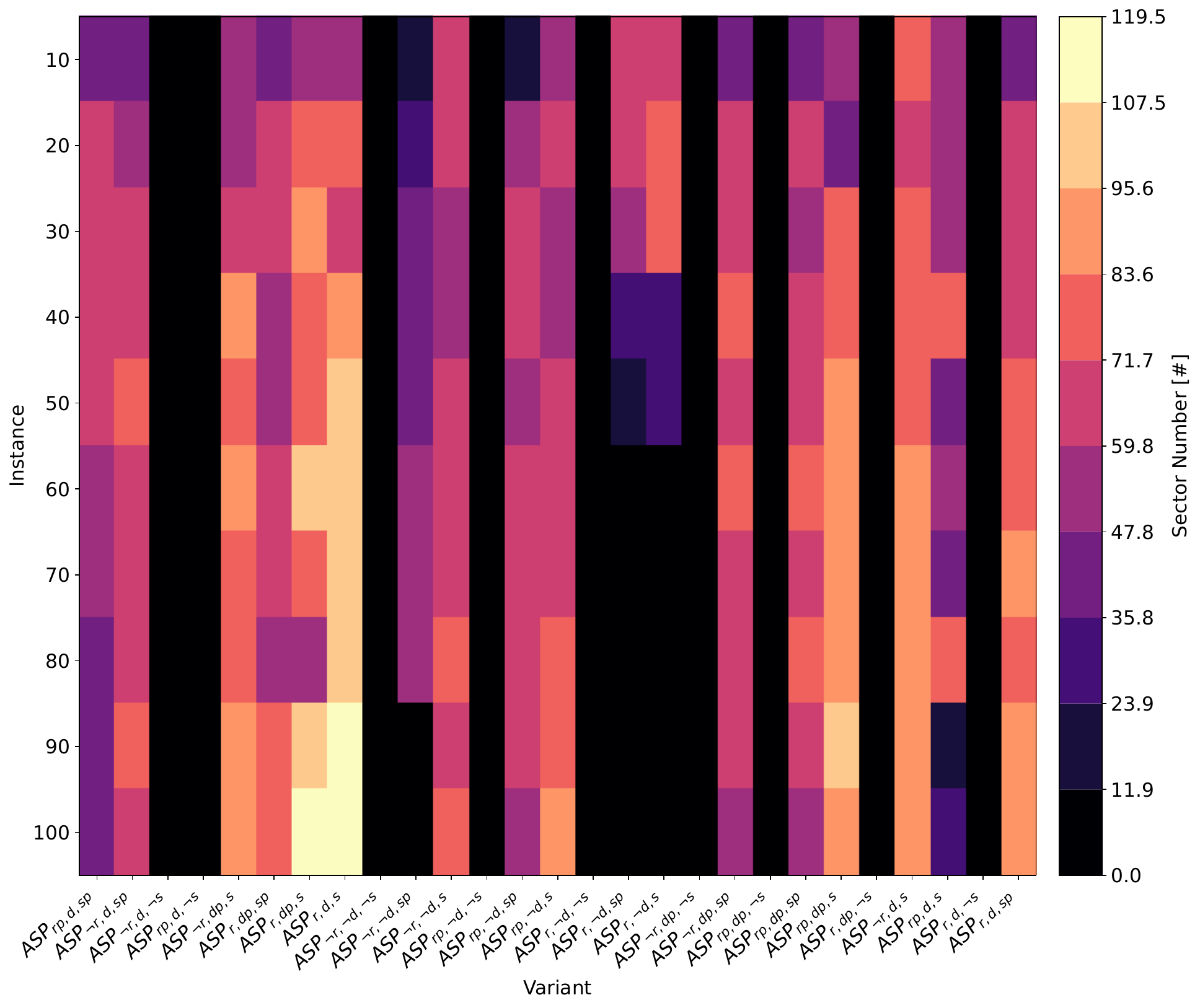}
        \caption{Sector Diff.}        
    \end{subfigure}
    
    \begin{subfigure}{0.5\textwidth}
        \includegraphics[width=8cm]{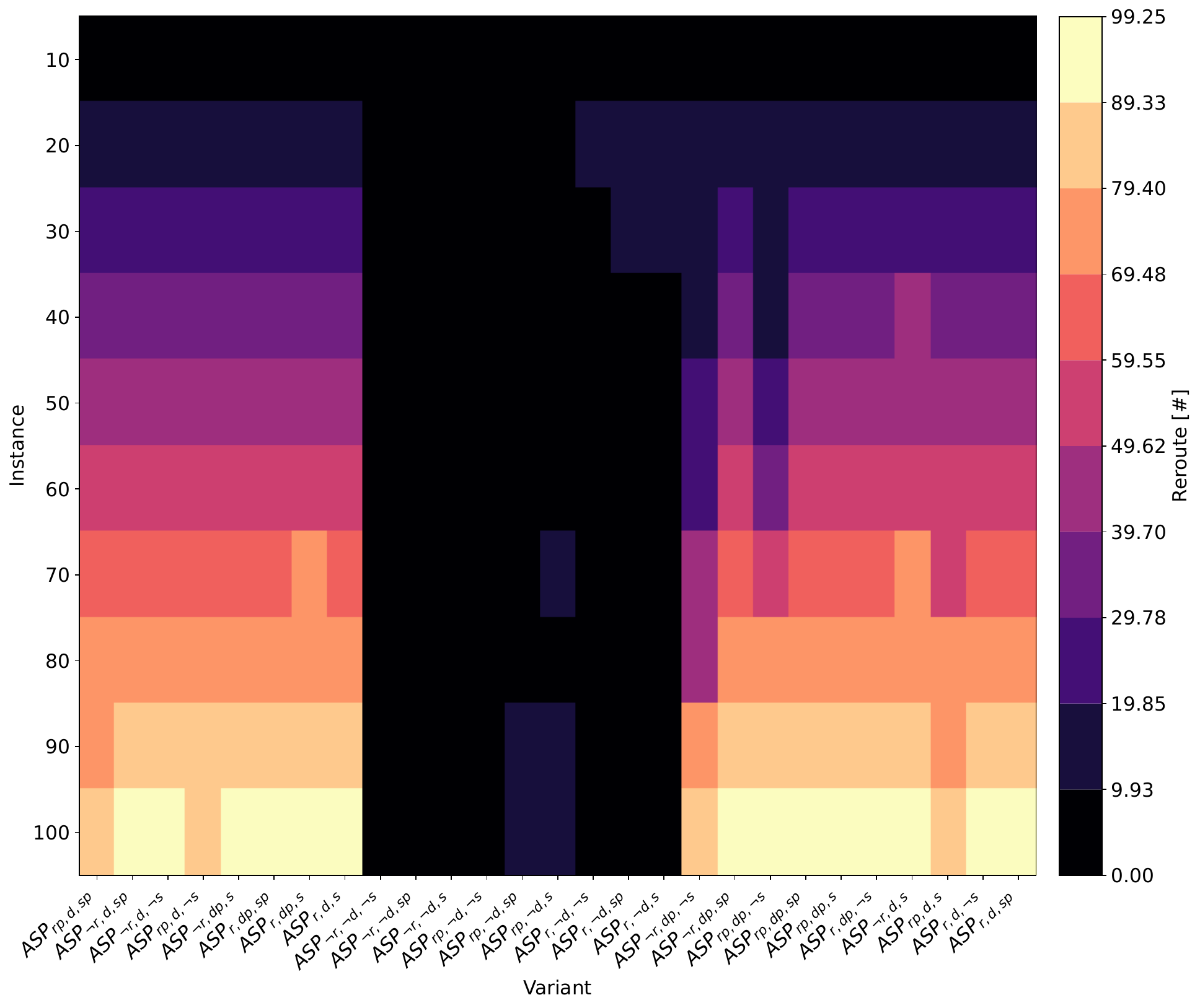}
        \caption{Reroute.}        
    \end{subfigure}
    \begin{subfigure}{0.49\textwidth}
        \includegraphics[width=8cm]{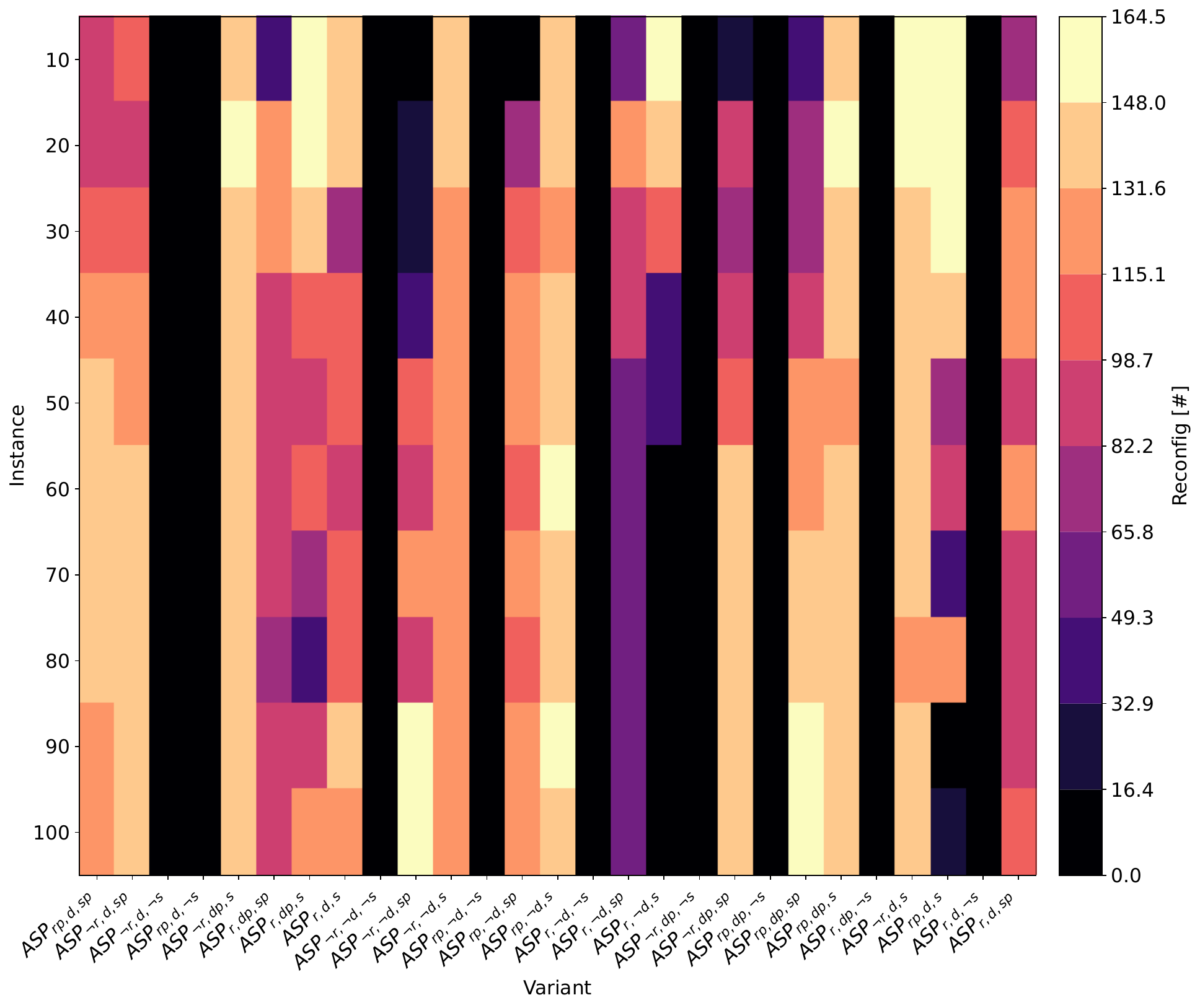}
        \caption{Reconfig.}        
    \end{subfigure}
    \caption{East Asia (3x3)}
    \label{fig:heat-maps-east-asia}
\end{figure}

\begin{figure}
    \centering
    \begin{subfigure}{0.5\textwidth}
        \includegraphics[width=8cm]{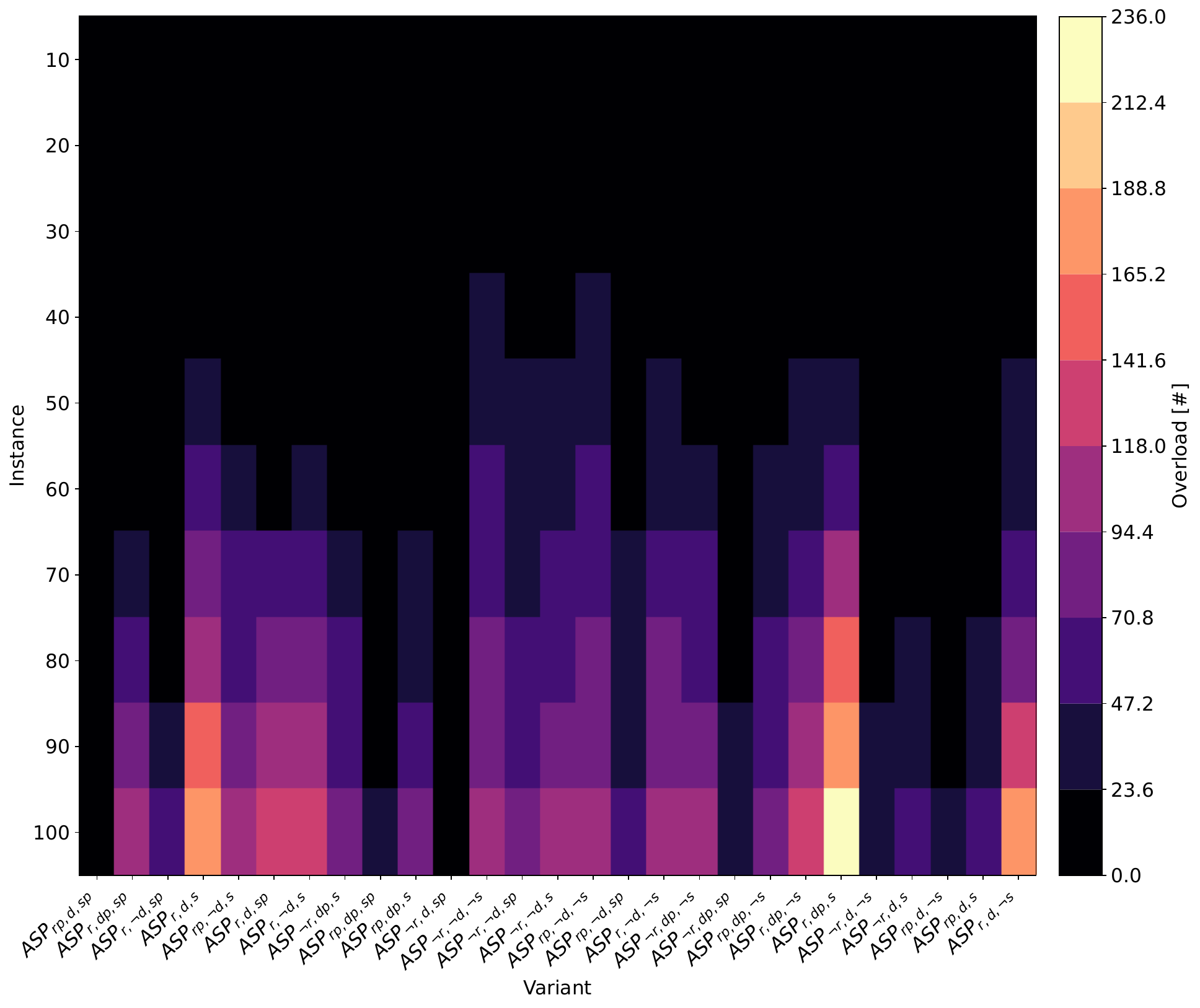}
        \caption{Overload.}        
    \end{subfigure}
    \begin{subfigure}{0.49\textwidth}
        \includegraphics[width=8cm]{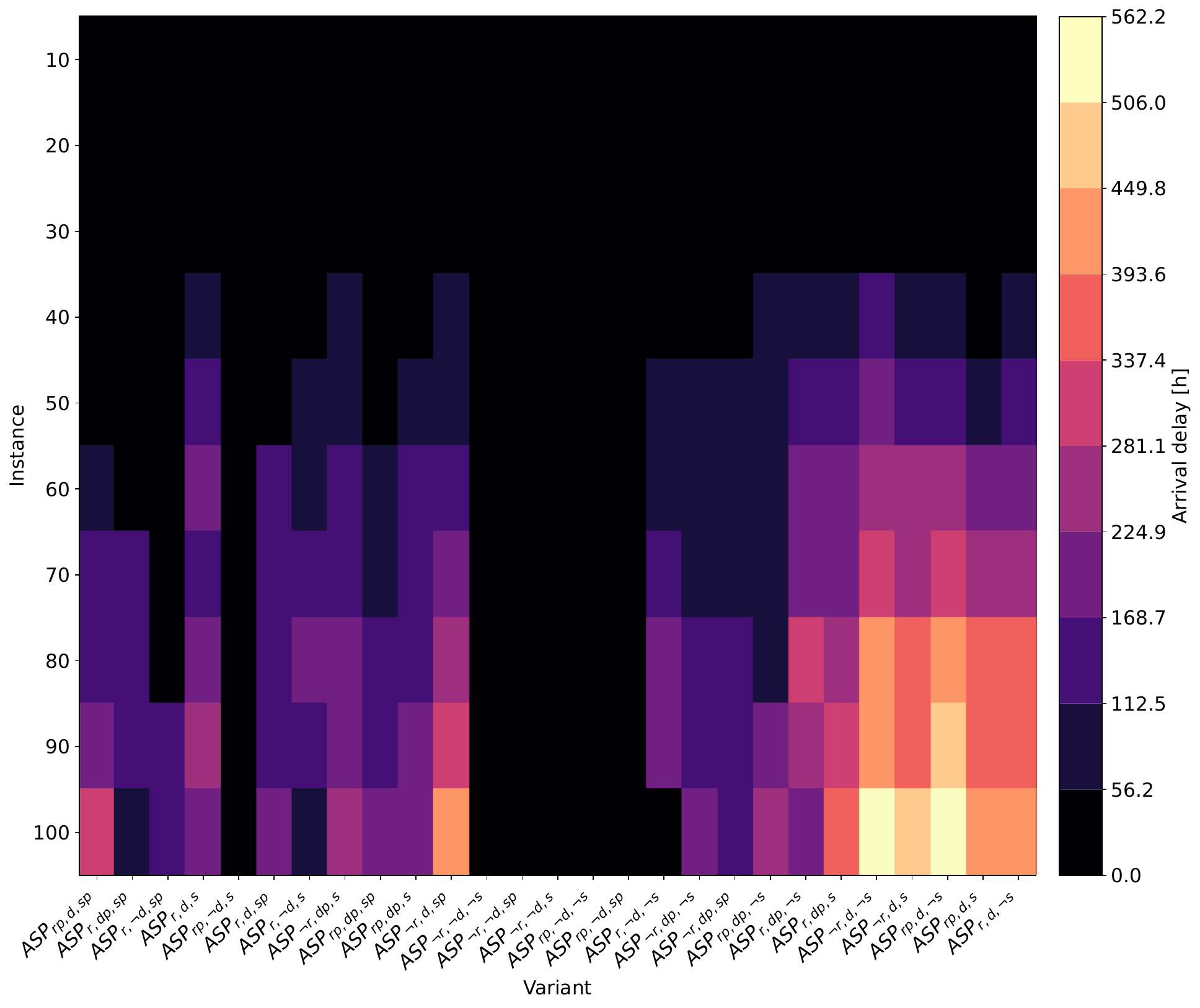}
        \caption{Arrival Delay.}        
    \end{subfigure}
    
    \begin{subfigure}{0.5\textwidth}
        \includegraphics[width=8cm]{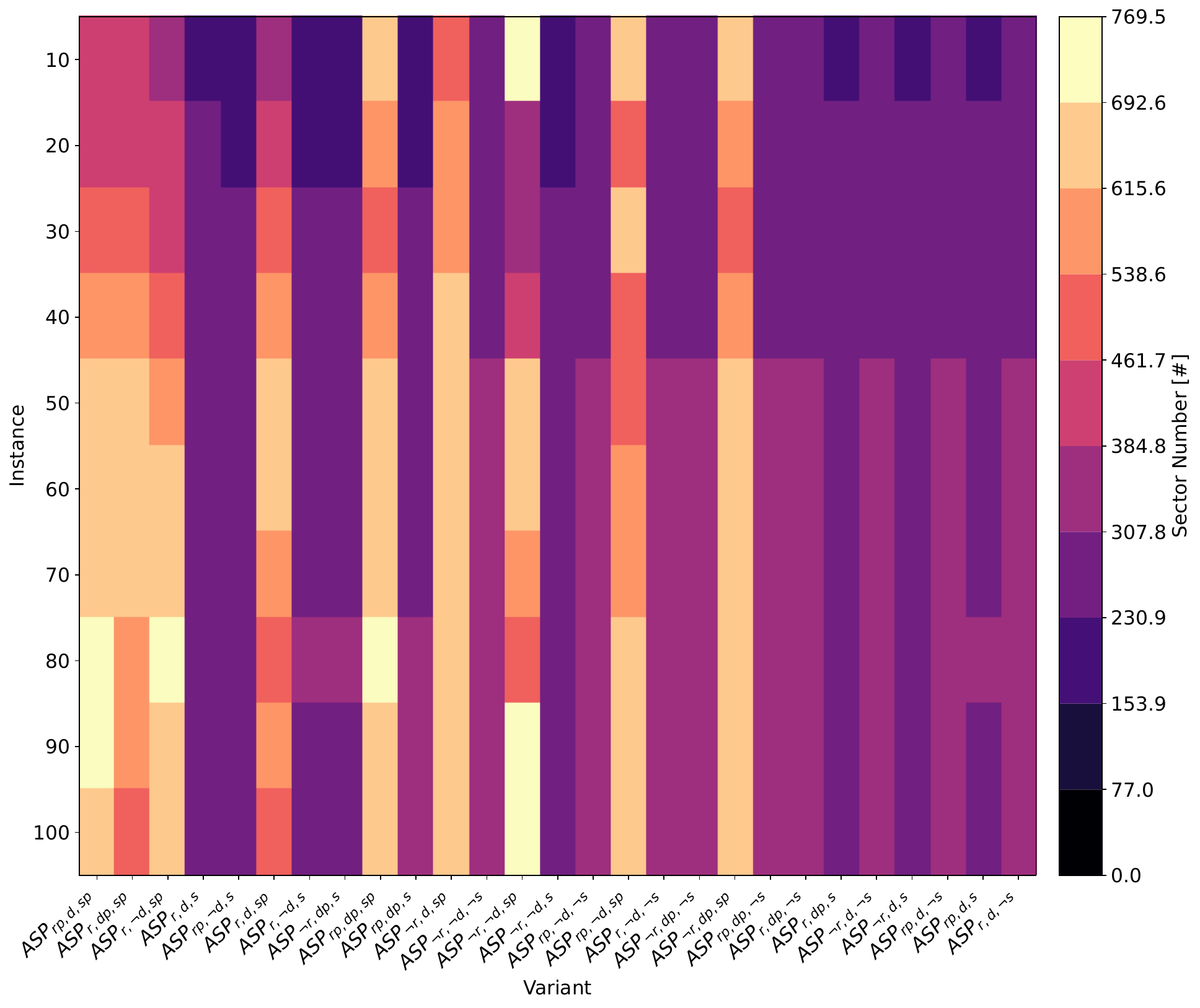}
        \caption{Sector Number.}        
    \end{subfigure}
    \begin{subfigure}{0.49\textwidth}
        \includegraphics[width=8cm]{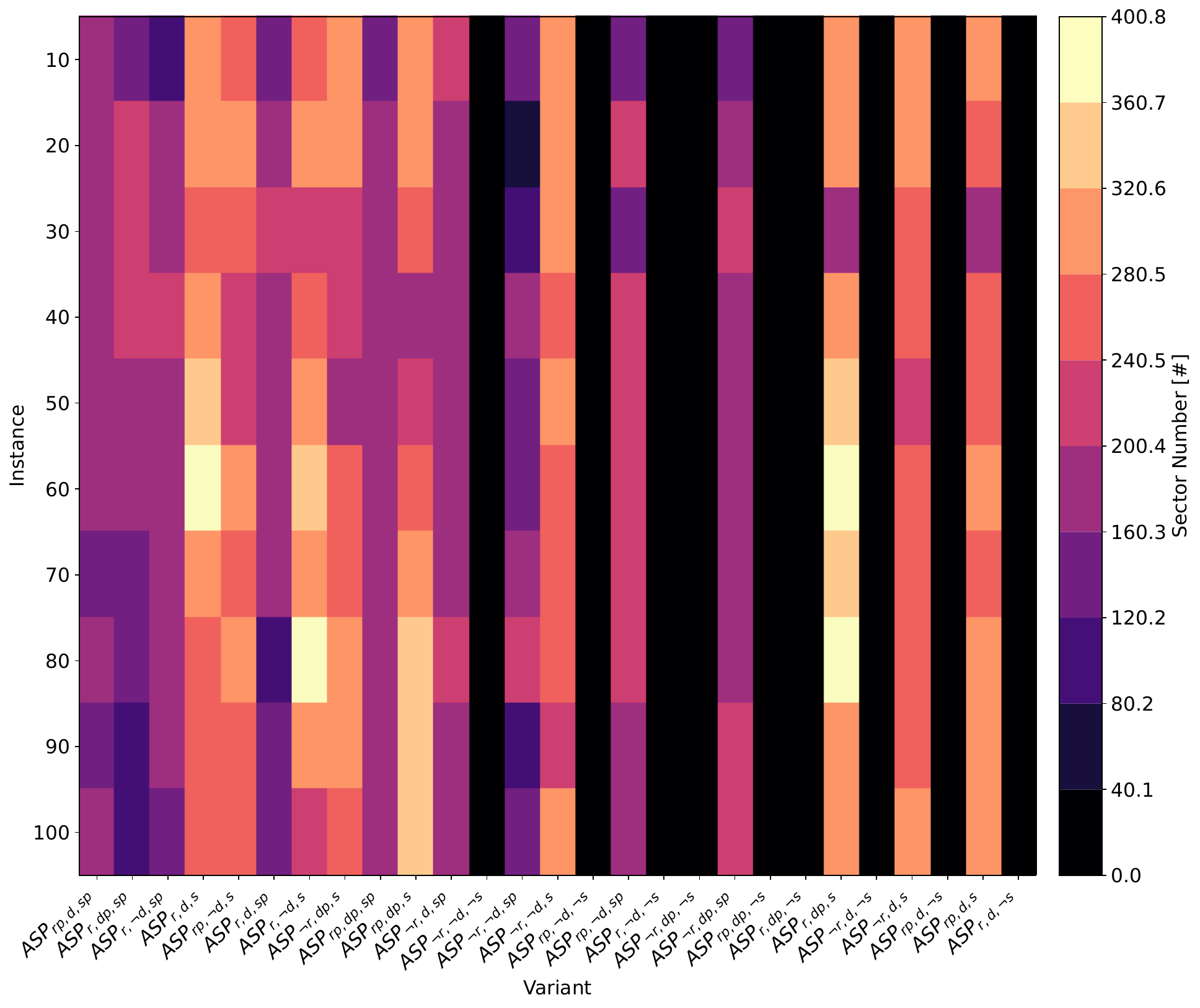}
        \caption{Sector Diff.}        
    \end{subfigure}
    
    \begin{subfigure}{0.5\textwidth}
        \includegraphics[width=8cm]{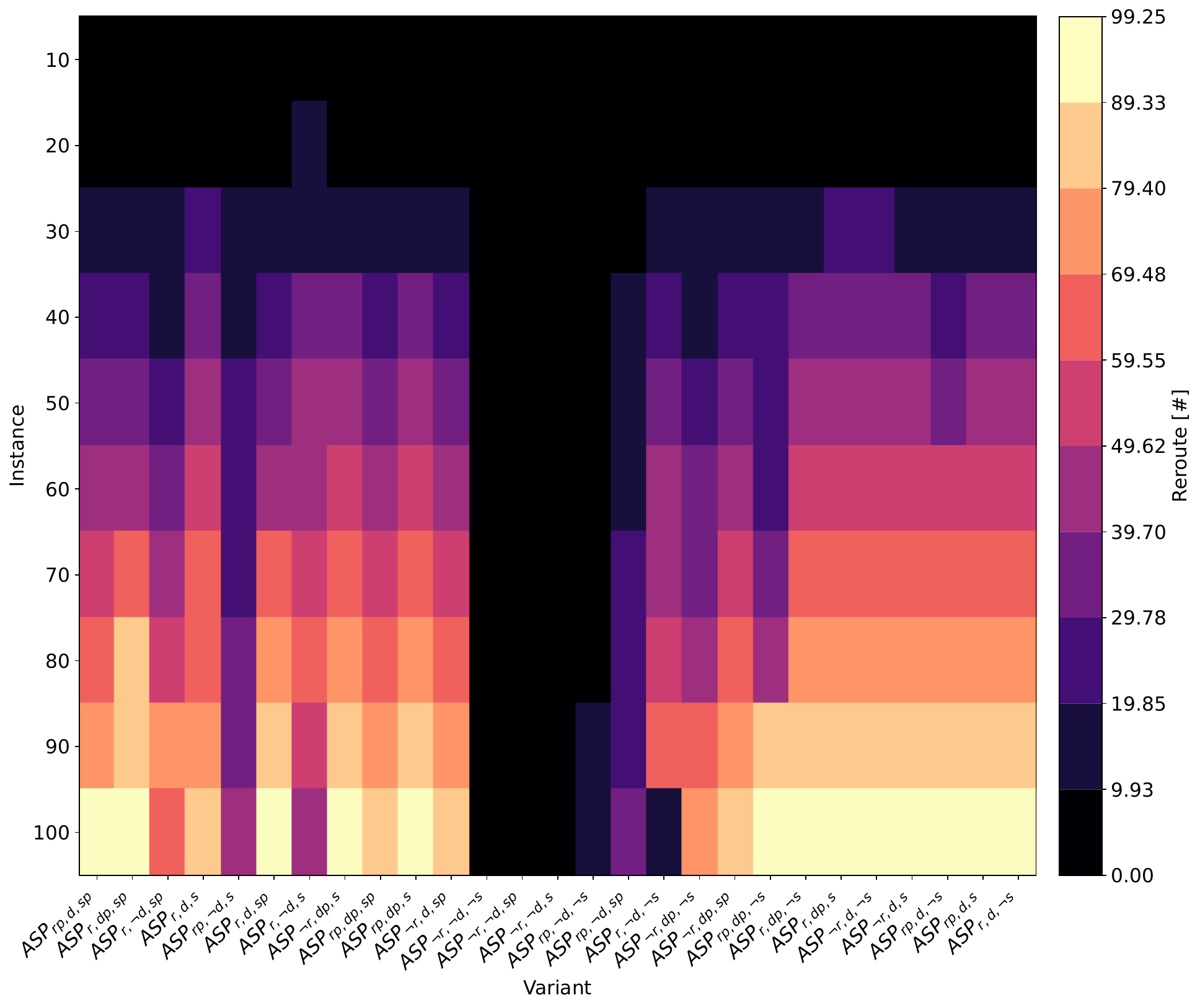}
        \caption{Reroute.}        
    \end{subfigure}
    \begin{subfigure}{0.49\textwidth}
        \includegraphics[width=8cm]{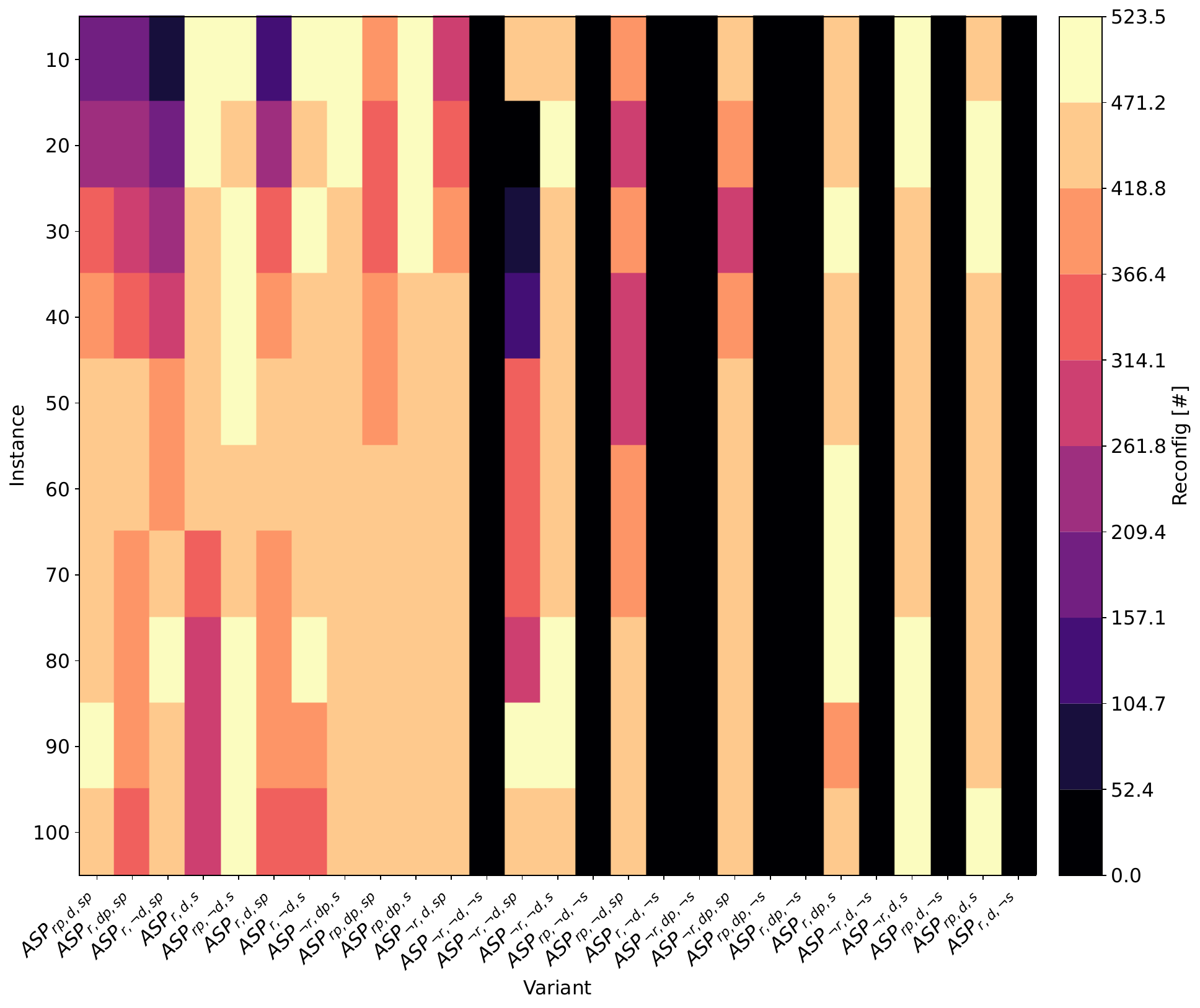}
        \caption{Reconfig.}        
    \end{subfigure}
    \caption{Central Europe (5x5)}
    \label{fig:heatmaps-central-europe}
\end{figure}

\begin{figure}
    \centering
    \begin{subfigure}{0.5\textwidth}
        \includegraphics[width=8cm]{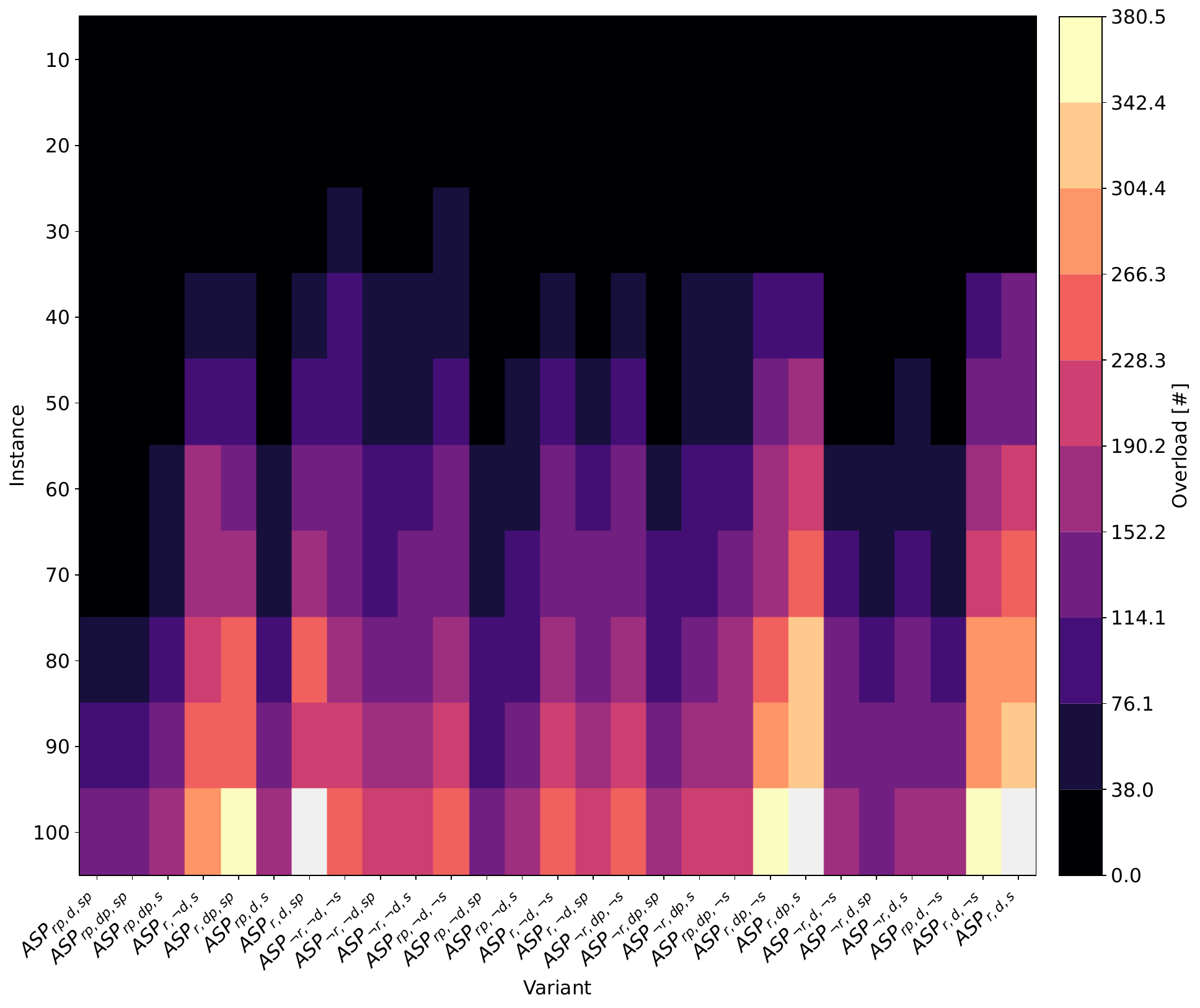}
        \caption{Overload.}        
    \end{subfigure}
    \begin{subfigure}{0.49\textwidth}
        \includegraphics[width=8cm]{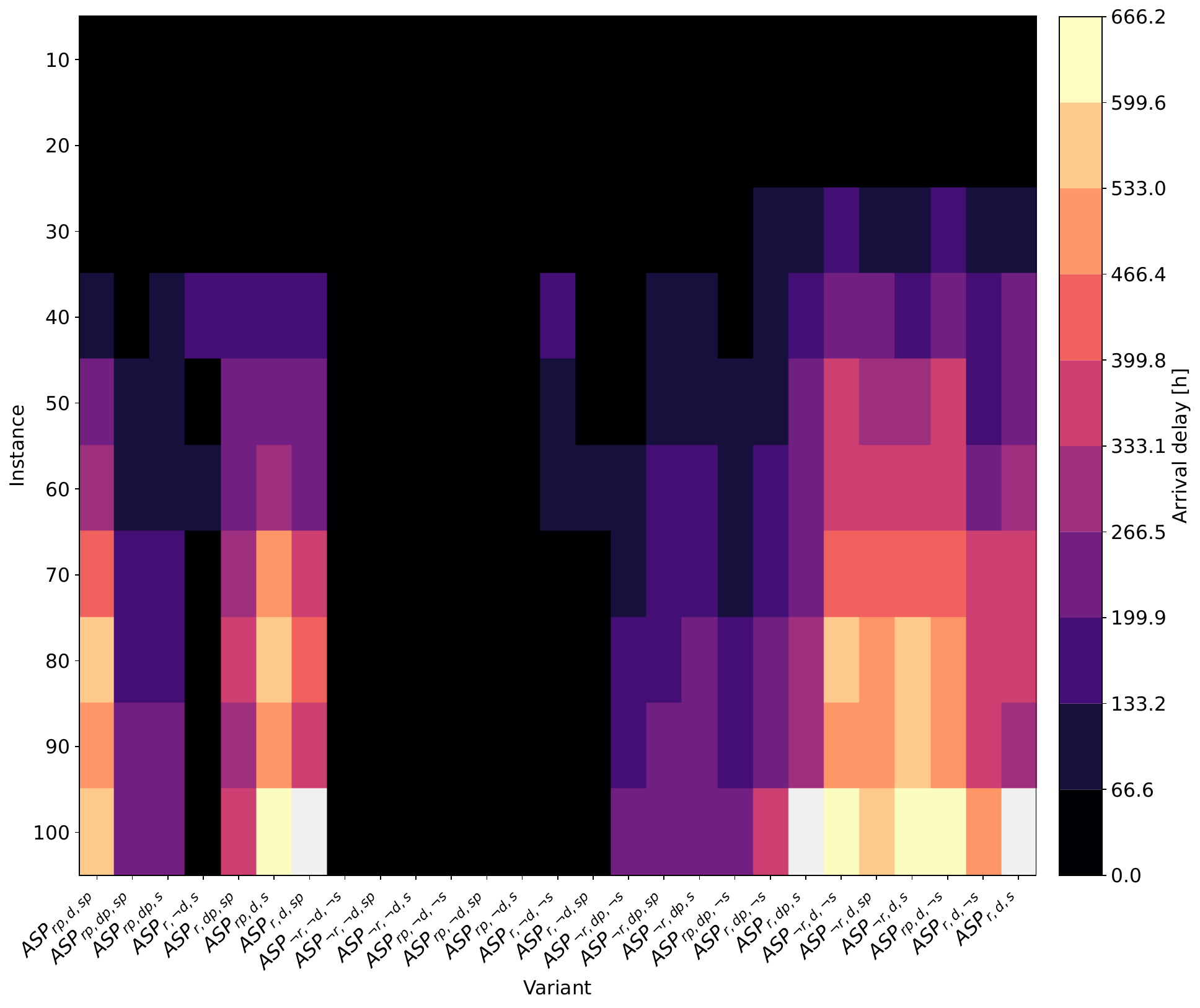}
        \caption{Arrival Delay.}        
    \end{subfigure}
    
    \begin{subfigure}{0.5\textwidth}
        \includegraphics[width=8cm]{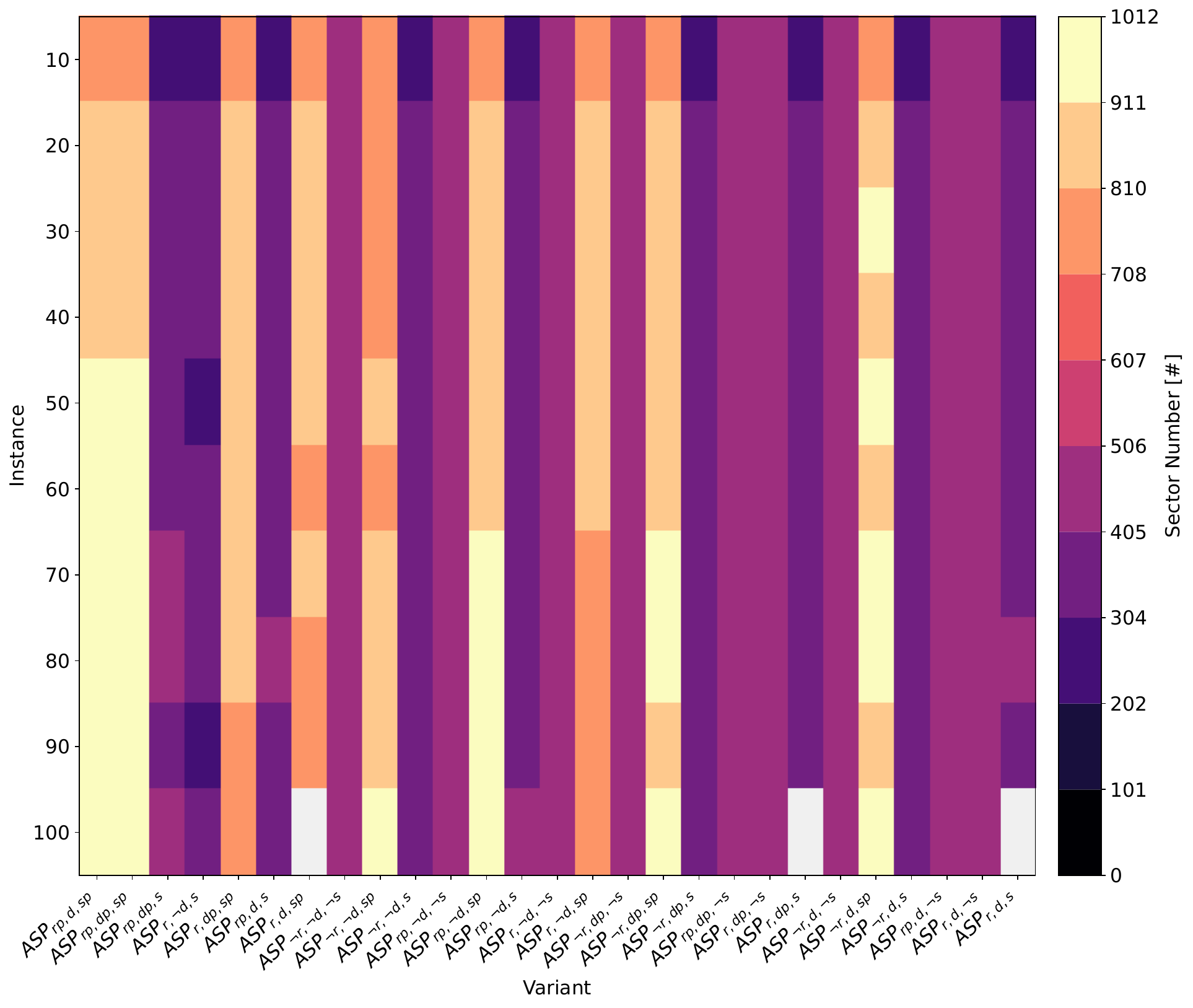}
        \caption{Sector Number.}        
    \end{subfigure}
    \begin{subfigure}{0.49\textwidth}
        \includegraphics[width=8cm]{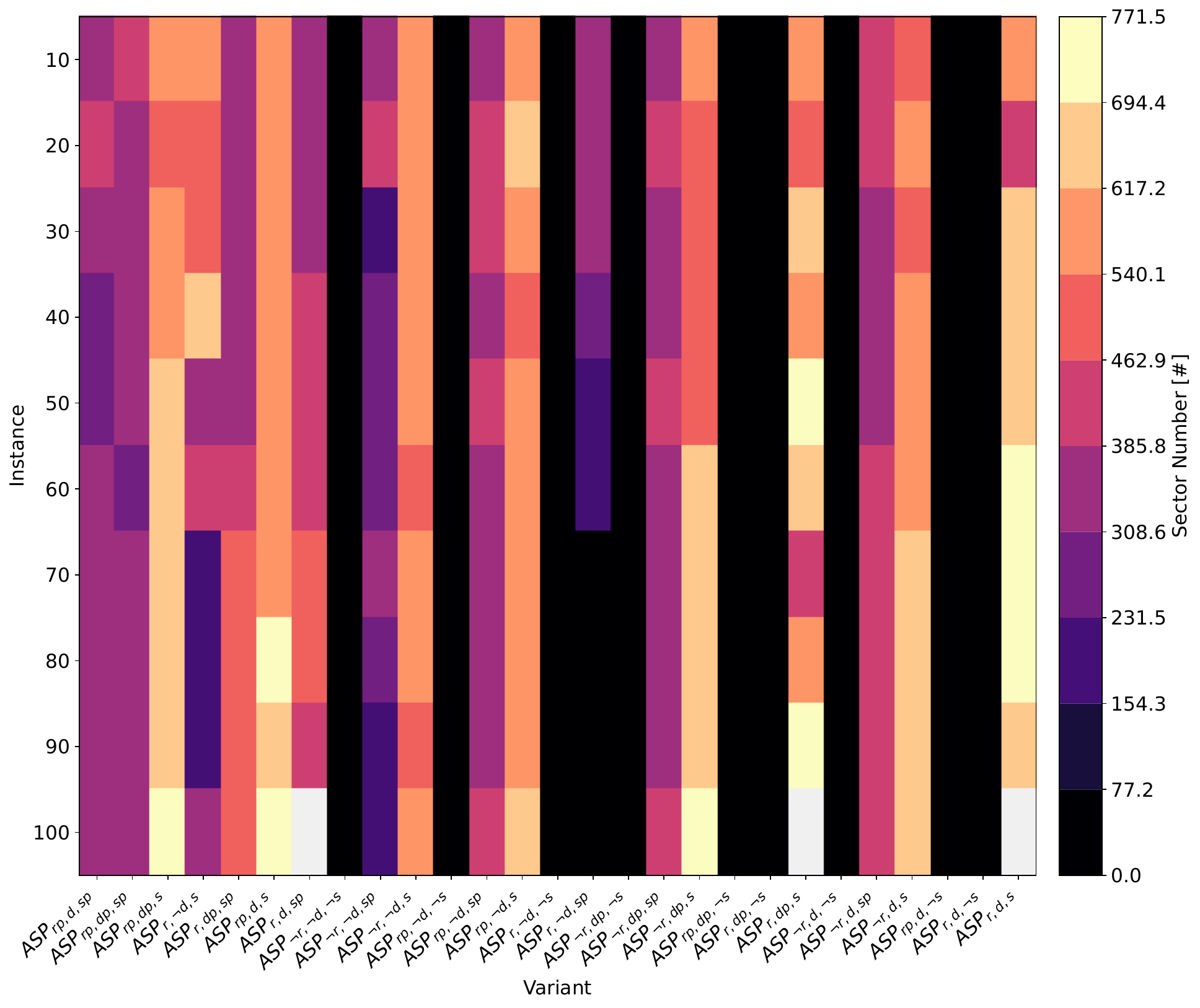}
        \caption{Sector Diff.}        
    \end{subfigure}
    
    \begin{subfigure}{0.5\textwidth}
        \includegraphics[width=8cm]{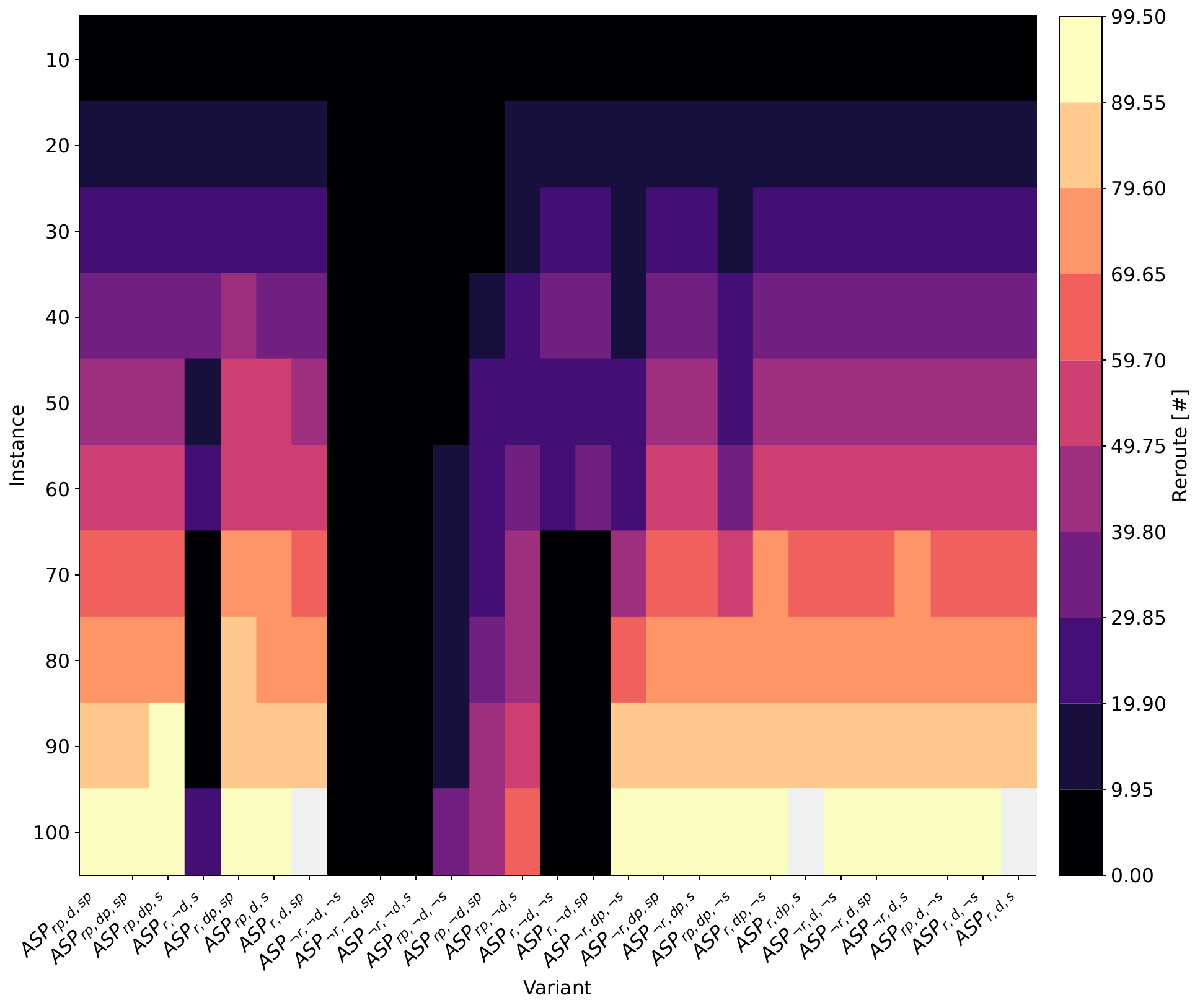}
        \caption{Reroute.}        
    \end{subfigure}
    \begin{subfigure}{0.49\textwidth}
        \includegraphics[width=8cm]{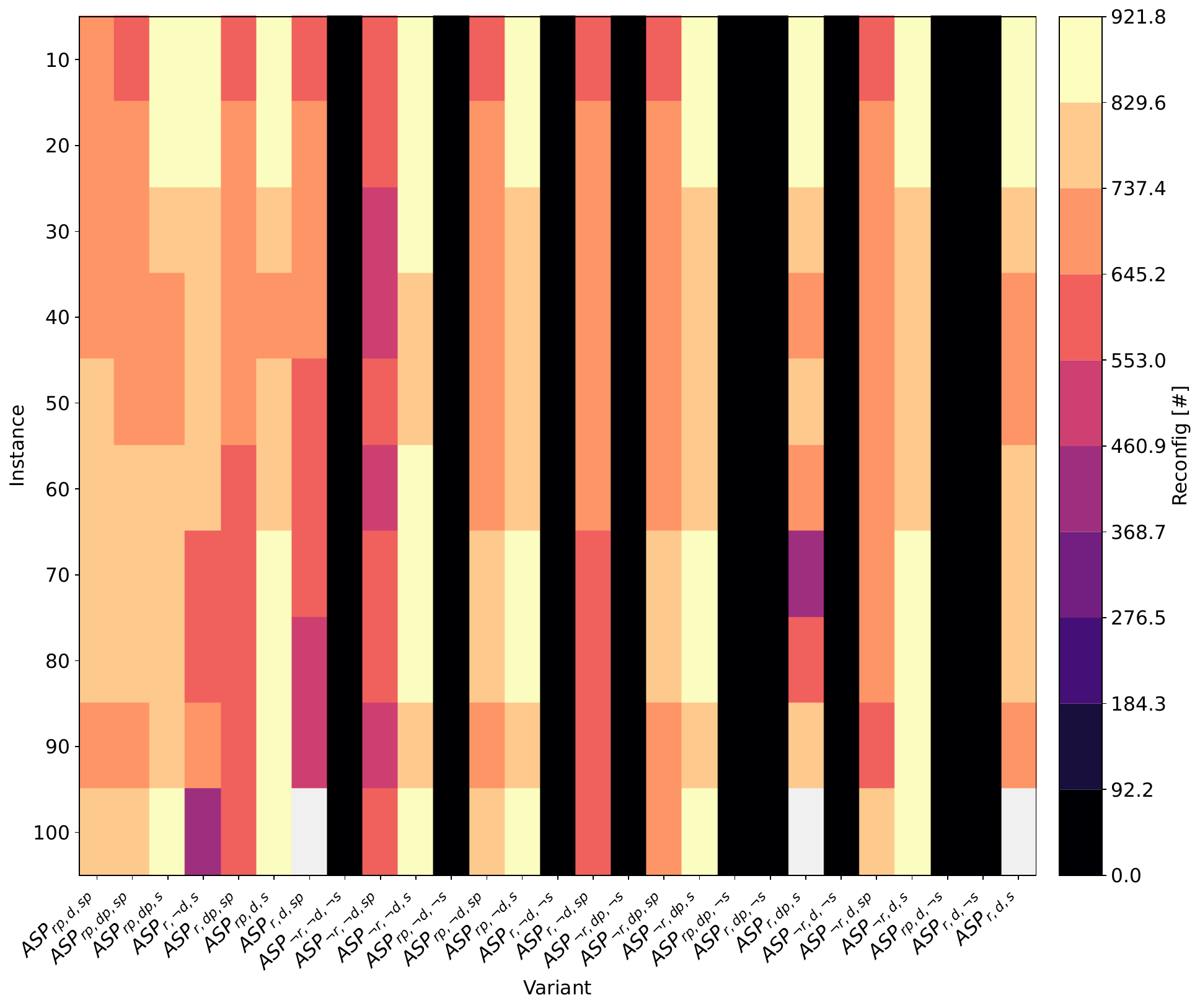}
        \caption{Reconfig.}        
    \end{subfigure}
    \caption{India (4x10)}
    \label{fig:heatmaps-india}
\end{figure}

\begin{figure}
    \centering
    \begin{subfigure}{0.5\textwidth}
        \includegraphics[width=8cm]{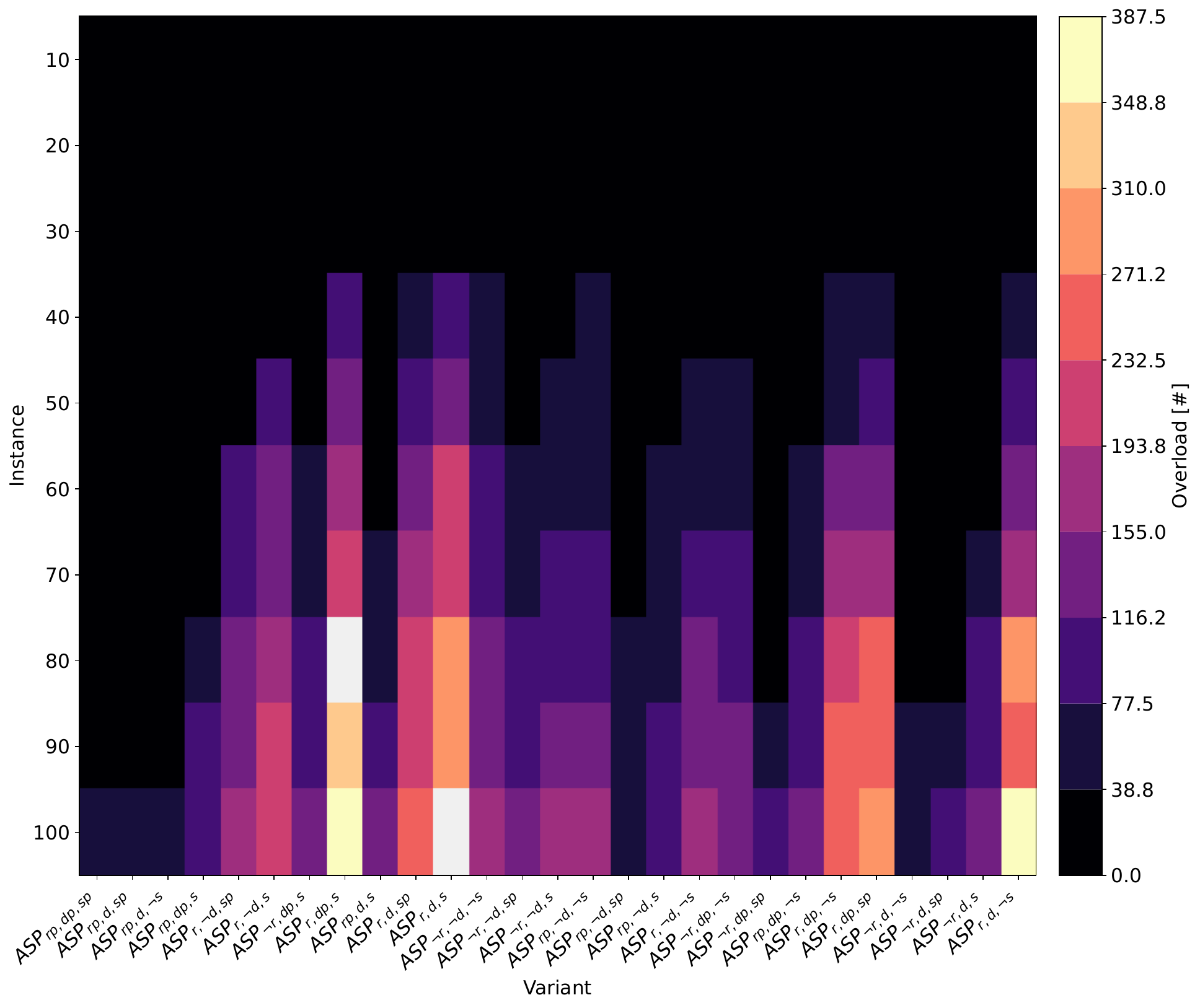}
        \caption{Overload.}        
    \end{subfigure}
    \begin{subfigure}{0.49\textwidth}
        \includegraphics[width=8cm]{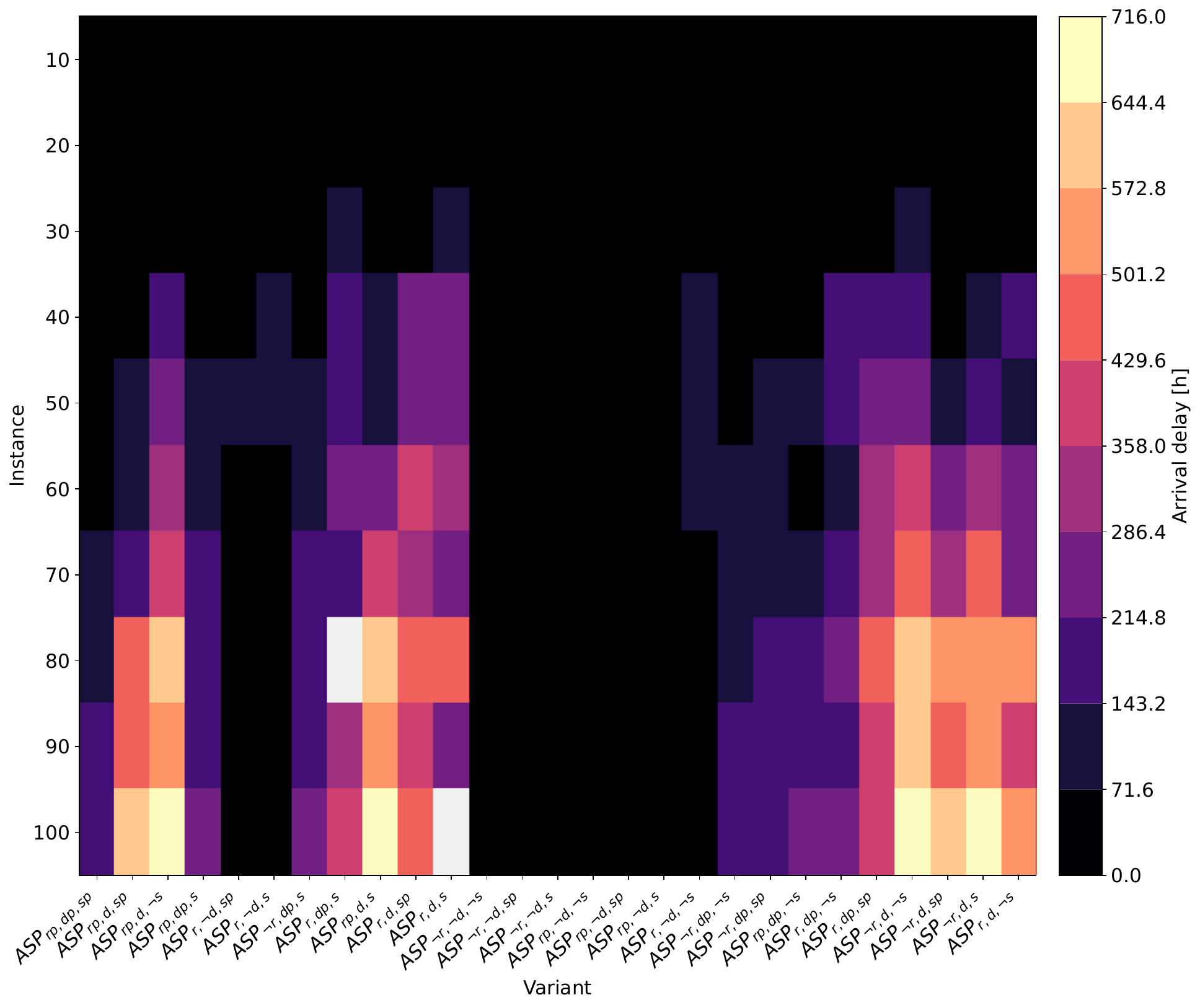}
        \caption{Arrival Delay.}        
    \end{subfigure}
    
    \begin{subfigure}{0.5\textwidth}
        \includegraphics[width=8cm]{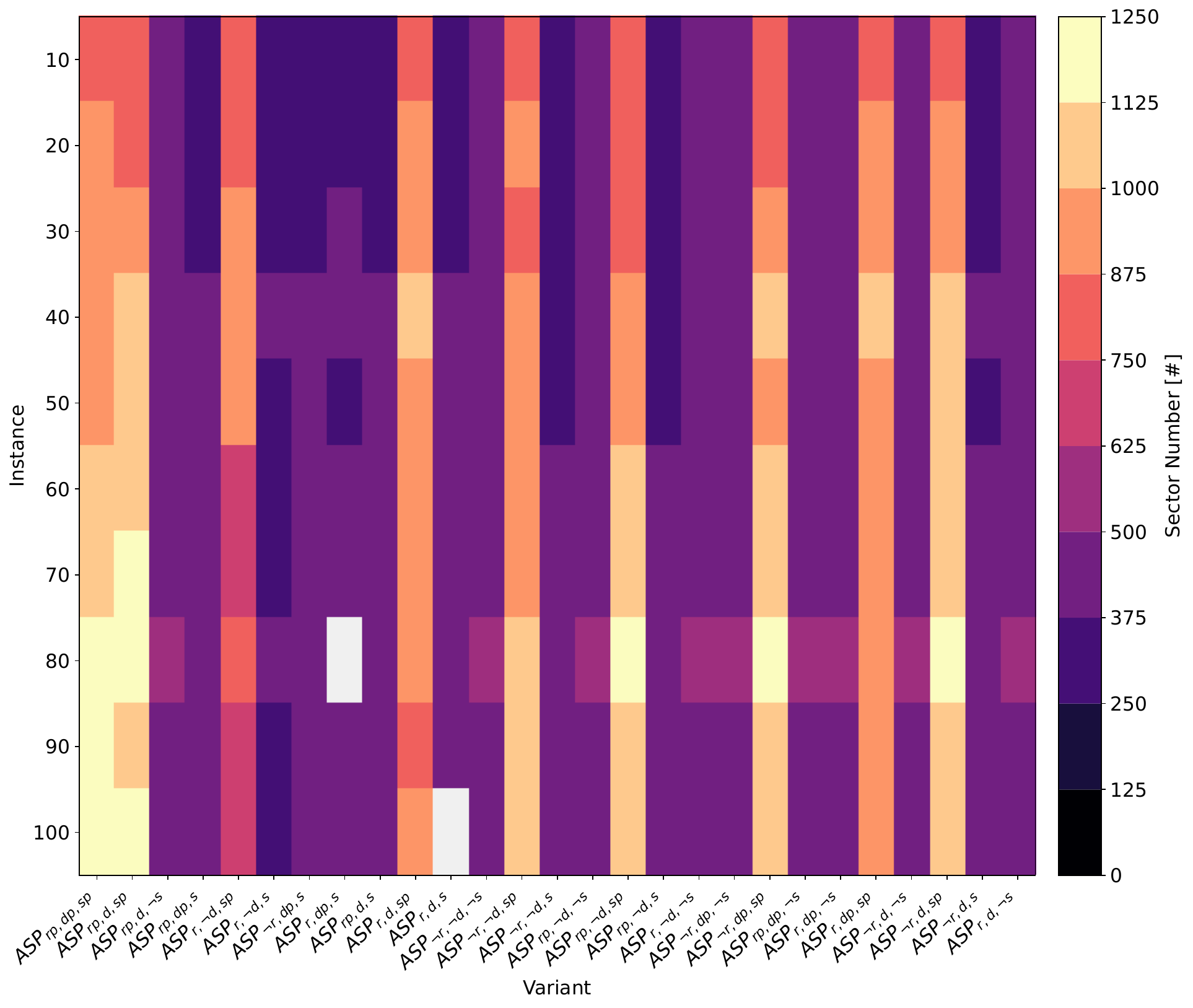}
        \caption{Sector Number.}        
    \end{subfigure}
    \begin{subfigure}{0.49\textwidth}
        \includegraphics[width=8cm]{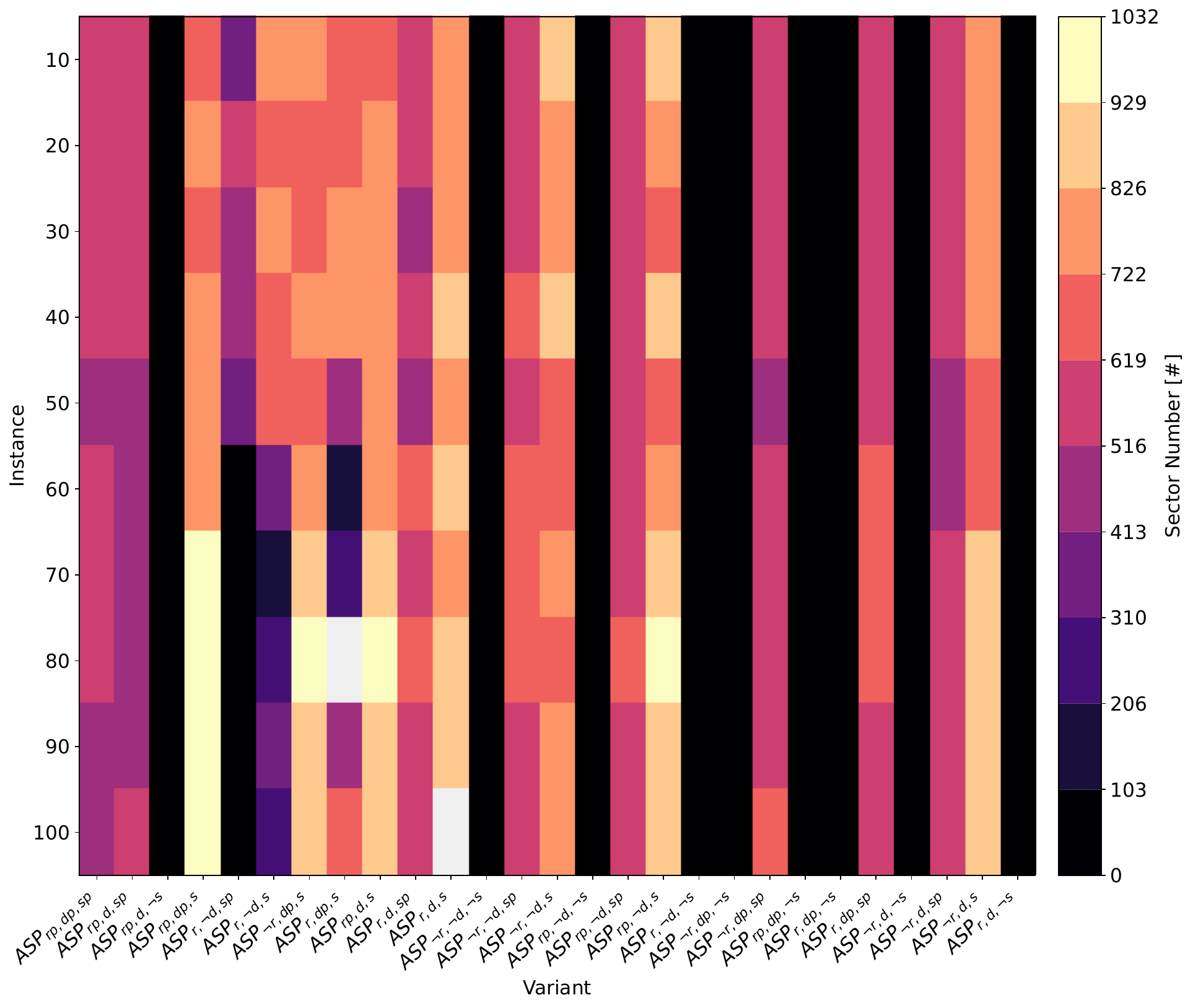}
        \caption{Sector Diff.}        
    \end{subfigure}
    
    \begin{subfigure}{0.5\textwidth}
        \includegraphics[width=8cm]{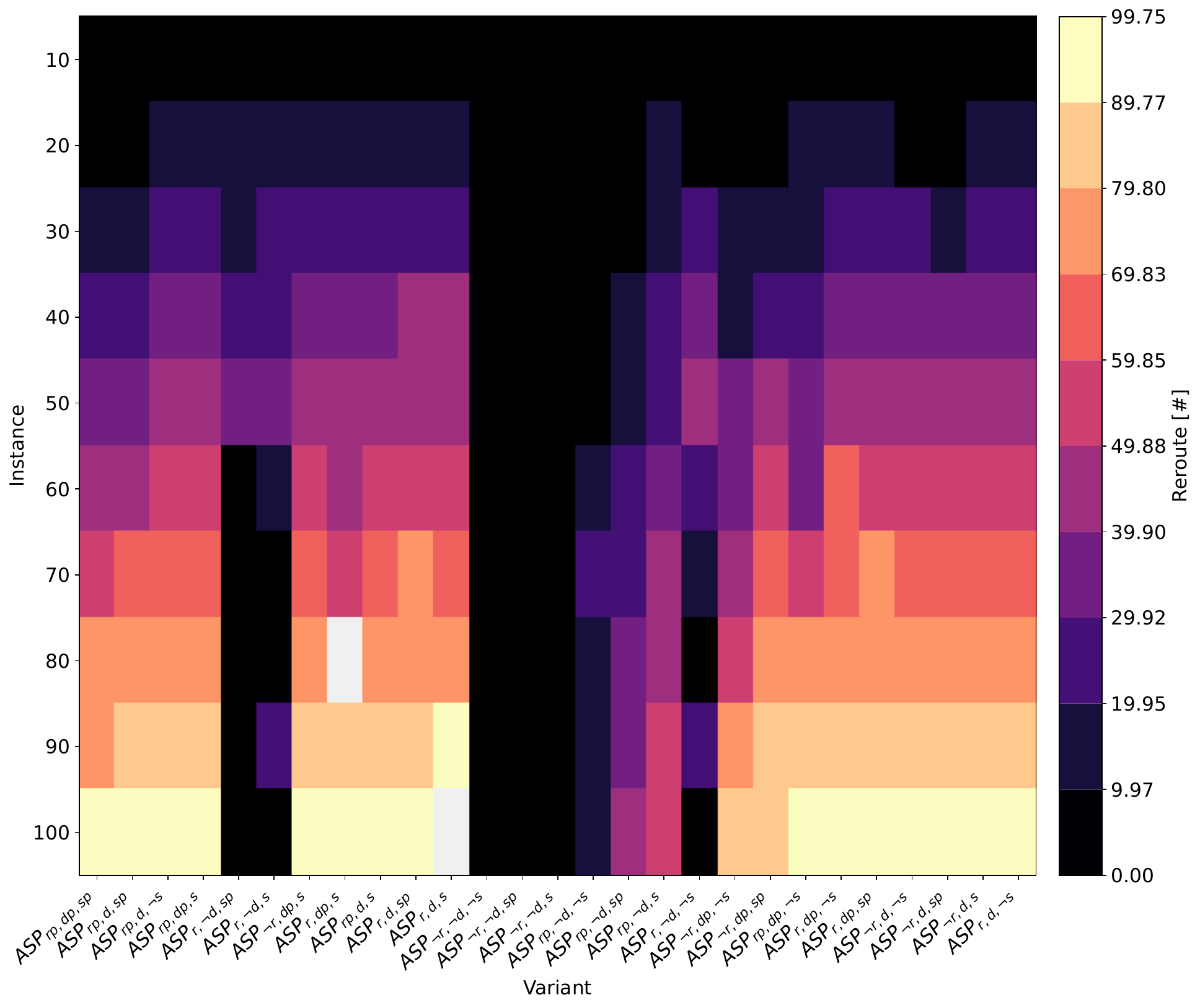}
        \caption{Reroute.}        
    \end{subfigure}
    \begin{subfigure}{0.49\textwidth}
        \includegraphics[width=8cm]{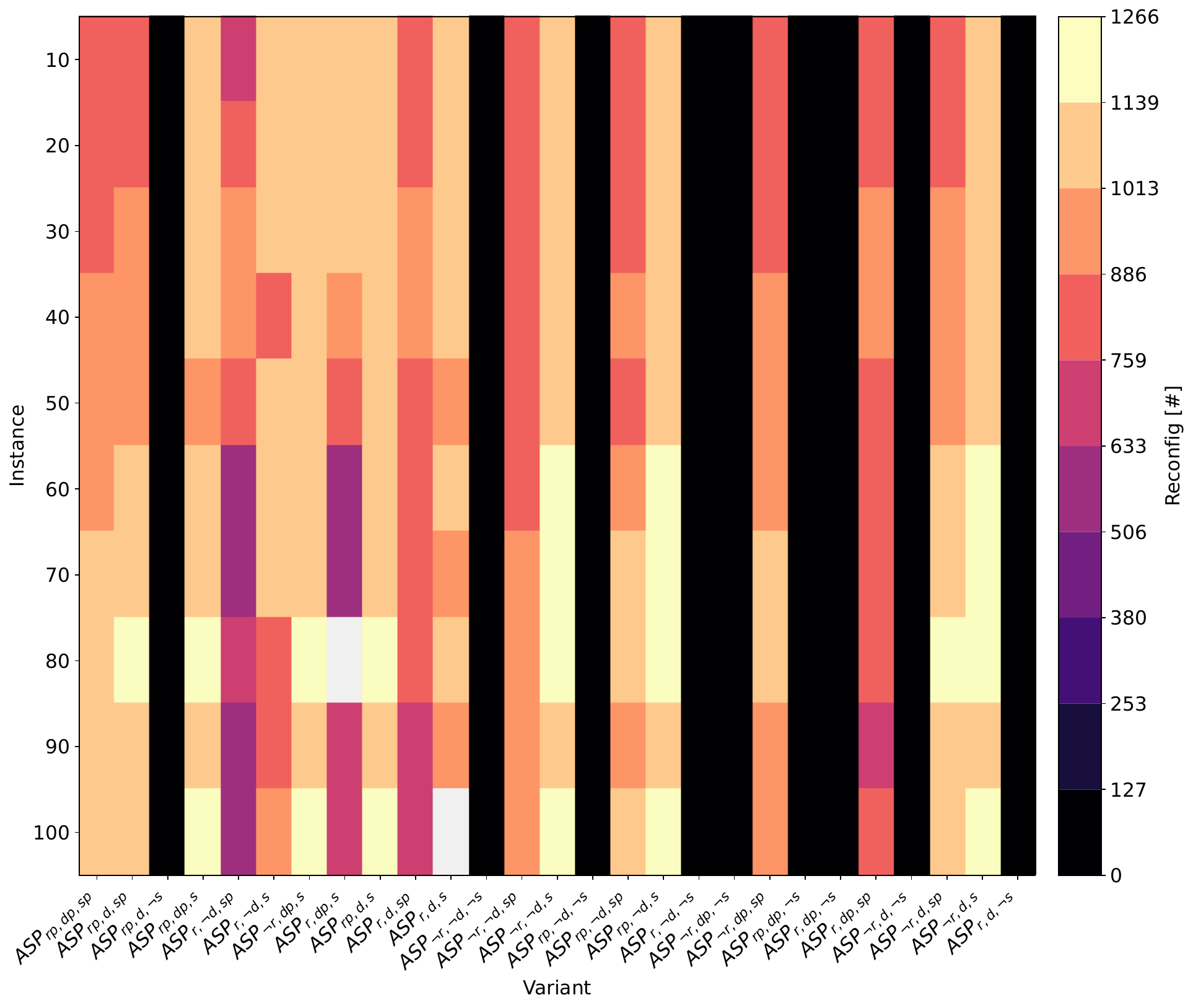}
        \caption{Reconfig.}        
    \end{subfigure}
    \caption{USA East (7x7)}
    \label{fig:heatmaps-USA}
\end{figure}

\begin{figure}
    \centering
    \begin{subfigure}{0.5\textwidth}
        \includegraphics[width=8cm]{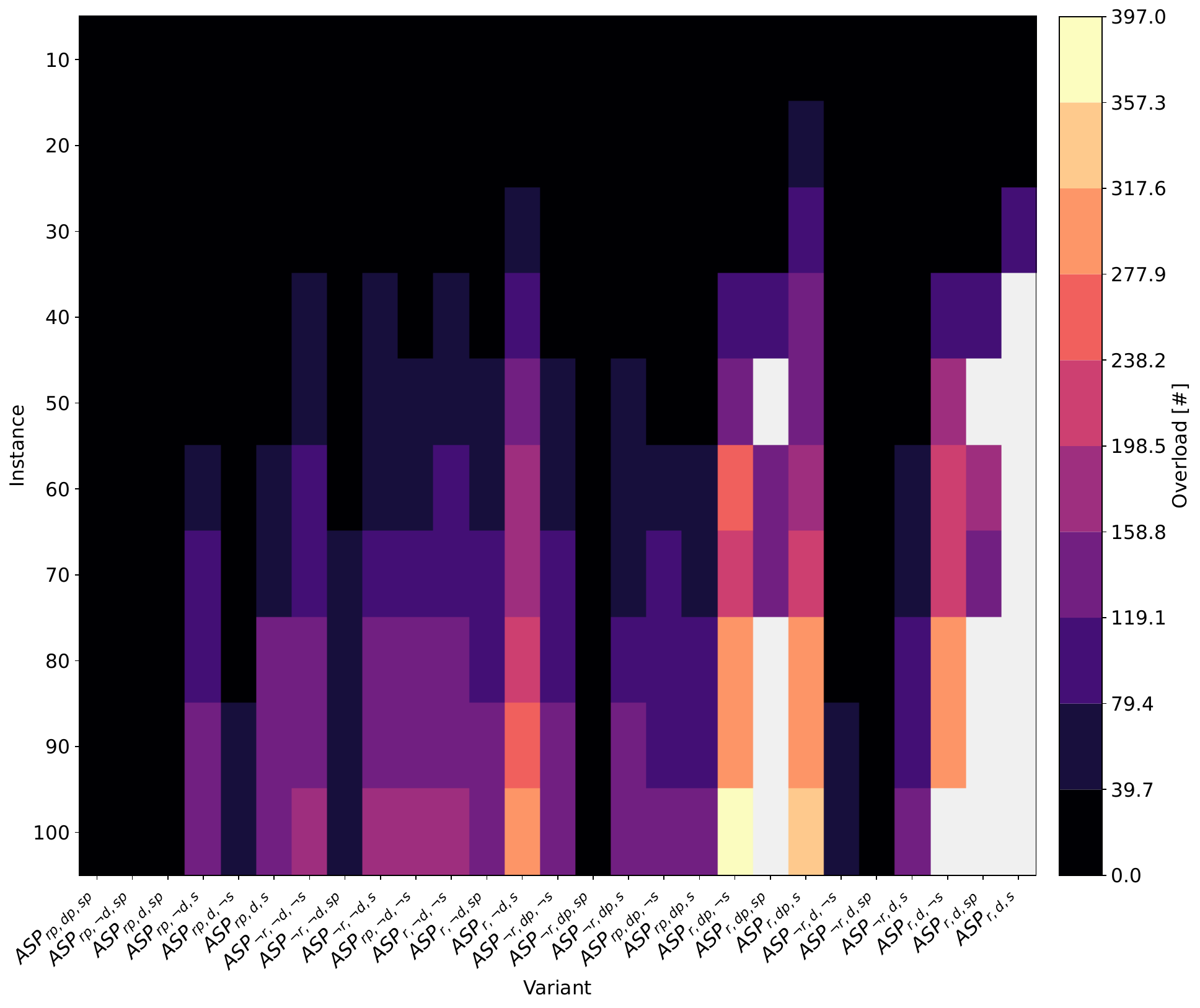}
        \caption{Overload.}        
    \end{subfigure}
    \begin{subfigure}{0.49\textwidth}
        \includegraphics[width=8cm]{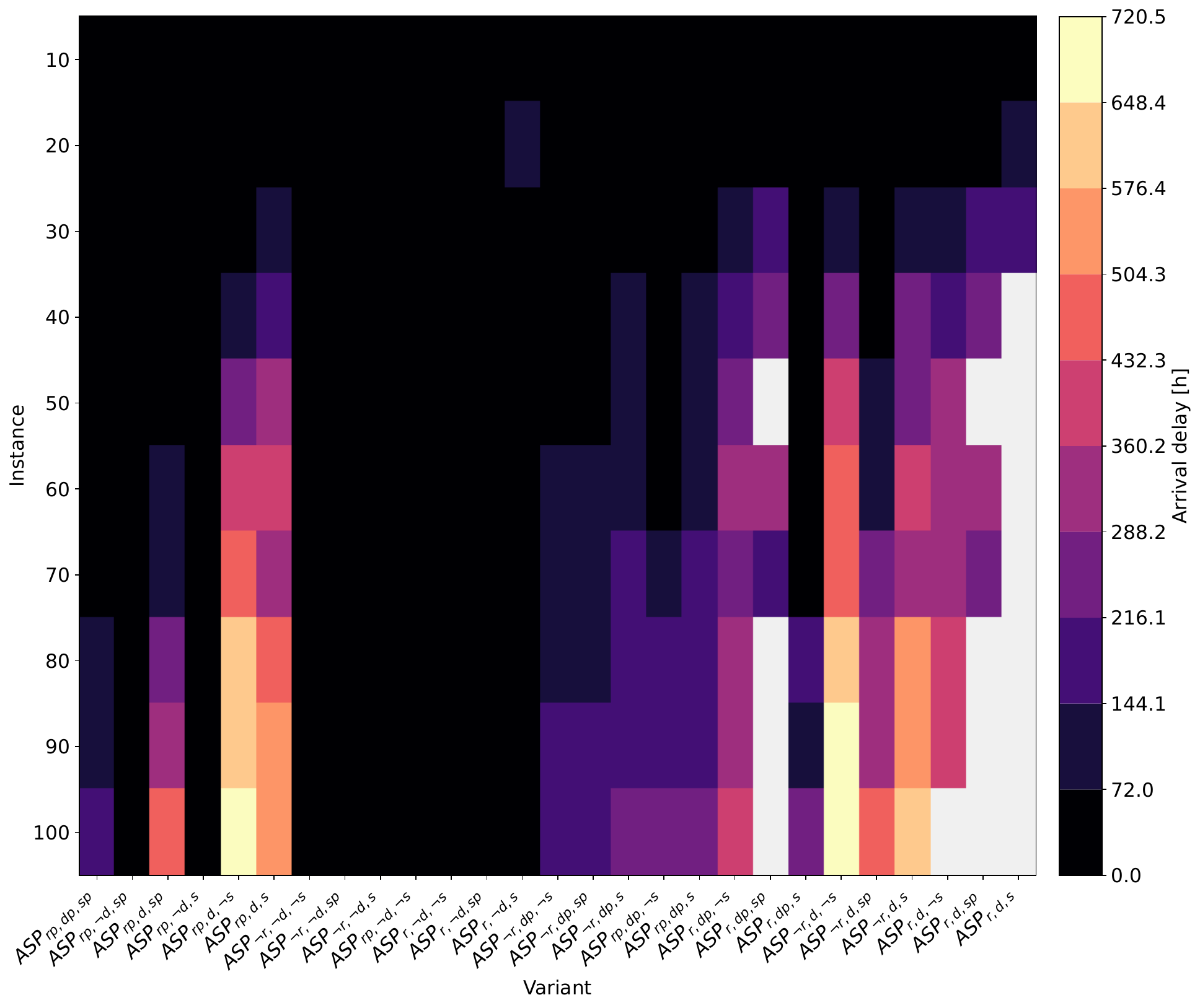}
        \caption{Arrival Delay.}        
    \end{subfigure}
    
    \begin{subfigure}{0.5\textwidth}
        \includegraphics[width=8cm]{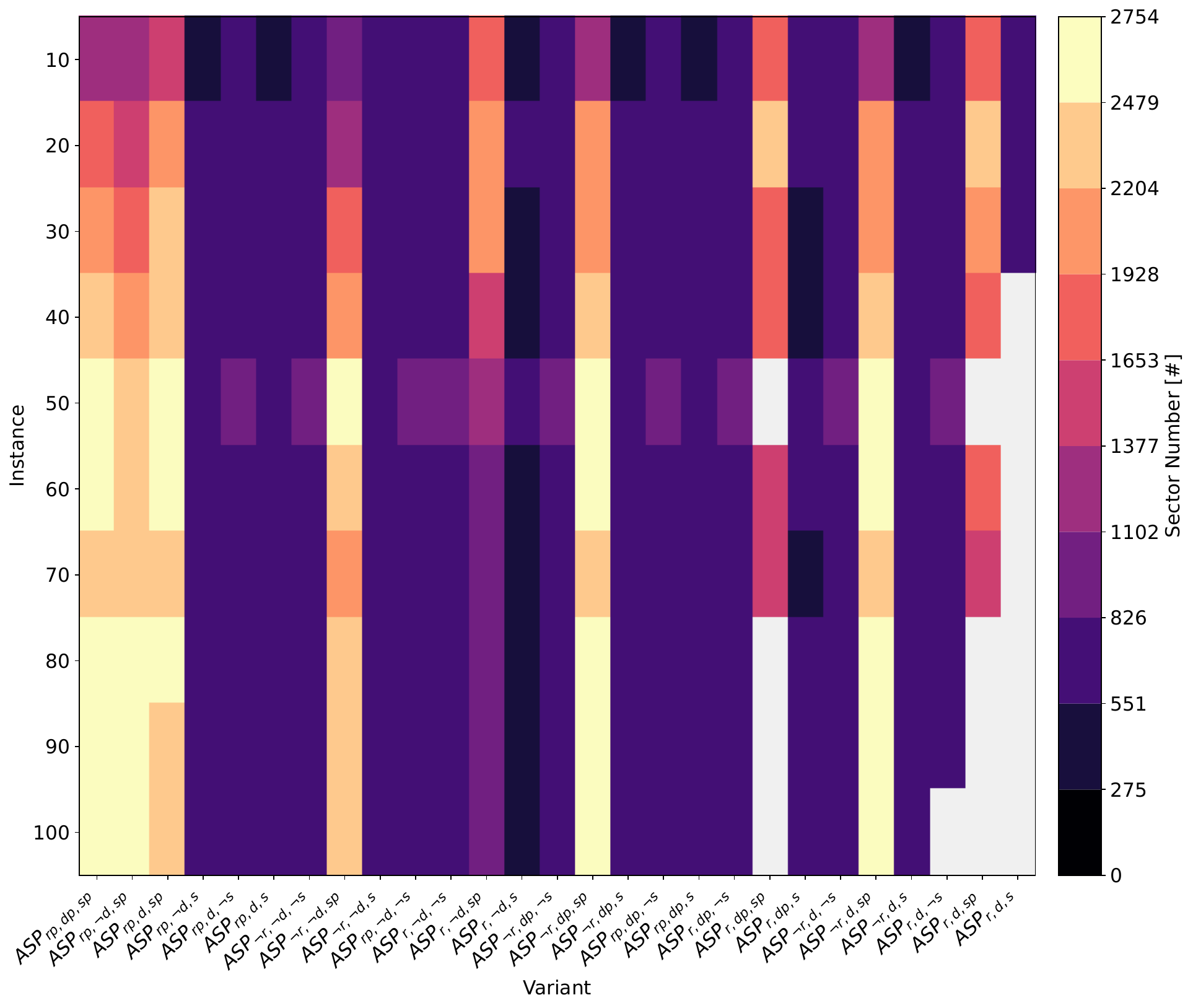}
        \caption{Sector Number.}        
    \end{subfigure}
    \begin{subfigure}{0.49\textwidth}
        \includegraphics[width=8cm]{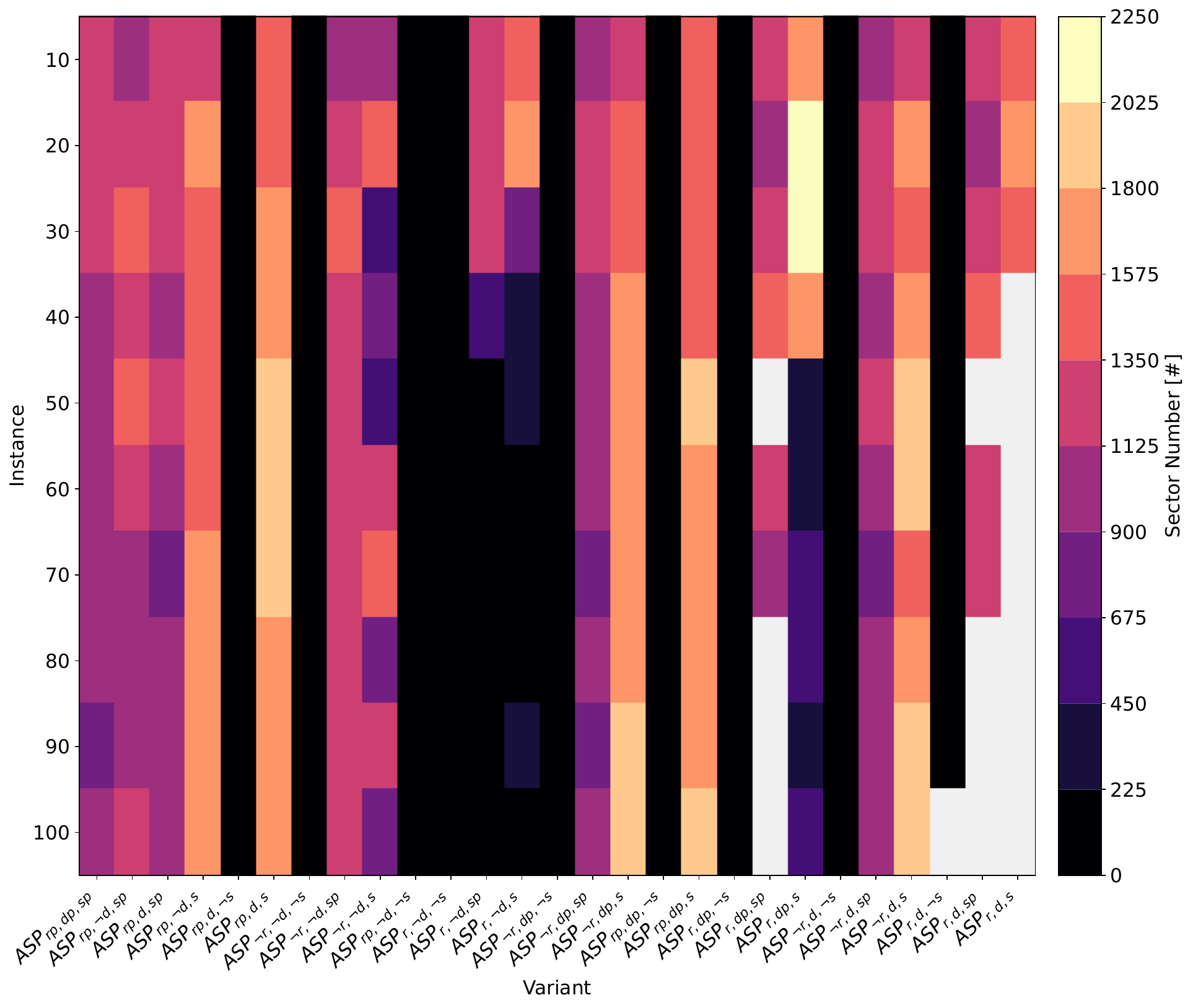}
        \caption{Sector Diff.}        
    \end{subfigure}
    
    \begin{subfigure}{0.5\textwidth}
        \includegraphics[width=8cm]{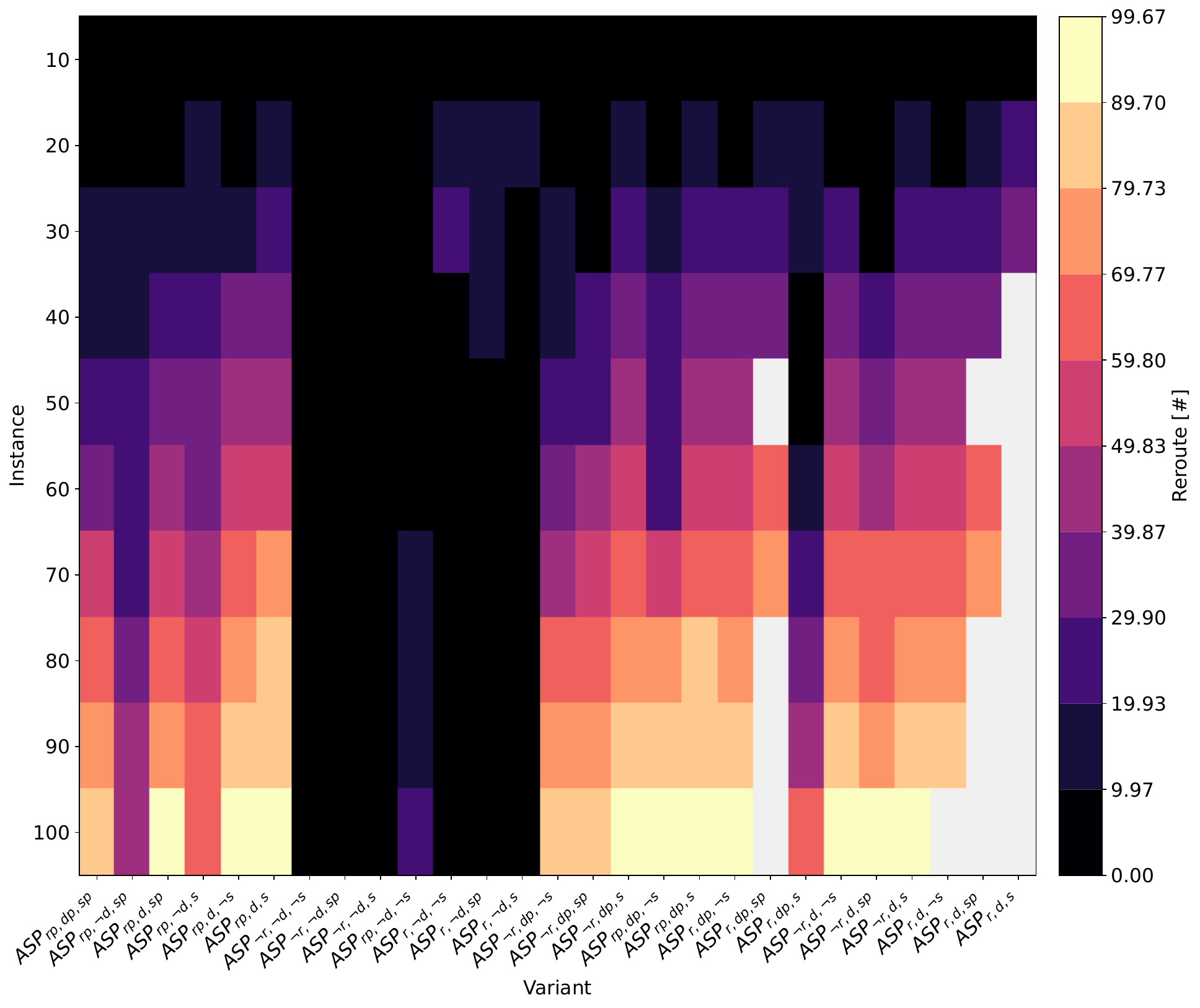}
        \caption{Reroute.}        
    \end{subfigure}
    \begin{subfigure}{0.49\textwidth}
        \includegraphics[width=8cm]{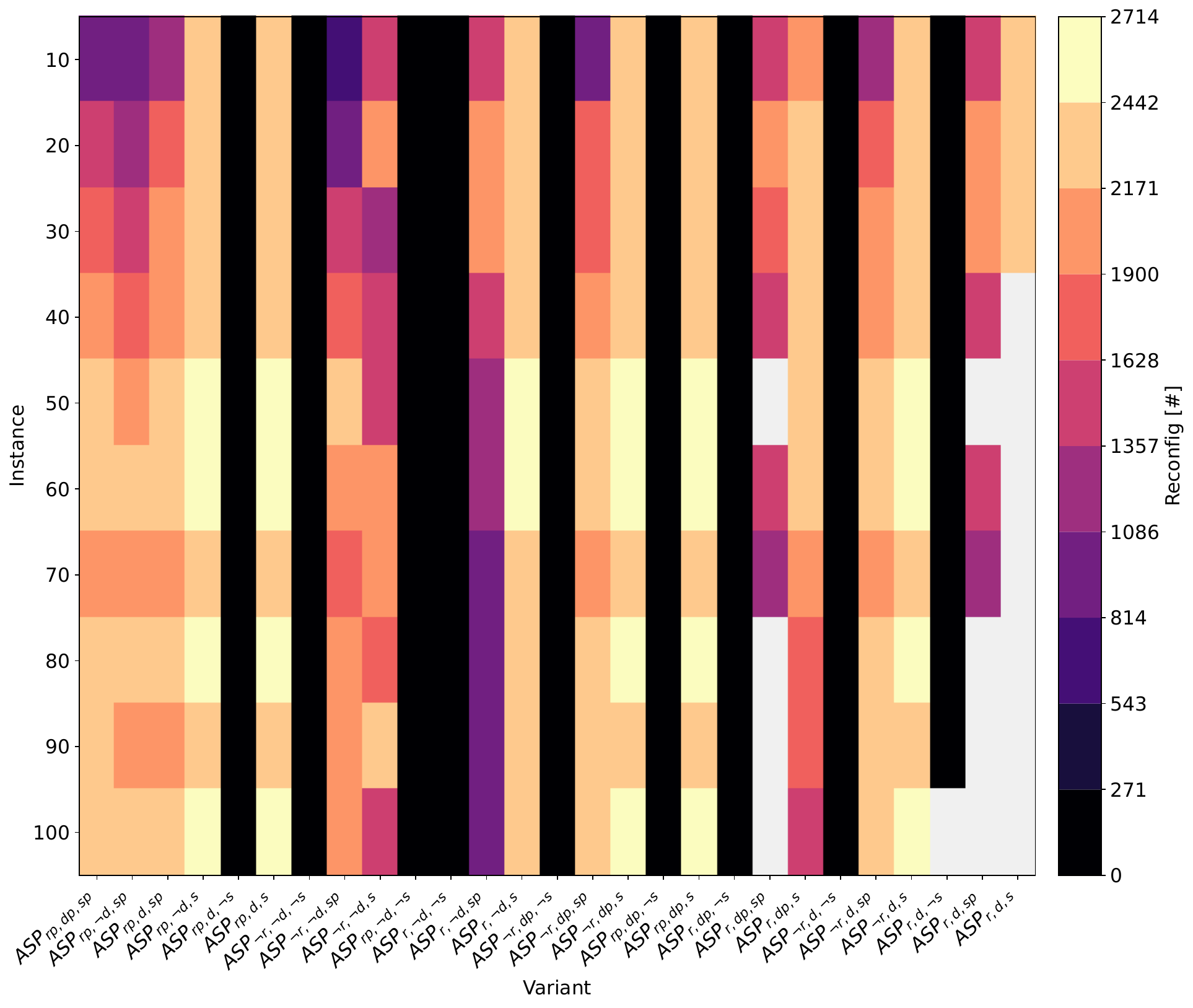}
        \caption{Reconfig.}        
    \end{subfigure}
    \caption{Major Europe (10x10)}
    \label{fig:heatmaps-MAJOR-EUROPE}
\end{figure}